\documentclass[journal]{IEEEtran}
\usepackage{hyperref}       % hyperlinks
\usepackage{url}            % simple URL typesetting
\usepackage{xcolor}
\usepackage{enumerate}
\usepackage{algorithm}
\usepackage{algorithmic}
\usepackage{subfloat}
\usepackage{subfig}
\newcommand*\cM{\mathcal{M}}
\newcommand*\cU{\mathcal{U}}
\newcommand*\cB{\mathcal{B}}
\newcommand*\cT{\mathcal{T}}
\newcommand*\C{\mathcal{C}}

\newcommand*\x{\mathbf{x}}

\newcommand*\gr{\nabla{F_i(w_t)}}

\newcommand*\pr{p_i(w_t)}
\newcommand*\er{e_i(t)}
\newcommand*\err{e_i(t+1)}
\usepackage{comment}
\usepackage{epsf}
\usepackage{graphicx}

\usepackage{color}

\usepackage{mathtools}
\usepackage{amsfonts}
\usepackage{amsthm}
\usepackage{amsmath, bm}
\usepackage{amssymb}

\usepackage{pgfplots}
\pgfplotsset{width=5.3cm,compat=1.9}

\usepackage{textcomp}
\usepackage{xcolor}

\DeclareSymbolFont{extraup}{U}{zavm}{m}{n}
\DeclareMathSymbol{\vardiamond}{\mathalpha}{extraup}{87}

\newtheorem{theorem}{Theorem}

\newtheorem{assumption}{Assumption}

\newtheorem{proposition}{Proposition}
\newtheorem{lemma}{Lemma}
\newtheorem{corollary}{Corollary}
\newtheorem{definition}{Definition}
\newtheorem{remark}{Remark}

\long\def\comment#1{}

\newcommand{\sign}{\mathsf{sign}}

\newcommand{\E}{\ensuremath{{\mathbb{E}}}}

\newcommand{\Prob}{\ensuremath{{\mathbb{P}}}}

\DeclareMathOperator*{\argmin}{argmin}

\newcommand{\real}{\ensuremath{\mathbb{R}}}

\begin{document}

\title{Communication-Efficient and Byzantine-Robust Distributed Learning with Error Feedback}

\author{Avishek  Ghosh,
        Raj Kumar  Maity,
        Swanand Kadhe,
        Arya Mazumdar, 
        and Kannan Ramchandran
        % <-this % stops a space
\thanks{Avishek Ghosh is with Halicioglu Data Science Institute, UC San Diego (work done while a PhD student at UC Berkeley). }% <-this % stops a space
\thanks{Raj Kumar Maity is with CS department, UMass Amherst}% <-this % stops a space
\thanks{Arya Mazumdar is with Halicioglu Data Science Institute, UC San Diego.}
\thanks{Swanand Kadhe and Kannan Ramachandran are with the EECS department, UC Berkeley.}
%\thanks{ This full paper is available at \href{https://tinyurl.com/nr93w6mh} .  Contact a2ghosh$@$ucsd.edu for further questions. }

}

%\vspace*{.15in}

%{\large{
%\begin{tabular}{ccc}
%Avishek Ghosh$^{\star}$, & Raj Kumar Maity $^{\dagger}$,& Swanand Kadhe $^{\star}$ \\
%& \hspace{-35mm} Arya Mazumdar$^{\dagger}$ & \hspace{-28mm} and   Kannan Ramchandran$^\star$
%\end{tabular}
%}}
%\vspace*{.15in}

%\begin{tabular}{c}
%Department of Electrical Engineering and Computer Sciences, UC Berkeley$^\star$ \\
%College of Information and Computer Sciences, UMASS Amherst$^\dagger$
%\end{tabular}

%\vspace*{.2in}

%\today

%\vspace*{.11in}

%\title{Communication-Efficient and Byzantine-Robust Distributed Learning}
%\date{}
%\author[]{Avishek Ghosh}
%\author[]{\hspace{2mm}Raj Kumar Maity}
%\author[]{\hspace{2mm}Swanand Kadhe}
%\author[]{\hspace{2mm}Arya Mazumdar}
%\author[]{Kannan Ramchandran}
%\affil[]{\{avishek\_ghosh,~swanand.kadhe,~kannanr\}@berkeley.edu, \{rajkmaity,~arya\}@cs.umass.edu}
%\affil[]{Department of Electrical Engineering and Computer Sciences, UC Berkeley \\
%Department of Computer Science, UMASS}

\markboth{IEEE Journal on Selected Areas in Information Theory }%
{Shell \MakeLowercase{\textit{et al.}}: Bare Demo of IEEEtran.cls for IEEE Journals}

\maketitle

\begin{abstract}
We develop a communication-efficient distributed learning algorithm that is robust against Byzantine worker machines. We propose and analyze a distributed gradient-descent algorithm that performs a simple thresholding based on gradient norms to mitigate Byzantine failures. We show the (statistical) error-rate of our algorithm matches that of Yin et al.~\cite{dong}, which uses more complicated schemes (coordinate-wise median, trimmed mean). Furthermore, for communication efficiency, we consider a generic class of $\delta$-approximate compressors from Karimireddi et al.~\cite{errorfeed} that encompasses sign-based compressors and top-$k$ sparsification. Our algorithm uses compressed gradients and gradient norms for aggregation and Byzantine removal respectively. We establish the statistical error rate for  non-convex smooth loss functions. We show that, in certain range of the compression factor $\delta$, the (order-wise) rate of convergence is not affected by the compression operation.
Moreover, we analyze the compressed gradient descent algorithm with error feedback (proposed in \cite{errorfeed}) in a distributed setting and in the presence of Byzantine worker machines. We show that exploiting error feedback  improves the statistical error rate. Finally, we experimentally validate our results and show good performance in convergence for convex (least-square regression) and non-convex (neural network training) problems. 
\end{abstract}


\IEEEpeerreviewmaketitle

\section{Introduction}
\label{sec:intro}

In many real-world applications, the size of training datasets has grown significantly over the years to the point that it is becoming crucial to implement learning algorithms in a distributed fashion. A commonly used distributed learning framework is data parallelism, in which large-scale datasets are distributed over multiple {\it worker machines} for parallel processing in order to speed up computation. In other applications such as {\it Federated Learning} (~\cite{konevcny2016federated}), the data sources are inherently distributed since the data are stored locally in users' devices.  

In a standard distributed gradient descent framework, a set of worker machines store the data, perform local computations, and communicate gradients to the central machine (e.g., a parameter server). The central machine processes the results from workers to update the model parameters. Such distributed frameworks need to address the following two fundamental challenges.
First, the gains due to parallelization are often bottlenecked in practice by heavy communication overheads between workers and the central machine. This is especially the case for large clusters of worker machines or for modern deep learning applications using models with millions of parameters. Moreover, in Federated Learning, communication from a user device to the central server is directly tied to the user's upload bandwidth costs. 
%Therefore, it is of paramount importance to reduce communication overhead in distributed learning algorithms. 
%reducing communication overhead in distributed learning algorithms has recently received significant research attention. In particular, various quantization or sparsification techniques have been developed to reduce communication overhead~\cite{alistarh2018convergence, stich2018sparsified, ivkin2019communication, dme,atomo,terngrad,qsgd}.
Second, messages from workers are susceptible to errors due to hardware faults or software bugs, stalled computations, data crashes, and unpredictable communication channels. In scenarios such as Federated Learning, users may as well be malicious and act adversarially. The inherent unpredictable (and potentially adversarial) nature of compute units is typically modeled as \emph{Byzantine failures}. Even if a single worker is Byzantine, it can be fatal to most learning algorithms (\cite{lamport}).
%Hence, it is crucial to develop distributed algorithms that are robust to Byzantine failures.
%developing Byzantine-robust learning algorithms have . 

Both these challenges, communication efficiency and Byzantine-robustness, have recently attracted significant research attention, albeit mostly separately. In particular, several recent works have proposed various quantization or sparsification techniques to reduce the communication overhead (\cite{alistarh2018convergence, stich2018sparsified, ivkin2019communication, dme,atomo,terngrad,qsgd,vqsgd}). The goal of these quantization schemes is to compute an unbiased estimate of the gradient with bounded second moment in order to achieve good convergence guarantees. The problem of developing Byzantine-robust distributed algorithms has been considered in~\cite{alistarh2017communication,su,feng,chen,dong, dong1,blanchard2017byzantine,ghosh2019robust}.

A notable exception to considering communication overhead separately from Byzantine robustness is the recent work
 of~\cite{anima}. In this work, a sign-based compression algorithm {\it signSGD} of \cite{signsgd} is shown to be Byzantine fault-tolerant. The main idea of {\it signSGD} is to communicate the coordinate-wise signs of the gradient vector to reduce communication and employ a majority vote during the aggregation to mitigate the effect of Byzantine units. However, {\it signSGD} suffers from two major drawbacks. First, sign-based algorithms do not converge in general (\cite{errorfeed}). In particular, \cite[Section 3]{errorfeed} presents several convex counter examples where {\it signSGD} fails to converge even though \cite[Theorem 2]{anima} shows convergence guarantee for non-convex objective under certain assumptions. Second, {\it signSGD} can handle only a limited class of adversaries, namely {\it blind multiplicative adversaries} (\cite{anima}). Such an adversary manipulates the gradients of the worker machines by multiplying it (element-wise) with a vector that can scale and randomize the sign of each coordinate of the gradient. However, the vector must be chosen before observing the gradient (hence `blind'). 
 In a very recent work  \cite{zhu2021broadcast}, authors address the problem of  stochastic and compression noise in the presence of Byzantine machines  and propose BROADCAST, a variance reduction method with  gradient difference compression scheme.

 In this work, we develop communication-efficient and robust learning algorithms that overcome both these drawbacks\footnote{We compare our algorithm with {\it signSGD} in Section~\ref{sec:experiment}.}. Specifically, we consider the following distributed learning setup. There are $m$ worker machines, each storing $n$ data points. The data points are generated from some unknown distribution $\mathcal{D}$. The objective is to learn a parametric model that minimizes a population loss function $F:\mathcal{W}\rightarrow \mathbb{R}$, where $F$ is defined as an expectation over $\mathcal{D}$, and $\mathcal{W}\subseteq\mathbb{R}^d$ denotes the parameter space. We choose the loss function $F$ to be non-convex. With the rapid rise of the neural networks, the study of \emph{local minima} in non-convex optimization framework  has become  imperative \cite{soudry2016bad,ge_etal}.  
For gradient compression at workers, we consider the notion of a $\delta$-approximate compressor from~\cite{errorfeed} that encompasses sign-based compressors like QSGD (\cite{qsgd}), $\ell_1$-QSGD (\cite{errorfeed}) and top-$k$ sparsification (\cite{stich2018sparsified}). 
We assume that $0\leq \alpha < 1/2$ fraction of the worker machines are Byzantine.  In contrast to blind multiplicative adversaries, we consider unrestricted adversaries. 

Our key idea is to use a simple threshold (on local gradient norms) based Byzantine resilience scheme instead of robust aggregation methods such as coordinate wise median or trimmed mean of \cite{dong}. We mention  that similar ideas are used in \emph{gradient clipping}, where gradients with norm more than a threshold is truncated. This is used in applications like training neural nets \cite{qian2021understanding} to handle the issue of exploding gradients, and in differentially private SGD \cite{chen2020understanding}, to limit the sensitivity of the gradients\footnote{ Note that although gradient clipping and norm based thresholding have some similarities, they are not identical. In gradient clipping, although we scale down (clip) the gradients, we retain them. On the other hand, in norm based thresholding, we aim to identify the Byzantine machines and remove them. Note that in our learning framework, we have $\alpha$ fraction of Byzantine workers, and an estimate of $\alpha$ is known to the learning algorithm. When $\alpha$ is very close to $0$, our learning algorithm does not trim worker machines, and the effect of all gradients are considered. If we employ gradient clipping in this regime, depending on the threshold used in the clipping operation, some gradients may be scaled back. As a result, the convergence rate will suffer. On the other hand, suppose $\alpha$ is large. In this regime, our algorithm tend to identify and remove the influence of the Byzantine workers, where gradient clipping would scale them down, but retain term in the learning process. This could potentially slow down the learning as the Byzantine machines may send \emph{any arbitrary} updates, which are different for the actual gradient norms and directions. Hence, in both the regimes, the knowledge of $\alpha$ helps our algorithm to handle the Byzantine workers graciously compared to the gradient clipping operation.}.

Our main result is to show that, for a wide range of compression factor $\delta$, the statistical error rate of our proposed threshold-based scheme is (order-wise) identical to the case of no compression considered in~\cite{dong}. In fact, our algorithm achieves order-wise optimal error-rate in parameters $(\alpha,n,m)$.
Furthermore, to alleviate convergence issues associated with sign-based compressors, we employ the technique of error-feedback from \cite{errorfeed}. In this setup, the worker machines store the difference between the actual and compressed gradient and add it back to the next step so that the \emph{correct} direction of the gradient is not forgotten.  We show that using error feedback with our threshold based Byzantine resilience scheme not only achieves better statistical error rate but also improves the rate of convergence. We outline our specific contributions in the following.

\vspace{2mm}
{\bf Our Contributions:} We propose a communication-efficient and robust distributed gradient descent (GD) algorithm. The algorithm takes as input the gradients compressed using a $\delta$-approximate compressor along with the norms\footnote{We can handle any convex norm.} (of either compressed or uncompressed gradients), and performs a simple thresholding operation  based on gradient norms to discard $\beta > \alpha$ fraction of workers with the largest norm values. We establish the statistical error rate of the algorithm for arbitrary smooth population loss functions as a function of the number of worker machines $m$, the number of data points on each machine $n$, dimension $d$, and the compression factor $\delta$. In particular, we show that our algorithm achieves the following statistical error rate\footnote{Throughout the paper $\mathcal{O}(\cdot)$ hides multiplicative constants, while $\tilde{\mathcal{O}}(\cdot)$ further hides logarithmic factors.} for the regime $\delta > 4\beta + 4\alpha -8\alpha^2 + 4\alpha^3$:
\begin{align}
\tilde{\mathcal{O}} \left( d^2 \left[ \frac{\alpha^2}{n} + \frac{1-\delta}{n} + \frac{1}{mn}\right ] \right).
\label{eqn:error-rate}
\end{align}
We first note that when $\delta =1$ (uncompressed), the error rate is $\tilde{\mathcal{O}}(d^2[\frac{\alpha^2}{n}+\frac{1}{mn}])$, which matches \cite{dong}. Notice that we use a simple threshold (on local gradient norms) based Byzantine resilience scheme in contrast with the coordinate wise median or trimmed mean of \cite{dong}. We note that for a fixed $d$ and the compression factor $\delta$ satisfying $\delta \geq 1 - \alpha^2$, the statistical error rate become $\tilde{\mathcal{O}}(\frac{\alpha^2}{n}+\frac{1}{mn})$, which is order-wise identical to the case of no compression~\cite{dong}. In other words, in this parameter regime, the compression term does not contribute (order-wise) to the statistical error. Moreover, it is shown in~\cite{dong} that, for strongly-convex loss functions and a fixed $d$, no algorithm can achieve an error lower than $\tilde{\Omega}(\frac{\alpha^2}{n}+\frac{1}{mn})$, implying that our algorithm is order-wise optimal in terms of the statistical error rate in the parameters $(\alpha,n,m)$.

Furthermore, we strengthen our distributed learning algorithm by using error feedback to correct the direction of the local gradient. We show (both theoretically and via experiments) that using error-feedback with a $\delta$-approximate compressor indeed speeds up the convergence rate and attains better (statistical) error rate.
Under the assumption that the gradient norm of the local loss function is upper-bounded by $\sigma$, we obtain the following (statistical) error rate:
\begin{align*}
\tilde{\mathcal{O}} \left( d^2 \left[ \frac{\alpha^2}{n} + \frac{(1-\delta) \sigma^2}{d^2 \,\delta} + \frac{1}{mn}\right ] \right)
\end{align*} 
provided a similar $(\delta,\alpha)$ trade-off\footnote{See Theorem~\ref{thm:non_convex_err} for details.}. We note that in the no-compression setting $(\delta=1)$, we recover the $\tilde{\mathcal{O}}(\frac{\alpha^2}{n}+\frac{1}{mn})$ rate. In experiments (Section~\ref{sec:experiment}), we  see that adding error feedback indeed improves the performance of our algorithm.

We experimentally evaluate our algorithm for convex and non-convex losses. For the convex case, we choose the linear regression problem, and for the non-convex case, we train a ReLU activated feed-forward fully connected neural net. We compare our algorithm with the non-Byzantine case and {\it signSGD} with majority vote, and observe that our algorithm converges faster using the standard MNIST dataset.

A major technical challenge of this paper is to handle compression and the Byzantine behavior of the worker machines simultaneously. We build up on the techniques of \cite{dong} to control the Byzantine machines. In particular, using certain distributional assumption on the partial derivative of the loss function and exploiting uniform bounds via careful covering arguments, we show that the local gradient on a non-Byzantine worker machine is close to the gradient of the population loss function.

Note that in some settings, our results may not have an optimal dependence on dimension $d$. This is due to the norm-based Byzantine removal schemes.
%Also, for a generic $\delta$-approximate compressors, we assume a restricted class of adversaries. 
Obtaining optimal dependence on $d$ is an interesting future direction. 

%Finally we experimentally demonstrate our method in different setup with convex and non-convex loss. For the convex loss we choose linear regression problem and for the byzantine affect we add Gaussian noise to the gradient of the worker node. We compare our algorithm with the vanilla SGD and signSGD with majority vote and show better convergence. Also we experimentally establish the error feedback scheme along with the norm based byzantine removal to be a good method  in byzantine setup despite the lack of theoretical guarantee. For non-convex problem, we train ReLU activated neural net and show good performance in minimizing training loss.    

\vspace{2mm}
{\bf Organization:} We describe the problem formulation in Section ~\ref{sec:setup}, and give a brief overview of $\delta$-compressors in Section ~\ref{sec:compression}. Then, we present our proposed algorithm in Section ~\ref{sec:algo}. We analyze the algorithm, first, for a {\em restricted} (as described next) adversarial model in Section ~\ref{sec:restricted}, and in the subsequent section, remove this restriction.
 In Section~\ref{sec:restricted}, we restrict our attention to an adversarial model in which Byzantine workers can provide arbitrary values as an input to the compression algorithm, but they correctly implement the same compression scheme as mandated. 
%A Byzantine worker can act maliciously by supplying an arbitrary vector as an input to its compressor, however the compression happens correctly.
%Though this adversarial model is restricted, we argue that it is well-suited for applications wherein compression happens outside of worker machines. For example, Apache MXNet, a deep learning framework designed to be distributed on cloud infrastructures, uses  NVIDIA Collective Communication Library (NCCL) that employs gradient compression (see \cite{MXNet}). 
%Also, in a Federated Learning setup the compression can be part of the communication protocol.  Furthermore, this can happen when worker machines are divided into groups, and each group is associated with a {\it compression unit}. As an example, cores in a multi-core processor (\cite{kangwook_core}) acting as a group of worker machines with the compression carried out by a separate processor, or servers co-located on a rack (\cite{costa2015r}) acting as a group with the compression carried out by the top-of-the-rack switch. 
In Section ~\ref{sec:arbitrary}, we remove this restriction on the Byzantine machines. 
%Here, we change the learning algorithm slightly to accommodate arbitrary Byzantine nodes. 
As a consequence, we observe (in Theorem~\ref{thm:non_convex1}) that   the modified algorithm works under a stricter assumption, and performs slightly worse than the one in restricted adversary setting.  
%To alleviate the convergence issues of sign based optimization algorithms (e.g., {\it signSGD}), \cite{errorfeed} proposes a class of optimization algorithms that use the error in compression as feedback. In this setup, the worker machines store the difference between the actual and compressed gradient and add it back to the next step so that the \emph{correct} direction of the gradient is not forgotten. 
In Section~\ref{sec:error_feedback}, we strengthen our algorithm by including error-feedback at worker machines, and provide statistical guarantees for non-convex smooth loss functions. We show that error-feedback indeed improves the performance of our optimization algorithm in the presence of arbitrary adversaries.

\subsection{Related Work}

\paragraph{Gradient Compression:} The foundation of gradient quantization was laid in  \cite{strom2015scalable, seide20141}. In the work of \cite{qsgd,terngrad,atomo} each co-ordinate of the gradient vector is represented with a small number of bits. Using this, an unbiased estimate of the gradient is computed. In these works, the communication cost is $\Omega(\sqrt{d})$ bits. In \cite{dme}, a quantization scheme was proposed for distributed  mean estimation. The tradeoff between communication and accuracy is studied in \cite{mjordan}. Variance reduction in communication efficient stochastic distributed  learning has been studied in \cite{diana19}. Sparsification techniques are also used instead of quantization to  reduce communication cost. Gradient sparsification has beed studied in \cite{stich2018sparsified, alistarh2018convergence, ivkin2019communication} with provable guarantees. The main idea is to communicate top components of the $d$-dimensional local gradient to get good estimate of the true global gradient.

\paragraph{Byzantine Robust Optimization:} In the distributed learning context, a generic framework of one shot median based robust learning has been proposed in \cite{feng}. In \cite{chen} the issue of Byzantine failure is tackled by grouping the servers in batches and computing the median of batched servers. Later in \cite{dong, dong1}, co-ordinate wise median, trimmed mean and iterative filtering based algorithm have been proposed and optimal statistical error rate is obtained. Also, \cite{mhamdi2018hidden, Damaskinos} considers adversaries may steer convergence to bad local minimizers. In this work, we do not assume such adversaries.

Gradient compression and Byzantine robust optimization have simultaneously been addressed in a recent paper \cite{anima}. Here, the authors use {\it signSGD} as compressor and majority voting as robust aggregator. As explained in \cite{errorfeed}, {\it signSGD} can run into convergence issues. Also, \cite{anima} can handle a restricted class of adversaries that are \emph{multiplicative} (i.e., multiply each coordinate of gradient by arbitrary scalar) and \emph{blind} (i.e., determine how to corrupt the gradient before observing the true gradient). In this paper, for compression, we use a generic $\delta$ approximate compressor. Also, we can handle arbitrary Byzantine worker machines.

Very recently, \cite{errorfeed} uses error-feedback to remove some of the issues of sign based compression schemes. In this work, we extend the framework to a distributed setting and prove theoretical guarantees in the presence of Byzantine worker machines.

\paragraph{Notation:} Throughout the paper, we assume $ C,C_1,C_2,.., c,c_1,..$ as positive universal constants, the value of which may differ from instance to instance. $[r]$ denotes the set of natural numbers $\{1,2,..,r\}$. Also, $\|.\|$ denotes the $\ell_2$ norm of a vector and the operator norm of a matrix unless otherwise specified.

\section{Problem Formulation}
\label{sec:setup}

In this section, we formally set up the problem. We consider a standard statistical problem of risk minimization. In a distributed setting, suppose we have one central and $m$ worker nodes and the worker nodes 
%can only 
communicate to the central node. Each worker node contains $n$ data points. We assume that the $m n$ data points are sampled independently from some unknown distribution $\mathcal{D}$. Also, let $f(w,z)$ be the non-convex loss function of a parameter vector $w \in \mathcal{W} \subseteq \real^d$ corresponding to data point $z$, where $\mathcal{W}$ is the parameter space. Hence, the population loss function is $F(w) =  \E_{z \sim \mathcal{D}}[f(w,z)]$. Our goal is to obtain the following:
\begin{align*}
w^* = \mathrm{argmin}_{w \in \mathcal{W}} F(w),
\vspace{-3mm}
\end{align*}
where we assume $\mathcal{W}$ to be a convex and compact subset of $\real^d$ with diameter $D$. In other words, we have $\|w_1 -w_2 \| \leq D$ for all $w_1,w_2 \in \mathcal{W}$. 
%\textcolor{cyan}{Reviewer 2: Not happy with the problem formulation. All the functions are closed to one another.}
Each worker node is associated with a local loss defined as $F_i(w) = \frac{1}{n}\sum_{j=1}^n f(w,z^{i,j})$, where $z^{i,j}$ denotes the $j$-th data point in the $i$-th machine. This is precisely the empirical risk function of the $i$-th worker node.

 We assume a setup where worker $i$ compresses the local gradient and sends to the central machine. The central machine aggregates the compressed gradients, takes a gradient step to update the model and broadcasts the updated model to be used in the subsequent iteration. Furthermore, we assume that $\alpha$ fraction of the total workers nodes are Byzantine, for some $\alpha <1/2$. Byzantine workers can send any arbitrary values to the central machine. In addition, Byzantine workers may completely know the learning algorithm and are allowed to collude with each other.

Next, we define a few (standard) quantities that will be required in our analysis.

\begin{definition}
(Sub-exponential random variable) A zero mean random variable $Y$ is called $v$-sub-exponential if $\E [e^{\lambda Y}] \leq e^{\frac{1}{2}\lambda^2 v^2}$, for all $|\lambda| \leq \frac{1}{v}$.
\end{definition}

\begin{definition}
(Smoothness) A function $h(.)$ is $L_F$-smooth if $ h(w) \leq h(w') + \langle \nabla h(w'), w -w' \rangle +\frac{L_F}{2}\|w - w' \|^2$ $\forall \,  w$, $w'$.
\end{definition}

\begin{definition}
(Lipschitz) A function $h(.)$ is $L$-Lipschitz if $\|h(w) -h(w')\| \leq L\|w -w'\|$  $\forall \, w$, $w'$.
\end{definition}

\section{Compression At Worker Machines}
\label{sec:compression}
In this section, we consider a generic class of compressors from~\cite{stich2018sparsified} and \cite{errorfeed} as described in the following.

%We first describe a class of compressors known as $\delta$ approximate compressor (\cite{stich2018sparsified,errorfeed}).
 
\begin{definition}[$\delta$-Approximate Compressor]
\label{asm:compress}
An operator $\mathcal{Q}(.): \real^d \rightarrow \real^d$ is defined as $\delta$-approximate compressor on a set $\mathcal{S} \subseteq \real^d$ if,  $\forall \, x \in \mathcal{S}$,
\begin{align*}
\|\mathcal{Q}(x) -x \|^2 \leq (1-\delta) \|x\|^2,
\end{align*}
where $\delta \in (0,1]$ is the compression factor.
\end{definition} 
\noindent Furthermore, a randomized operator $\mathcal{Q}(.)$ is $\delta$-approximate compressor on a set $\mathcal{S} \subseteq \real^d$ if, 
\begin{align*}
\mathbb{E}\left( \|\mathcal{Q}(x) -x \|^2 \right) \leq (1-\delta)\|x\|^2
\end{align*}
holds for all $x \in \mathcal{S}$, where the expectation is taken with respect to the randomness of $\mathcal{Q}(.)$. In this paper, for the clarity of exposition, we consider the deterministic form of the compressor (as in Definition~\ref{asm:compress}). However, the results can be easily extended for randomized $\mathcal{Q}(.)$.

Notice that $\delta=1$ implies $\mathcal{Q}(x)=x$ (no compression). We list a few examples of $\delta$-approximate compressors (including a few from \cite{errorfeed}) here:

\begin{enumerate}
\item top$_k$ operator, which selects $k$ coordinates with largest absolute value; for $1\leq k \leq d$, $(\mathcal{Q}(x))_i = (x)_{\pi(i)}$ if $i \leq k$, and $0$ otherwise, where $ \pi$ is a permutation of $[d]$ with $(|x|)_{\pi(i)} \geq (|x|)_{\pi(i+1)}$ for $i \in [d-1]$. This is a $k/d$-approximate compressor.

\item $k$-PCA that uses top $k$ eigenvectors to approximate a matrix $X$ (\cite{atomo}).

\item Quantized SGD (QSGD) \cite{qsgd}, where $\mathcal{Q}(x_i) = \| x\|\cdot \sign(x_i)\cdot \xi_i(x)$, where $\sign(x_i)$ is the coordinate-wise sign vector, and $\xi_i(x)$ is defined as following: let $0 \leq l_i \leq s $, be an integer such that $|x_i|/\|x\| \in [l_i/s, (l_i+1)/s]$. Then, $\xi_i = l_i/s$ with probability $1- \frac{|x_i|}{c\|x\|\sqrt{d}}+ l_i$ and $(l+1)/s$ otherwise. \cite{qsgd} shows that it is a $1-\min(d/s^2, \sqrt{d}/s)$-approximate compressor.

\item Quantized SGD with $\ell_1$ norm \cite{errorfeed}, $\mathcal{Q}(x) = \frac{\|x\|_1}{d} \sign(x)$, which is $\frac{\|x\|_1^2}{d\|x\|^2}$-approximate compressor. In this paper, we call this compression scheme as $\ell_1$-QSGD.
\end{enumerate}

Apart from these examples, several randomized compressors are also discussed in~\cite{stich2018sparsified}.
Also,  the \emph{signSGD}  compressor, $\mathcal{Q}(x)=\sign(x)$, where $\sign(x)$ is the (coordinate-wise) sign operator, was proposed in \cite{anima_signsgd, signsgd}. Here the local machines send a $d$-dimensional vector containing coordinate-wise sign of the gradients.

\begin{algorithm}[t!]
  \caption{Robust Compressed Gradient Descent}
  \begin{algorithmic}[1]
 \STATE  \textbf{Input:} Step size $\gamma$, Compressor $\mathcal{Q}(.)$, $q >1$, $\beta < 1$. Also define,
 \begin{equation}
  \mathcal{C}(x) =
    \begin{cases}
      \{ \mathcal{Q}(x), \|x\|_q \} \quad \forall x \in \mathbb{R}^d  & \text{Option I}\\
        \mathcal{Q}(x)  \qquad \qquad \forall x \in \mathbb{R}^d  & \text{Option II}
    \end{cases}  \nonumber     
\end{equation}
 \STATE \textbf{Initialize:} Initial iterate $w_0 \in \mathcal{W}$ \\
  \FOR{$t=0,1, \ldots, T-1 $}
  \STATE \underline{Central machine:} broadcasts $w_t$  \\
  \textbf{ for $ i \in [m]$ do in parallel}\\
  \STATE \underline{$i$-th worker machine:} 
    \begin{itemize}
    \item Non-Byzantine:
    \begin{itemize}
     \item Computes $\nabla F_i (w_t)$; sends $\mathcal{C}(\nabla F_i(w_t))$ to the central machine,
    \end{itemize}
        
        \item Byzantine:
        \begin{itemize}
        \item Generates $\star$ (arbitrary), and sends $\mathcal{C}(\star)$ to the central machine:  Option I,
        \item Sends $\star$ to the central machine:  Option II,
        \end{itemize}         
        \end{itemize}
    \textbf{end for}
\STATE \underline{Central Machine:}
    \begin{itemize}
        \item Sort the worker machines in a non decreasing order according to
\begin{itemize}
\item Local gradient norm:  Option I,
\item Compressed local gradient norm:  Option II,
\end{itemize}         
        \item Return the indices of the first $1-\beta$, fraction of elements as $\mathcal{U}_t$,
        \item Update model parameter: $w_{t+1} =  w_t - \frac{\gamma}{|\mathcal{U}_t|}\sum_{i\in \mathcal{U}_t} \mathcal{Q}(\nabla F_i (w_t)) $.
   \end{itemize}
  \ENDFOR
  \end{algorithmic}\label{alg:main_algo}
\end{algorithm}

\section{Robust Compressed Gradient Descent}
\label{sec:algo}
In this section, we describe a communication-efficient and robust distributed gradient descent algorithm for $\delta$-approximate compressors. The optimization algorithm we use is formally given in Algorithm~\ref{alg:main_algo}. Note that the algorithm uses a compression scheme $\mathcal{Q}(.)$ to reduce communication cost and a norm based thresholding to remove Byzantine worker nodes. The idea behing norm based thresholding is quite intuitive. Note that, if the Byzantine worker machines try to diverge the learning algorithm by increasing the norm of the local gradients; Algorithm~\ref{alg:main_algo} can identify them as outliers. Furthermore, when the Byzantine machines behave like inliers, they can not diverge the learning algorithm since they are only a few ($\alpha < 1/2$) in number. It turns out that this simple approach indeed works.

As seen in Algorithm~\ref{alg:main_algo}, robust compressed gradient descent operates under two different setting, namely {\em Option I} and {\em Option II}. Option I and II are analyzed in Sections~\ref{sec:restricted} and \ref{sec:arbitrary} respectively. For Option I, we use a $\delta$-approximate compressor along with the norm information. In particular, the worker machines send the pair denoted by $\mathcal{C}(x) = \{ \mathcal{Q}(x), \|x\|_q, \}$ where \footnote{Throughout the paper, we use $q=2$. However, any norm, i.e., $q \geq 1$ can be handled.} we have $q \geq 1$, to the center machine. $\mathcal{C}(x)$ is comprised of a scalar (norm of $x$) and a compressed vector $\mathcal{Q}(x)$. For compressors such as QSGD (\cite{qsgd}) and $\ell_1$-QSGD (\cite{errorfeed}), the quantity $\mathcal{Q}(.)$ has the norm information and hence sending the norm separately is not required.

As seen in Option I of Algorithm~\ref{alg:main_algo}, worker node $i$  compresses the local gradient $\nabla F_i(.)$ sends $\mathcal{C}(\nabla F_i(.))$ to the central machine. Adversary nodes can send arbitrary $\mathcal{C}(\star)$ to the central machine. The central machine aggregates the gradients, takes a gradient step and broadcasts the updated model for next iteration.

For Option I, we restrict to the setting where the Byzantine worker machines can send arbitrary values to the input of the compression algorithm, but they adhere to the compression algorithm. In particular, Byzantine workers can provide arbitrary values, $\star$ to the input of the compression algorithm, $\mathcal{Q}(.)$ but they correctly implement the same compression algorithm, i.e., computes $\mathcal{Q}(\star)$.

We now explain how Algorithm~\ref{alg:main_algo} tackles the Byzantine worker machines. The central machine receives the compressed gradients comprising a scalar ( $||x||_q ,q\ge1$) and a quantized vector ($\mathcal{Q}(x) $) and outputs a set of indices $\mathcal{U}$ with $|\mathcal{U}| = (1-\beta)m$. Here we employ a simple thresholding scheme on the (local) gradient norm.

Note that, if the Byzantine worker machines try to diverge the learning algorithm by increasing the norm of the local gradients; Algorithm~\ref{alg:main_algo} can identify them as outliers. Furthermore, when the Byzantine machines behave like inliers, they can not diverge the learning algorithm since $\alpha < 1/2$. In the subsequent sections, we show theoretical justification of this argument.

With Option II, we remove this restriction on Byzantine machines at the cost of slightly weakening the convergence guarantees. This is explained in Section~\ref{sec:arbitrary}. With Option II, the $i$-th local machine sends $\mathcal{C}= \{ \mathcal{Q}(\nabla F_i(w_t)), \|\mathcal{Q}(\nabla F_i(w_t))\|_q \}$ to the central machine, where $q \geq 1$.  Effectively, the $i$-th local machine just sends $ \mathcal{Q}(\nabla F_i(w_t))$ since its norm can be computed at the central machine. Byzantine workers just send arbitrary ($\star$) vector instead of compressed local gradient. Note that the Byzantine workers here do not adhere to any compression rule.

The Byzantine resilience scheme with Option II is similar to Option I except the fact that the central machine sorts the worker machines according to the norm of the compressed gradients rather than the norm of the gradients. 

\section{Distributed Learning with Restricted Adversaries}
\label{sec:restricted}
In this section, we analyze the performance of Algorithm~\ref{alg:main_algo} with \emph{Option I}. We restrict to an adversarial model in which Byzantine workers can provide arbitrary values to the input of the compression algorithm, but they adhere to the compression rule. 
Though this adversarial model is restricted, we argue that it is well-suited for applications wherein compression happens outside of worker machines. For example, Apache MXNet, a deep learning framework designed to be distributed on cloud infrastructures, uses  NVIDIA Collective Communication Library (NCCL) that employs gradient compression (see \cite{MXNet}). 
Also, in a Federated Learning setup the compression can be part of the communication protocol.  Furthermore, this can happen when worker machines are divided into groups, and each group is associated with a {\it compression unit}. As an example, cores in a multi-core processor (\cite{kangwook_core}) acting as a group of worker machines with the compression carried out by a separate processor, or servers co-located on a rack (\cite{costa2015r}) acting as a group with the compression carried out by the top-of-the-rack switch. 

%\textcolor{cyan}{Reviewer 2: Not happy with the restricted adversary}
\subsection{Main Results}
\label{sec:main_results}
We analyze Algorithm~\ref{alg:main_algo} (with Option I) and obtain the rate of the convergence under non-convex loss functions. We start with the following assumption.
\begin{assumption}
\label{asm:struct_loss}
For all $z$, the partial derivative of the loss function $f(.,z)$ with respect to the $k$-th coordinate (denoted as $\partial_k f(.,z))$ is $L_k$ Lipschitz with respect to the first argument for each $k \in [d]$, and let $\widehat{L}=\sqrt{\sum_{i=1}^d L_k^2}$. The population loss function $F(.)$ is $L_F$ smooth.
\end{assumption}
We also make the following assumption on the tail behavior of the partial derivative of the loss function.
\begin{assumption}
\label{asm:sub-exp}
(Sub-exponential gradients) For all $k \in [d]$ and $z$, the quantity $\partial_k f(w,z))$ is $v$ sub-exponential for all $w \in \mathcal{W}$.
\end{assumption}

The assumption implies that the moments of the partial derivatives are bounded. We like to emphasize that the sub-exponential assumption on gradients is fairly common (\cite{dong,su,wu}). For instance,  \cite[Proposotion 2]{dong} gives a concrete example of coordinate-wise sub-exponential gradients in the context of a regression problem. Furthermore, in \cite{dong1}, the gradients are assumed to be sub-gaussian, which is stronger than Assumption~\ref{asm:sub-exp}.

To simplify notation and for the clarity of exposition, we define the following three quantities which will be used throughout the paper.
\begin{align}
\begin{split}
\epsilon_1 &= v \sqrt{d}  ( \max \{ \frac{d}{n} \log(1+2nD\hat{L}d),\\
&\sqrt{\frac{d}{n} \log(1+2nD\hat{L}d)}\}  )  + \frac{1}{n}, \label{eqn:epsilon_def}
\end{split}
\end{align}

\small
\begin{align}
\begin{split}
    \epsilon_2 &= v \sqrt{d}  \bigg (\max \Big\{ \frac{d}{(1-\alpha)m n}\log (1+2(1-\alpha)m nD\hat{L}d), \\ &\sqrt{\frac{d}{(1-\alpha)m n}\log (1+2(1-\alpha)m nD\hat{L}d)} \Big\} \bigg ), \label{eqn:epsilon_tilde_def}
    \end{split}
\end{align} \normalsize
\begin{align}
\epsilon & = 2 \left( 1+\frac{1}{\lambda_0} \right) \bigg [ \left ( \frac{1-\alpha}{1-\beta} \right )^2 \epsilon_2^2  + \left(\frac{\sqrt{1-\delta}+\alpha + \beta}{1-\beta} \right )^2 \epsilon_1^2  \bigg ] .\label{eqn:phi_def}
\end{align}
where $\lambda_0$ is a positive constant. For intuition, one can think of $\epsilon_1 = \tilde{\mathcal{O}}(\frac{d}{\sqrt{n}})$ and $\epsilon_2 = \tilde{\mathcal{O}}(\frac{d}{\sqrt{mn}})$ as small problem dependent quantities. Assuming $\beta = c \alpha$ for a universal constant $c >1$, we have
\begin{align}
\epsilon = \tilde{\mathcal{O}} \left( d^2 \left[ \frac{\alpha^2}{n} + \frac{1-\delta}{n} + \frac{1}{mn}\right ] \right) \label{eqn:ep_order}.
\end{align}
\begin{assumption}
\label{asm:size_para}
(Size of parameter space $\mathcal{W}$) Suppose that $\| \nabla F(w) \| \leq M$ for all $w \in \mathcal{W}$. We assume that $\mathcal{W}$ contains the $\ell_2 $ ball $ \lbrace w: \|w - w_0 \| \leq c [(2 - \frac{c_0}{2})M + \sqrt{\epsilon}] \frac{F(w_0)-F(w^*)}{\epsilon}   \rbrace$, where $c_0$ is a constant, $\delta$ is the compression factor, $w_0$ is the initial parameter vector and $\epsilon$ is defined in equation~\eqref{eqn:phi_def}.
\end{assumption}
%\textcolor{cyan}{Reviewer (1,2): The assumption is contradictory. Either relax or remove the assumptions.}
We use the above assumption to ensure that the iterates of Algorithm~\ref{alg:main_algo} stays in $\mathcal{W}$. We emphasize that this is a standard assumption on the size of $\mathcal{W}$ to control the iterates for non-convex loss function. Note that, similar assumptions have been used in prior works \cite[Assumption 5]{dong}, \cite{dong1}. We point out that Assumption~\ref{asm:size_para} is used for simplicity and is not a hard requirement. We show (in the proof of Theorem~\ref{thm:non_convex}) that the iterates of Algorithm~\ref{alg:main_algo} stay in a bounded set around the initial iterate $w_0$. Also, note that the dependence  of $M$ in the final statistical rate (implicit, via diameter $D$) is logarithmic (weak dependence), as will be seen in Theorem~\ref{thm:non_convex}.
 Algorithm~\ref{alg:main_algo} for $T$ iterations with step size $\gamma = \frac{1}{L_F + \lambda_F}$ yields
%\begin{align*}
%\|w_T - w^* \| \leq \rho^{T/2}\|w_0 - w^* \| + \frac{\sqrt{2 \epsilon}}{(L_F + \lambda_F)(1-\rho^{1/2})}
%\end{align*}
%with probability greater than or equal to $1- \frac{c_1 T(1-\alpha)md}{(1+n \hat{L} D)^d} - \frac{c_2 dT}{(1+(1-\alpha)m n \hat{L} D)^d}$, provided $\delta > \delta_0 + 6\alpha - 4\alpha^2$. Here $ \rho = 2\left (1-\frac{2 L_F \lambda_F}{(\lambda_F + L_F)^2} \right)$ and $\epsilon$ is defined in equation~\eqref{eqn:phi_def}.
%\end{theorem}
%\begin{remark}
%Using AM-GM inequality, we get $\rho \in (0,1)$, and hence the theorem implies an exponential rate of convergence of Algorithm~\ref{alg:main_algo}. Note that, it is well known that gradient descent converges exponentially fast for a strongly convex and smooth function. \cite{dong} shows that the rate of convergence is preserved in Byzantine setting. Theorem~\ref{thm:convex} shows that  the exponential convergence holds even with compressed gradients in a Byzantine setting.
%\end{remark}
%
%\begin{remark}
%Running Algorithm~\ref{alg:main_algo} for $T \geq \mathcal{O}\left( \log \frac{L_F + \lambda_F}{\sqrt{\epsilon}}\|w_0 - w^* \| \right)$ iterations, we obtain a solution $ \widehat{w}=w_T$ such that $\|\widehat{w} - w^* \|^2 = \mathcal{O}(\epsilon)$.
%%
%%\swa{This is very confusing. We define $\epsilon$ in~\eqref{eqn:phi_def} and then argue that order-wise one has~\eqref{eqn:ep_order}. Then, what do we mean by "$\epsilon$ matches equation~\eqref{eqn:ep_order}"? It should always match, right? What message do we want to convey here?}
%\end{remark}

%\subsubsection{Non-convex Loss}

We provide the following rate of convergence to a critical point of the (non-convex) population loss function $F(.)$. 
\begin{theorem}
\label{thm:non_convex}
Suppose Assumptions~\ref{asm:struct_loss}, \ref{asm:sub-exp} and \ref{asm:size_para} hold, and $\alpha \leq \beta < 1/2$. For sufficiently small constant $c$, we choose the step size $\gamma = \frac{c}{L_F}$. Then, running Algorithm~\ref{alg:main_algo} for $T = C_3 \frac{L_F(F(w_0) - F(w^*))}{\epsilon}$ iterations yields
\begin{align*}
\min_{t=0,\ldots, T}\| \nabla F(w_t) \|^2 \leq C\, \epsilon,
\end{align*}
with probability greater than or equal to $1- \frac{c_1(1-\alpha)md}{(1+n \hat{L} D)^d} - \frac{c_2 d}{(1+(1-\alpha)m n \hat{L} D)^d}$, provided the compression factor satisfies $\delta > \delta_0 + 4\alpha -9\alpha^2 + 4\alpha^3$, where $\delta_0 = \left(1-\frac{(1-\beta)^2}{1+\lambda_0} \right)$ and $\lambda_0$ is a (sufficiently small) positive constant.
\end{theorem}

A few remarks are in order. In the following remarks, we fix the dimension $d$, and discuss the dependence of $\epsilon$ on $(\alpha, \delta, n, m)$.
%\textcolor{cyan}{Reviewer 3: Worried about $d^2$ dependency. High dimensional problems. }

\begin{remark}
(Rate of Convergence) Algorithm~\ref{alg:main_algo} with $T$ iterations yields
\begin{align*}
    \min_{t=0,., T}\| \nabla F(w_t) \|^2 &\leq \frac{C_1 L_F(F(w_0) - F(w^*))}{T+1} + C_2 \epsilon
\end{align*}
with high probability. We see that Algorithm~\ref{alg:main_algo} converges at a rate of $\mathcal{O}(1/T)$, and finally plateaus at an error floor of $\epsilon$. Note that the rate of convergence is same as \cite{dong}. Hence, even with compression, the (order-wise) convergence rate is unaffected.
\end{remark}
\begin{remark}
We observe, from the definition of $\epsilon$ that the price for compression is $\tilde{\mathcal{O}}(\frac{1-\delta}{n})$.
\end{remark}

\begin{remark}
Substituting $\delta = 1$ (no compression) in $\epsilon$, we get $\epsilon = \tilde{\mathcal{O}}(\frac{\alpha^2}{n}+\frac{1}{mn})$, which matches the (statistical) rate of \cite{dong}.  A simple \emph{norm based thresholding} operation is computationally simple and efficient in the high dimensional settings compared to  the coordinate wise median and trimmed mean to achieve robustness and obtain the the same statistical error and iteration complexity as \cite{dong} 
\end{remark}
%\textcolor{cyan}{Reviewer 2: Not Happy with the statement that computing  median is complicated!!   }

\begin{remark}
When the compression factor $\delta$ is large enough, satisfying $\delta \geq 1 -\alpha^2$, we obtain  $\epsilon = \tilde{\mathcal{O}}(\frac{\alpha^2}{n}+\frac{1}{mn})$. In this regime, the iteration complexity and the final statistical error of Algorithm~\ref{alg:main_algo} is order-wise identical to the setting with no compression \cite{dong}. We emphasize here that a reasonable high $\delta$ is often observed in practical applications like training of neural nets \cite[Figure 2]{errorfeed}.
\end{remark}
%\textcolor{cyan}{Reviewer 2 : Not happy that compression does not affect the ultimate performance. WHY?? Stich et al}
\begin{remark}
(Optimality) For a distributed mean estimation problem, Observation 1 in \cite{dong} implies that any algorithm will yield an (statistical) error of $\Omega(\frac{\alpha^2}{n} + \frac{d}{mn})$. Hence, in the regime where $\delta \geq 1-\alpha^2$, our error-rate is optimal. 
\end{remark}
%One possible algorithm is to compress the gradients and run Algorithm~\ref{alg:main_algo}. \swa{Is this not implied due to the use of `any'?}This implies Using \cite{dong}, this would yield a statistical error of $\Omega(\frac{\alpha^2}{n} + \frac{d}{mn})$. In the regime where $\delta$ is a fixed constant, this error matches $\epsilon$ for a fixed $d$. Hence the dependence of $\epsilon$ on $\alpha$, $n$ and $m$ are \swa{are $\rightarrow$ is?} optimal and unimprovable \swa{optimal implies unimprovable, right?}. \swa{In general, this is written in a very confusing way. Please restructure. High level message is very simple: our error-rate is optimal in the regime when $\delta$ and $d$ are fixed.}
\begin{remark}
For the convergence of Algorithm~\ref{alg:main_algo}, we require $\delta > \delta_0 + 4\alpha - 9\alpha^2 + 4\alpha^3$, implying that our analysis will not work if $\delta$ is very close to $0$. Note that a very small $\delta$ does not give good accuracy in practical applications  \cite[Figure 2]{errorfeed}. Also, note that, from the definition of $\delta_0$, we can choose $\lambda_0$ sufficiently small at the expense of increasing the multiplicative constant in $\epsilon$ by a factor of $1/\lambda_0$. Since the error-rate considers asymptotic in $m$ and $n$, increasing a constant factor is insignificant. A sufficiently small $\lambda_0$ implies $\delta_0 = \mathcal{O}(2\beta)$, and hence we require $\delta > 4\alpha + 2\beta$ (ignoring the higher order dependence).
\end{remark}

\begin{remark}
The requirement $\delta > 4\alpha + 2\beta$ can be seen as a trade-off between the amount of compression and the fraction of adversaries in the system. As $\alpha$ increases, the amount of (tolerable) compression decreases and vice versa.
\end{remark}

\section{Distributed Optimization with Arbitrary Adversaries}
\label{sec:arbitrary}

%\begin{algorithm}[t!]
%  \caption{Robust Compressed Gradient Descent for Arbitrary Adversary}
%  \begin{algorithmic}[1]
% \STATE  \textbf{Input:} Step size $\gamma$, Compressor $\tilde{\mathcal{C}}(.)$, $\beta < 1$
% \STATE \textbf{Initialize:} Initial iterate $w_0$ \\
%  \FOR{$t=0,1, \ldots, T-1 $}
%  \STATE \underline{Central machine:} broadcasts $w_t$  \\
%  \textbf{ for $ i \in [m]$ do in parallel}\\
%  \STATE \underline{$i$-th worker machine:} 
%    \begin{itemize}
%        \item (Non-Byzantine) computes $\nabla F_i (w_t)$; sends $\tilde{\mathcal{C}}(\nabla F_i(w_t))$ to the central machine
%        \item (Byzantine) sends $\star$ (arbitrary) to the central machine 
%        \end{itemize}
%    \textbf{end for}
%\STATE \underline{Central Machine:}
%    \begin{itemize}
%        \item Sort the norm of compressed local gradients in a non decreasing order
%        \item Return the indices of the first $1-\beta$ fraction of elements as $\mathcal{U}_t$.
%        \item Update model parameter: $w_{t+1} = \Pi_{\mathcal{W}} \left ( w_t - \frac{\gamma}{|\mathcal{U}_t|}\sum_{i\in \mathcal{U}_t} \mathcal{Q}(\nabla F_i (w_t)) \right )$
%   \end{itemize}
%  \ENDFOR
%  \end{algorithmic}\label{alg:main_algo}
%\end{algorithm}

In this section we remove the assumption of restricted adversary (as in Section~\ref{sec:restricted}) and make the learning algorithm robust to the adversarial effects of both the computation and compression unit. In particular, here we consider Algorithm ~\ref{alg:main_algo} with Option II. Hence, the Byzantine machines do not need to adhere to the mandated compression algorithm.

In Option II, the worker machines send $\mathcal{Q}(\nabla F_i (w_t))$ to the center machine. The center machine computes its norm, and discards the top $\beta$ fraction of the worker machines having largest norm. Note that it is crucial that the center machine computes the norm of $\mathcal{Q}(\nabla F_i (w_t))$, instead of asking the worker machine to send it (similar to Option I). Otherwise, a Byzantine machine having a large $\|\mathcal{Q}(x)\|_q$ can (wrongly) report a small value of $\|\mathcal{Q}(x)\|_q$, gets selected in the trimming phase and influences (or can potentially diverge) the optimization algorithm. Hence, the center needs to compute $\|\mathcal{Q}(x)\|_q$ to remove such issues.

Although this framework is more general in terms of Byzantine attacks, however, in this setting, the statistical error-rate of our proposed algorithm is slightly weaker than that of Theorem~\ref{thm:non_convex}. Furthermore, the $(\delta,\alpha)$ trade-off is stricter compared to Theorem~\ref{thm:non_convex}.

\subsection{Main Results}
\label{sec:main_results_one}

%\subsubsection{Convex Loss}
%We now assume that the loss function $F(.)$ is convex, and our goal is to estimate $w^*$. We assume that $w^* \in \mathcal{W}$. We have the following result.
%\begin{theorem}
%\label{thm:convex1}
%Suppose that Assumptions~\ref{asm:struct_loss} and \ref{asm:sub-exp} hold and the loss function $F(.)$ is $\lambda_F$ strongly convex. Also suppose that $\alpha \leq \beta$. Then, running Algorithm~\ref{alg:main_algo} for $T$ iterations with step size $\gamma = \frac{1}{L_F + \lambda_F}$ yields
%\begin{align*}
%\|w_T - w^* \| \leq \rho^{T/2}\|w_0 - w^* \| + \frac{\sqrt{2 \epsilon}}{(L_F + \lambda_F)(1-\rho^{1/2})}
%\end{align*}
%with probability greater than or equal to $1- \frac{c_1 T(1-\alpha)md}{(1+n \hat{L} D)^d} - \frac{c_2 dT}{(1+(1-\alpha)m n \hat{L} D)^d}$, provided $\delta > \delta_0 + 6\alpha - 4\alpha^2$. Here $ \rho = 2\left (1-\frac{2 L_F \lambda_F}{(\lambda_F + L_F)^2} \right)$ and $\epsilon$ is defined in equation~\eqref{eqn:phi_def}.
%\end{theorem}
%
%
%
%
%
%\subsubsection{Non-convex Loss}
We continue to assume that the population loss function $F(.)$ is smooth and  non-convex and analyze Algorithm~\ref{alg:main_algo} with Option II. We have the following result. For the clarity of exposition, we define the following quantity which will be used in the results of this section:
\small
\begin{align*}
\widetilde{\epsilon} = 2(1+\frac{1}{\lambda_0})\bigg(\bigg(\frac{(1+\beta)\sqrt{1-\delta} +\alpha+\beta}{1-\beta} \bigg)^2\epsilon^2_1+(\frac{1-\alpha}{1-\beta})^2\epsilon_2^2 \bigg). \label{eqn:epsilon_tilde}
\end{align*}\normalsize 
Comparing $\widetilde{\epsilon}$ with $\epsilon$, we observe that $\widetilde{\epsilon} > \epsilon$. Also, note that,
\begin{align}
\widetilde{\epsilon} = \tilde{\mathcal{O}}\left( d^2 \left[ \frac{\alpha^2}{n} + \frac{1-\delta}{n} + \frac{1}{mn} \right] \right),
\end{align}
which suggests that $\widetilde{\epsilon}$ and $\epsilon$ are order-wise similar.
We have the following assumption, which parallels Assumption~\ref{asm:size_para}, with $\epsilon$ replaced by $\widetilde{\epsilon}$.
\begin{assumption}
\label{asm:size_para_one}
(Size of parameter space $\mathcal{W}$) Suppose that $\| \nabla F(w) \| \leq M$ for all $w \in \mathcal{W}$.  We assume that $\mathcal{W}$ contains the $\ell_2 $ ball $ \lbrace w: \|w - w_0 \| \leq c [(2 - \frac{c_0}{2})M + \sqrt{\widetilde{\epsilon}}] \frac{F(w_0)-F(w^*)}{\widetilde{\epsilon}}   \rbrace$, where $c_0$ is a constant, $\delta$ is the compression factor and $\widetilde{\epsilon}$ is defined in equation~\eqref{eqn:epsilon_tilde}.
\end{assumption}
\begin{theorem}
\label{thm:non_convex1}
Suppose Assumptions~\ref{asm:struct_loss},\ref{asm:sub-exp} and \ref{asm:size_para_one} hold, and $\alpha \leq \beta < 1/2$. For sufficiently small constant $c$, we choose the step size $\gamma = \frac{c}{L_F}$. Then, running Algorithm~\ref{alg:main_algo} for $T = C_3 \frac{L_F(F(w_0) - F(w^*))}{\widetilde{\epsilon}}$ iterations yields
\begin{align*}
\min_{t=0,\ldots, T}\| \nabla F(w_t) \|^2 \leq C\, \widetilde{\epsilon},
\end{align*}
with probability greater than or equal to $1- \frac{c_1(1-\alpha)md}{(1+n \hat{L} D)^d} - \frac{c_2 d}{(1+(1-\alpha)m n \hat{L} D)^d}$, provided the compression factor satisfies $\delta > \widetilde{\delta_0} + 4\alpha -8\alpha^2 + 4\alpha^3$, where $\widetilde{\delta_0} = \left(1-\frac{(1-\beta)^2}{(1+\beta)^2 (1+\lambda_0)} \right)$ and $\lambda_0$ is a (sufficiently small) positive constant.
\end{theorem}

\begin{remark}
The above result and their consequences resemble that of Theorem~\ref{thm:non_convex}. Since $\widetilde{\epsilon} > \epsilon$, the statistical error-rate in Theorem~\ref{thm:non_convex1} is strictly worse than that of Theorem~\ref{thm:non_convex} (although order-wise they are same).
\end{remark}
\begin{remark}
 Note that the definition of $\delta_0$ is different than in Theorem~\ref{thm:non_convex}. For a sufficiently small $\lambda_0$, we see $\widetilde{\delta_0} = \mathcal{O}(4 \beta)$, which implies we require $\delta > 4\beta + 4\alpha$ for the convergence of Theorem~\ref{thm:non_convex1}. Note that this is a slightly strict requirement compared to Theorem~\ref{thm:non_convex}. In particular, for a given $\delta$, Algorithm~\ref{alg:main_algo} with Option II can tolerate less number of Byzantine machines compared to Option I.
\end{remark}
%\textcolor{cyan}{Reviewer 1 : Is the rate $\delta > 4\beta + 4\alpha$ optimal or else comment?}
\begin{remark}
 The result in Theorem~\ref{thm:non_convex1} is applicable for arbitrary adversaries, whereas Theorem~\ref{thm:non_convex} relies on the adversary being restrictive. Hence, we can view the limitation of Theorem~\ref{thm:non_convex1} (such as worse statistical error-rate and stricter $(\delta,\alpha)$ trade-off) as a price of accommodating arbitrary adversaries.
% \textcolor{cyan}{Reviewer 1 : Bad rate $O(d\alpha)$!! Check back with the comparison}
\end{remark}

\section{Byzantine Robust Distributed Learning with Error Feedback}
\label{sec:error_feedback}
We now investigate the role of error feedback \cite{errorfeed} in distributed learning with Byzantine worker machines. We stick to the formulation of Section~\ref{sec:intro}.

In order to address the issues of convergence for sign based algorithms (like \emph{signSGD}), \cite{errorfeed} proposes a class of optimization algorithms that uses \emph{error feedback}. In this setting, the worker machine locally stores the error between the actual local gradient and its compressed counterpart. Using this as feedback, the worker machine adds this error term to the compressed gradient in the subsequent iteration. Intuitively, this accounts for correcting the the direction of the local gradient. The error-feedback has its roots in some of the classical communication system like ``delta-sigma'' modulator and adaptive modulator (\cite{haykin1994introduction}).

\begin{algorithm}[t!]
  \caption{Distributed Compressed Gradient Descent with Error Feedback}
  \begin{algorithmic}[1]
 \STATE  \textbf{Input:} Step size $\gamma$, Compressor $\mathcal{Q}(.)$, parameter $\beta (> \alpha)$.
 \STATE \textbf{Initialize:} Initial iterate $w_0$, $e_i(0) = 0$ $\forall$ $i \in [m]$ \\
  \FOR{$t=0,1, \ldots, T-1 $}
  \STATE \underline{Central machine:} sends $w_t$ to all worker \\
  \textbf{for $i \in [m]$ do in parallel}\\
  \STATE \underline{$i$-th non-Byzantine worker machine:}
    \begin{itemize} 
        \item computes $p_i(w_t) = \gamma \nabla F_i(w_t) + e_i(t)$
        \item sends $\mathcal{Q}(p_i(w_t))$ to the central machine
        \item computes $e_i(t+1) = p_i(w_t) - \mathcal{Q}(p_i(w_t))$ 
    \end{itemize}
    \STATE  \underline{Byzantine worker machine:}
    \begin{itemize}
    \item sends $\star$ to the central machine.
    \end{itemize}
  \STATE  \underline{At Central machine:}
    \begin{itemize}
    \item sorts the worker machines in non-decreasing order according to $\|\mathcal{Q}(p_i(w_t))\|$.
    \item returns the indices of the first $1-\beta$ fraction of elements as $\mathcal{U}_t$.
        \item $w_{t+1} =  w_t - \frac{\gamma}{|\mathcal{U}_t|}\sum_{i \in \mathcal{U}_t} \mathcal{Q}(p_i(w_t)) $
    \end{itemize}
  \ENDFOR
  \end{algorithmic}\label{alg:err_feed}
\end{algorithm}

We analyze the distributed error feedback algorithm in the presence of Byzantine machines. The algorithm is presented in Algorithm~\ref{alg:err_feed}. We observe that here the central machine sorts the worker machines according to the norm of the compressed local gradients, and ignore the largest $\beta$ fraction.

Note that, similar to Section~\ref{sec:arbitrary}, we handle arbitrary adversaries. In the subsequent section, we show (both theoretically and experimentally) that the statistical error rate of Algorithm~\ref{alg:err_feed} is smaller than Algorithm~\ref{alg:main_algo}.

\subsection{Main Results}

In this section we analyze Algorithm~\ref{alg:err_feed} and obtain the rate of the convergence under non-convex smooth  loss functions. Throughout the section, we select $\gamma$ as the step size and assume that Algorithm~\ref{alg:err_feed} is run for $T$ iterations. We start with the following assumption.

\begin{assumption}
\label{asm:bounded_moment}
For all non-Byzantine worker machine $i$, the local loss functions $F_i(.)$ satisfy $\|\nabla F_i(x)\|^2 \leq \sigma^2$, where $x \in \{w_j\}_{j=0}^T$, and $\{w_0,\ldots,w_T\}$ are the iterates of Algorithm~\ref{alg:err_feed}.
\end{assumption}

Note that several learning problems satisfy the above condition (with high probability). In Appendix (Section~\ref{sec:bounded}), we consider the canonical problem of least squares and obtain an expression of $\sigma^2$ with high probability.

Note that since $F_i(.)$ can be written as loss over data points of machine $i$, we observe that the bounded gradient condition is equivalent to the bounded second moment condition for SGD, and have featured in several previous works, see, e.g., \cite{karimireddy2019scaffold}, \cite{mayekar2020limits}. Here, we are using all the data points and (hence no randomness over the choice of data points) perform gradient descent instead of SGD. Also, note that Assumption~\ref{asm:bounded_moment} is weaker than the bounded second moment condition since we do not require $\|\nabla F_i(x)\|^2$ to be bounded for all $x$; just when $x \in \{w_j\}_{j=0}^T$.

We also require the following assumption on the size of the parameter space $\mathcal{W}$, which parallels Assumption~\ref{asm:size_para} and \ref{asm:size_para_one}.

\begin{assumption}
\label{asm:size_para_err}
(Size of parameter space $\mathcal{W}$) Suppose that $\| \nabla F(w) \| \leq M$ for all $w \in \mathcal{W}$.  We assume that $\mathcal{W}$ contains the $\ell_2 $ ball $ \lbrace w: \|w - w_0 \| \leq \gamma r^* T  \rbrace$, where
\begin{align*}
\begin{split}
    r^* &= \epsilon_2 + M + \frac{6 \beta (1+\sqrt{1-\delta})}{(1-\beta)} \left(  \epsilon_1 +  M + \sqrt{\frac{3(1-\delta)}{\delta}} \sigma \right) \\
    &+ \sqrt{\frac{12(1-\delta)}{\delta}} \sigma,
    \end{split}
\end{align*}
and $(\epsilon_1,\epsilon_2)$ are defined in equations~\eqref{eqn:epsilon_def} and \eqref{eqn:epsilon_tilde_def} respectively.
\end{assumption}

Similar to Assumption~\ref{asm:size_para} and \ref{asm:size_para_one}, we use the above assumption to ensure that the iterates of Algorithm~\ref{alg:err_feed} stays in $\mathcal{W}$, and we emphasize that this is a standard assumption to control the iterates for non-convex loss function (see \cite{dong,dong1}).

To simplify notation and for the clarity of exposition, we define the following quantities which will be used in the main results of this section.
\small
\begin{align}
\begin{split}
    \Delta_1 &= \frac{9 (1+\sqrt{1-\delta})^2 }{2c(1-\beta)^2} \left[\alpha^2 + \beta^2 + (\beta - \alpha)^2 \right]\\
    & \times\left(  \epsilon_1^2 + \frac{3(1-\delta)}{\delta} \sigma^2 \right)+ \frac{50}{c}\epsilon_2^2, \label{eqn:del_one}
    \end{split}
\end{align}

\begin{align}
\begin{split}
    \Delta_2 &= \frac{ L^2 }{2}  \frac{3(1-\delta)\sigma^2}{c\delta} + \frac{2L \epsilon_2^2}{c}  + \left(\frac{1}{2}+L\right) \frac{9 (1+\sqrt{1-\delta})^2 }{c(1-\beta)^2}\\
    &\times \left[\alpha^2 + \beta^2 + (\beta - \alpha)^2 \right] \left( \epsilon_1^2+ \frac{3(1-\delta)}{\delta} \sigma^2 \right), \label{eqn:del_two}
    \end{split}
\end{align} \normalsize
\begin{align}
    \Delta_3 =  (\frac{ L^2}{100} + 25  L^2 ) \frac{3(1-\delta)\sigma^2}{c\delta}, \label{eqn:del_three}
\end{align}
where $c$ is a universal constant.

We show the following rate of convergence to a critical point of the population loss function $F(.)$. 
\begin{theorem}
\label{thm:non_convex_err}
Suppose Assumptions~\ref{asm:struct_loss}, \ref{asm:sub-exp}, \ref{asm:bounded_moment} and \ref{asm:size_para_err} hold, and $\alpha \leq \beta < 1/2$. Then, running Algorithm~\ref{alg:main_algo} for $T$ iterations with step size $\gamma$ yields
\begin{align*}
\min_{t=0,\ldots, T}\| \nabla F(w_t) \|^2 \leq \frac{F(w_0)-F^*}{c\gamma(T+1)} + \Delta_1 + \gamma \Delta_2 + \gamma^2 \Delta_3,
\end{align*}
with probability greater than or equal to $1- \frac{c_1(1-\alpha)md}{(1+n \hat{L} D)^d} - \frac{c_2 d}{(1+(1-\alpha)m n \hat{L} D)^d}$, provided the compression factor satisfies $\frac{ (1+\sqrt{1-\delta})^2 }{(1-\beta)^2} \left[\alpha^2 + \beta^2 + (\beta - \alpha)^2 \right] < 0.107$. Here $\Delta_1,\Delta_2$ and $\Delta_3$ are defined in equations~\eqref{eqn:del_one},\eqref{eqn:del_two} and \eqref{eqn:del_three} respectively.
\end{theorem}
\begin{remark}
(Choice of Step Size $\gamma$) Substituting $\gamma = \frac{1}{\sqrt{T+1}}$, we obtain
\small
\begin{align*}
     \min_{t=0,\ldots,T}\|\nabla F(w_t)\|^2 &\leq \frac{F(w_0)-F^*}{c\sqrt{T+1}} + \Delta_1 + \frac{\Delta_2}{\sqrt{T+1}} + \frac{\Delta_3}{T+1},
\end{align*}\normalsize
with high probability. Hence, we observe that the quantity associated with $\Delta_3$ goes down at a considerably faster rate ($\mathcal{O}(1/T)$) than the other terms and hence can be ignored, when $T$ is large. 
\end{remark}

\begin{remark}
Note that when no Byzantine worker machines are present, i.e., $\alpha = \beta = 0$, we obtain 
\begin{align*}
\Delta_1 &= \frac{50}{c}\epsilon_2^2, \,\, \,\,\, \Delta_2 = \frac{ L^2 }{2}  \frac{3(1-\delta)\sigma^2}{c\delta} + \frac{2L \epsilon_2^2}{c}, \\
\Delta_3 & = (\frac{ L^2}{100} + 25  L^2 ) \frac{3(1-\delta)\sigma^2}{c\delta}.
\end{align*}
Additionally, if $\delta = \Theta(1)$ (this is quite common in applications like training of neural nets, as mentioned earlier), we obtain $\Delta_2 = C ( L^2 \sigma^2 + L \epsilon_2^2)$, and $\Delta_3 = C_1 L^2$. Substituting $\epsilon_2 =\mathcal{O}(\frac{d}{\sqrt{mn}})$ and for a fixed $d$, the upper bound in the above theorem is order-wise identical to that of standard SGD in a population loss minimization problem under similar setting \cite{bubeck2015convex},\cite{hardt2021patterns},\cite[Remark 4]{errorfeed}. 
\end{remark}

\begin{remark} (No compression setting)
In the setting, where $\delta = 1$ (no compression), we obtain
\begin{align*}
    \Delta_1 = \mathcal{O}\left[ d^2 \left( \frac{\alpha^2}{n} + \frac{1}{m n} \right) \right],
\end{align*}
and 
\begin{align*}
    \Delta_2 = \mathcal{O}\left[ d^2 L \left( \frac{\alpha^2}{ n} + \frac{1}{mn}\right) \right],
\end{align*}
and $\Delta_3 = 0$. The statistical rate (obtained by making $T$ sufficiently large) of the problem is $\Delta_1$, and this rate matches exactly to that of \cite{dong}. Hence, we could recover the optimal rate without compression. Furthermore, this rate is optimal in $(\alpha,m,n)$ as shown in \cite{dong}.
\end{remark}

\begin{remark} [Comparison with Algorithm~\ref{alg:main_algo}]
In numerical experiments (Section~\ref{sec:experiment}), we compare the performance of Algorithm~\ref{alg:err_feed} with the one without error feedback (Algorithm~\ref{alg:main_algo}). We keep the experiment setup (ex., learning rate, compression) identical for both the algorithms, and compare their performance (see Figure~\ref{fig:regwfd1000}). We observe that the convergence of Algorithm~\ref{alg:err_feed} with error feedback is faster than Algorithm~\ref{alg:main_algo}, which is intuitive since error feedback helps in correcting the direction of the local gradient.
\end{remark}

\section{Experiments}
\label{sec:experiment}

% \textcolor{cyan}{ Reviewer 3: The experiments section is somewhat lacking. I have the following recommendations for the authors:
%\begin{itemize}
% Please provide training loss-error curves for the case of no byzantine workers and the case of byzantine workers with no alleviation mechanism in each experiment. This would give a better idea as to how much does the proposed algorithm alleviate the byzantine attacks.  Please provide details as to what choices of $\delta$ and $\beta$ were used for the experiments.
%}
%\textcolor{cyan}{ Reviewer 3:
% It would be particularly useful to get more insight on how the proposed algorithms perform under different choices of $\beta$ and $\delta $. Training loss curves for the proposed algorithm for different choices of $ \beta$ and $\delta$ would be insightful in that regard.How do the test error curves look like for the MNIST experiment? What is the limit of $\alpha$ that the algorithms support empirically? In particular, at the limit does the algorithm diverge or converges to a wrong set of minima or has considerably reduced speed of convergence?}

%\vspace{-2mm}

In this section we validate the correctness of our proposed algorithms for linear regression problem and training ReLU network. In all the experiments, we choose the following compression scheme: given any $x \in \mathbb{R}^d$, we report  $\mathcal{C}(x)=\{\frac{\|x\|_1}{d},\sign(x)\}$ where $\sign(x)$ serves as the quantized vector and  $\frac{\|x\|_1}{d}$ is the scaling factor. %\swa{We need to agree on what is the compressor $\mathcal{Q}$ and what is the tuple $\mathcal{C}$. It is okay if $\mathcal{C}$ is redundant, we just define it for simplicity.} 
All the reported results are averaged over 20 different runs.

\begin{figure*}[t!]
    \centering
    % \vspace{-17 mm}
    \subfloat[ $\#$ Byzantine nodes=10]{{\includegraphics[width=3.5cm]{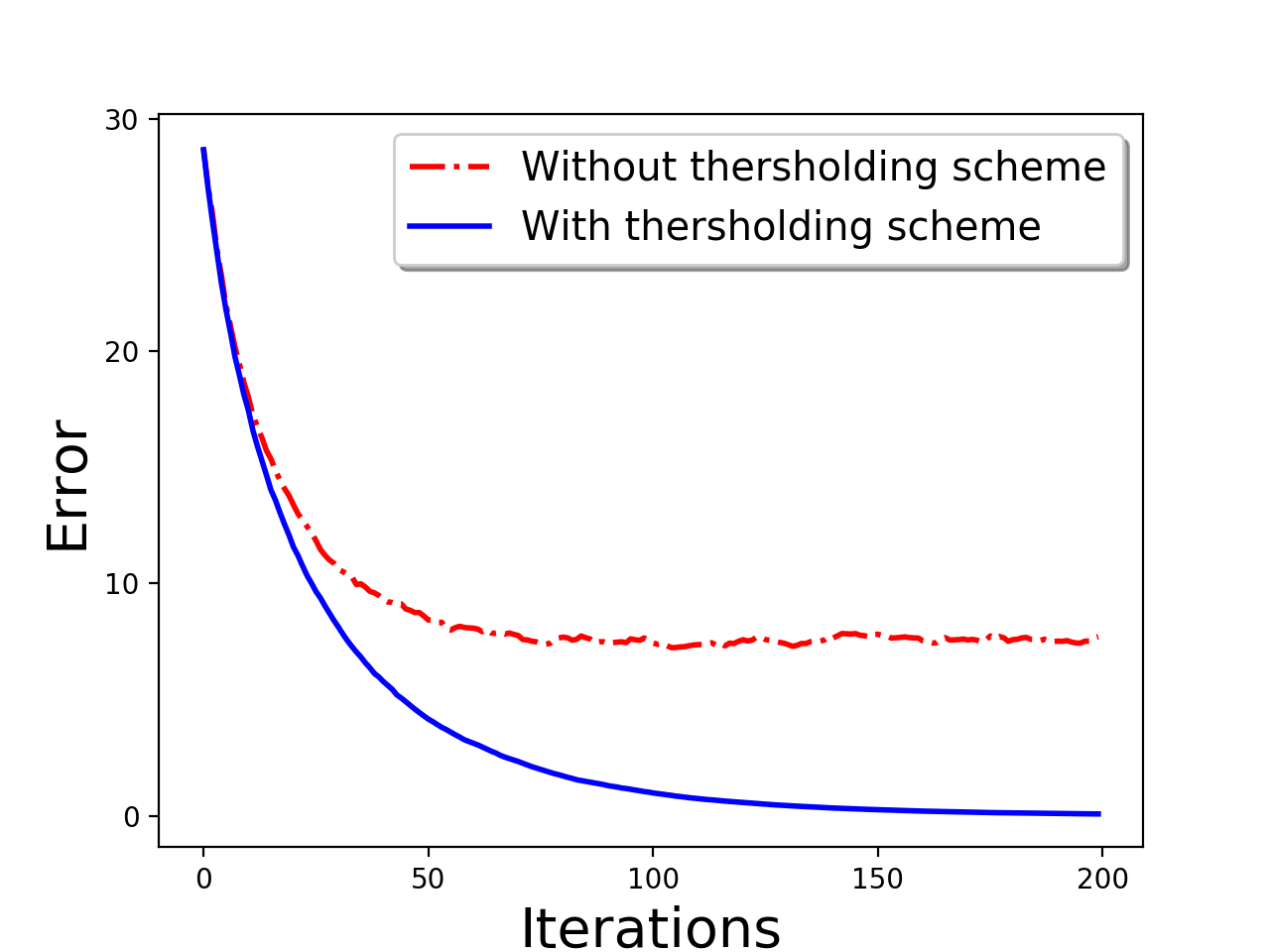} }}%
    \qquad
    \subfloat[$\#$ Byzantine nodes=20]{{\includegraphics[width=3.5cm]{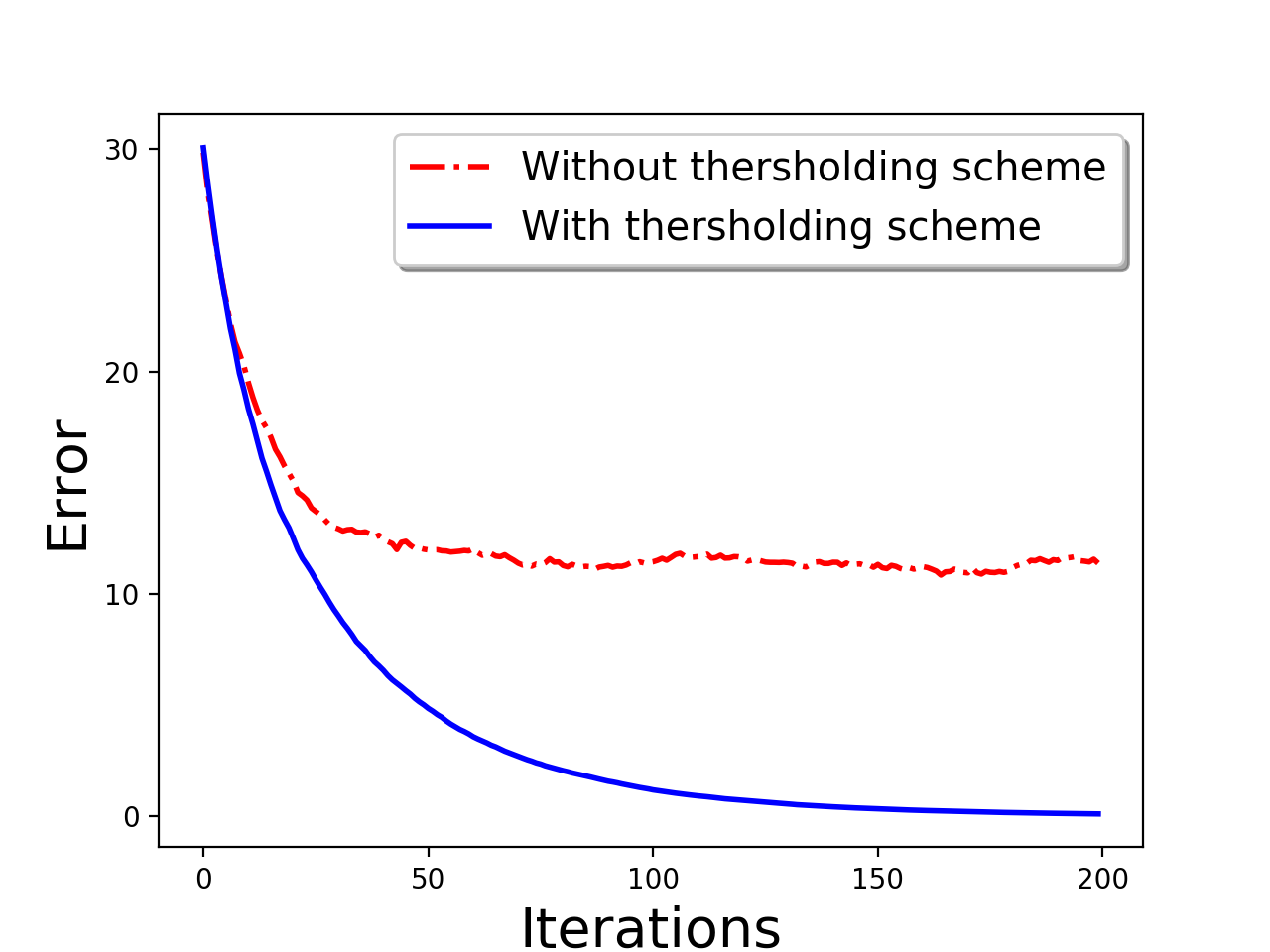} }}%
    \qquad
    \subfloat[$\#$ Byzantine nodes=10]{{\includegraphics[width=3.5cm]{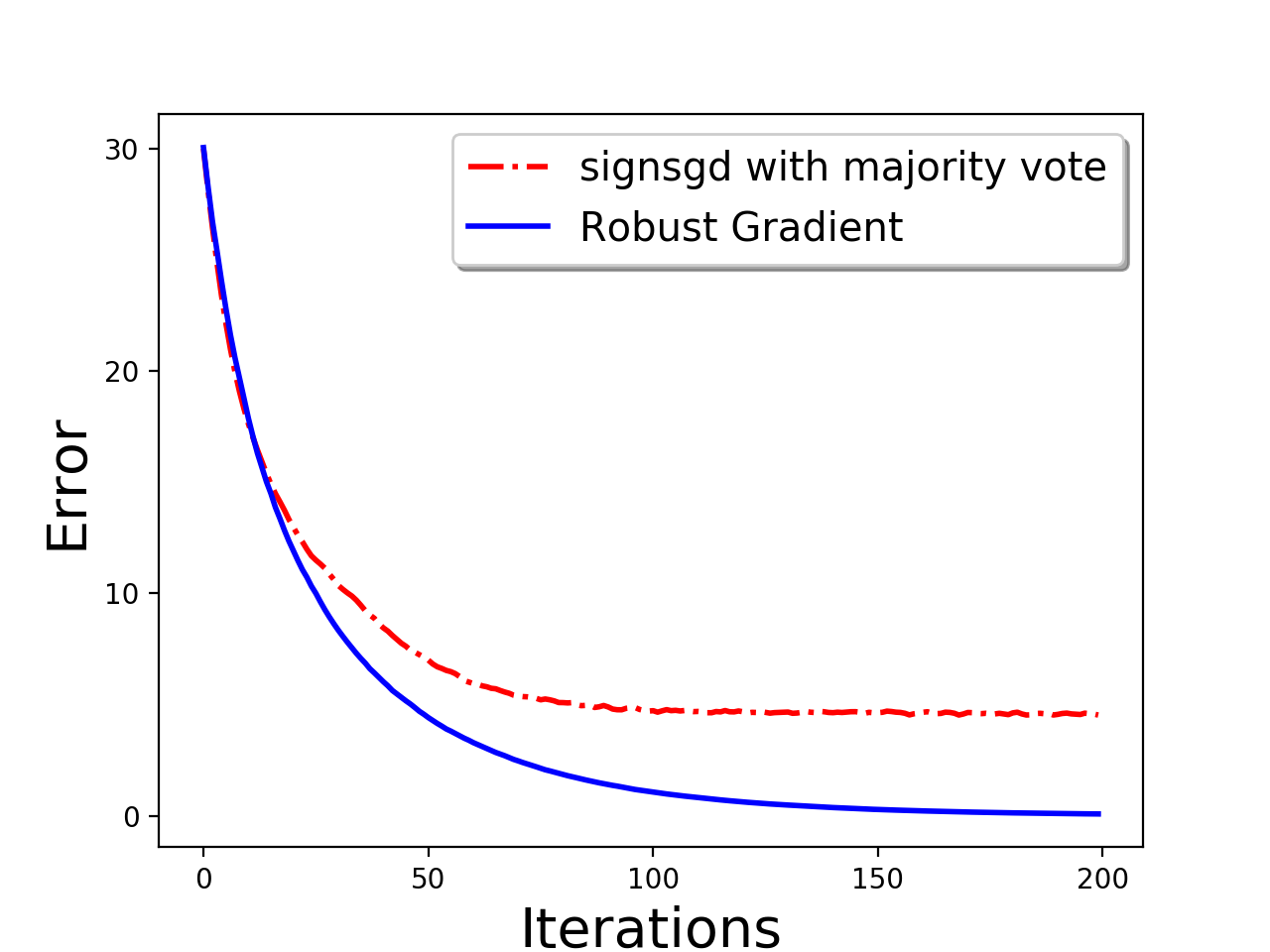} }}%
    \qquad
    \subfloat[$\#$ Byzantine nodes=20]{{\includegraphics[width=3.5cm]{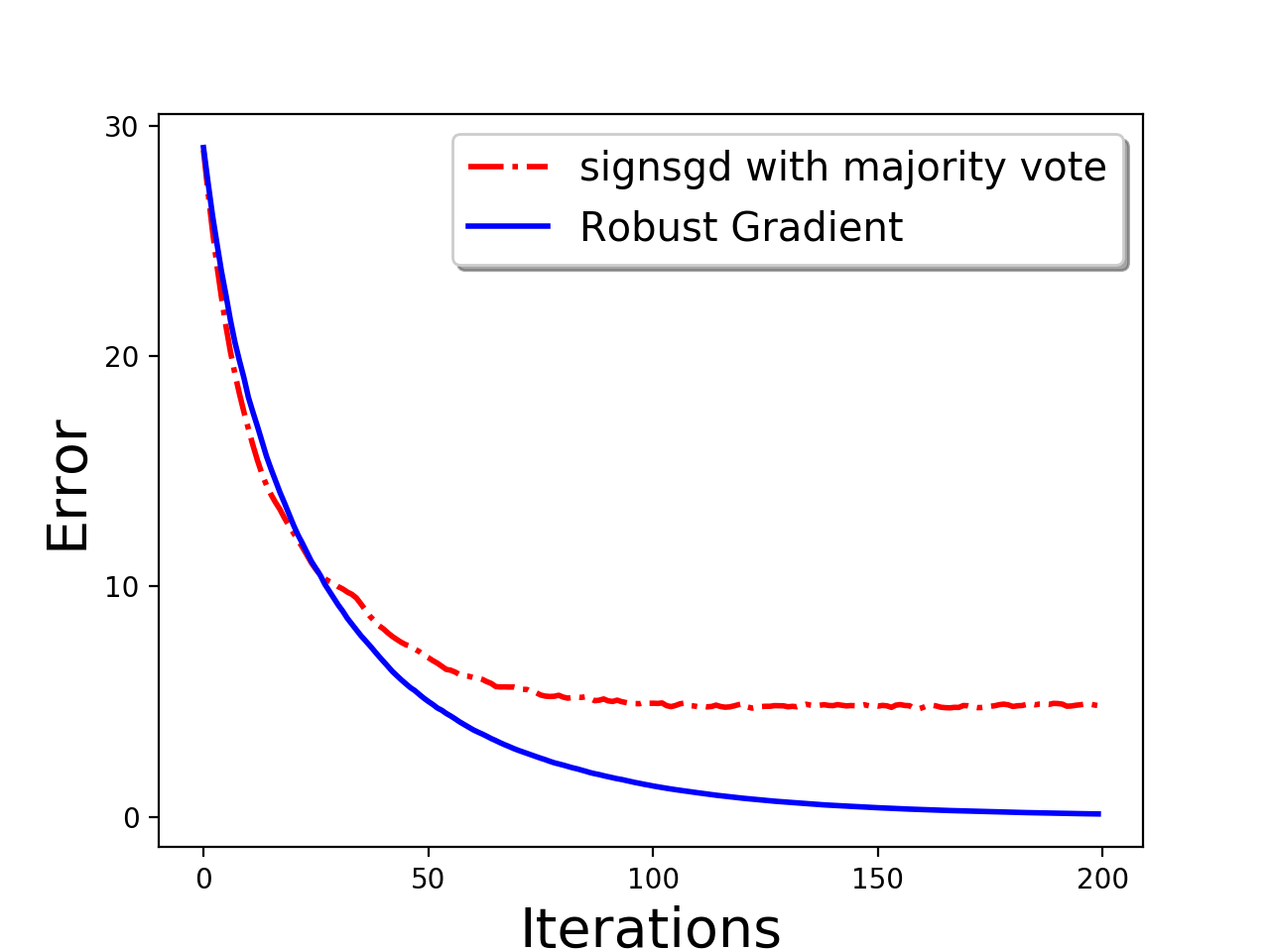} }}%
    \caption{Comparison of Robust Compressed Gradient  Descent with and without thresholding scheme in a regression problem (a,b). The plots show better convergence with thersholding. Comparison of Robust Compressed Gradient  Descent with majority vote based {\it signSGD} \cite{anima}  in regression Problem. The plots (c,d) show  better convergence with thresholding in comparison to the majority vote based robustness of \cite{anima} }%
    \label{fig:companima1000}
\end{figure*}
First we consider a least square regression problem $w^*=\argmin_{w}\|Aw -b\|_2$. For the regression problem we generate  matrix $A \in \mathbb{R}^{N\times d}$, vector $w^* \in \mathbb{R}^d$ by sampling each item independently from standard normal distribution and set $b=Aw^*$. 
%\swa{It will be good to give more details. Is each entry of $A$ and $w^*$ chosen iid with standard normal?} 
Here we choose $N=4000$ and consider $d=1000 $. We  partition the data set equally into $m=200$ servers. 
%\swa{This means each machine would have 20 examples, right? This case we have $d \gg m > n$. On the other hand, our theory is for fix $d$ and asymptotic $n$ and $m$. Just an observation.}
We randomly choose $\alpha m$ $(= 10,20) $ workers to be Byzantine and apply  norm based thresholding operation  with parameter  $\beta m$ $( = 12,22)$  respectively. 
%\swa{The term thresholding would be confusing.}
 We simulate the Byzantine workers by adding i.i.d $\mathcal{N}(0,10I_d)$ entries to the gradient. In our experiments the gradient is the most pertinent information of the the worker server. So we choose to add noise to the gradient to make it a Byzantine worker. However, later on, we consider several kinds of attack models. We choose $\|w_t-w^*\| $ as the error metric for this problem. 

\paragraph{Effectiveness of thresholding} We compare Algorithm~\ref{alg:main_algo} with compressed gradient descent (with vanilla  aggregation). Our method is equipped with Byzantine tolerance steps and the vanilla compressed gradient just computes the  average of the compressed gradient sent by the workers.  From Figure~\ref{fig:companima1000} (a,b) it is evident that the the application of norm based thresholding scheme provides better convergence result compared to the compressed gradient method without it.
 
\paragraph{Comparison with {\it signSGD} with majority vote}  Next, in Figure ~\ref{fig:companima1000}(c,d), we show the comparison of our method   with  \cite{anima} in the regression setup described above.  Our method shows a better trend in convergence.

%\textcolor{blue}{Rev2Point13Remove the SignSGD description :}In \cite{anima}, a communication efficient byzantine tolerant algorithm is proposed where communication efficiency is achieved by communicating sign of the gradient and robustness is attained by taking co-ordinate wise majority vote. The robustness in our algorithm comes from thresholding operation on the scaling factor. We show a comparison of  both method in Figure~\ref{fig:companima1000} in the regression setup depicted above.  Our method shows a better trend in convergence.
% \swa{In sign-SGD with majority vote, the center node also sends back the sign. Are we considering that for simulations? That might affect the rate of convergence as well, right? It is good to discuss this.}
\paragraph{Error-feedback with thresholding scheme}  We demonstrate the effectiveness of Byzantine resilience with error-feedback scheme as described in  Algorithm~\ref{alg:err_feed}. We compare our scheme with Algorithm~\ref{alg:main_algo} (which does not use error feedback) in Figure~\ref{fig:regwfd1000}.    It  is evident that with error-feedback, better convergence is achieved. 

%\textcolor{blue}{Rev2point17: Add some discussion.}
%As described in Section ~\ref{sec:error_feedback}, in each iteration, the worker machines add the error term occurred in the previous iteration and correct the direction of the gradient locally.
\begin{figure*}[t!]
    \centering
    \vspace{-5mm}
    \subfloat[Number of Byzantine nodes=10 ]{{\includegraphics[width=5cm]{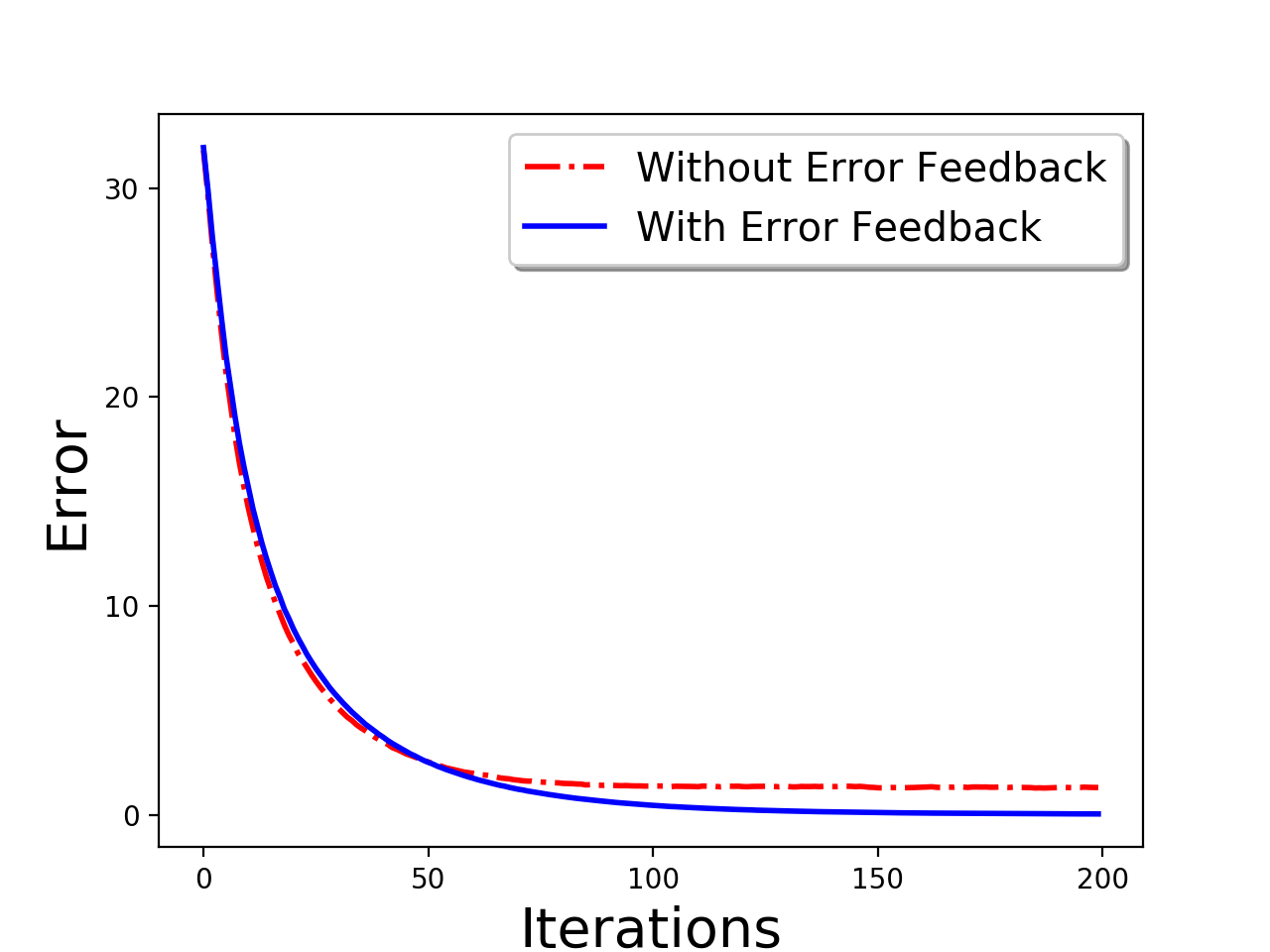} }}%
    \qquad
    \subfloat[Number of Byzantine nodes=20]{{\includegraphics[width=5cm]{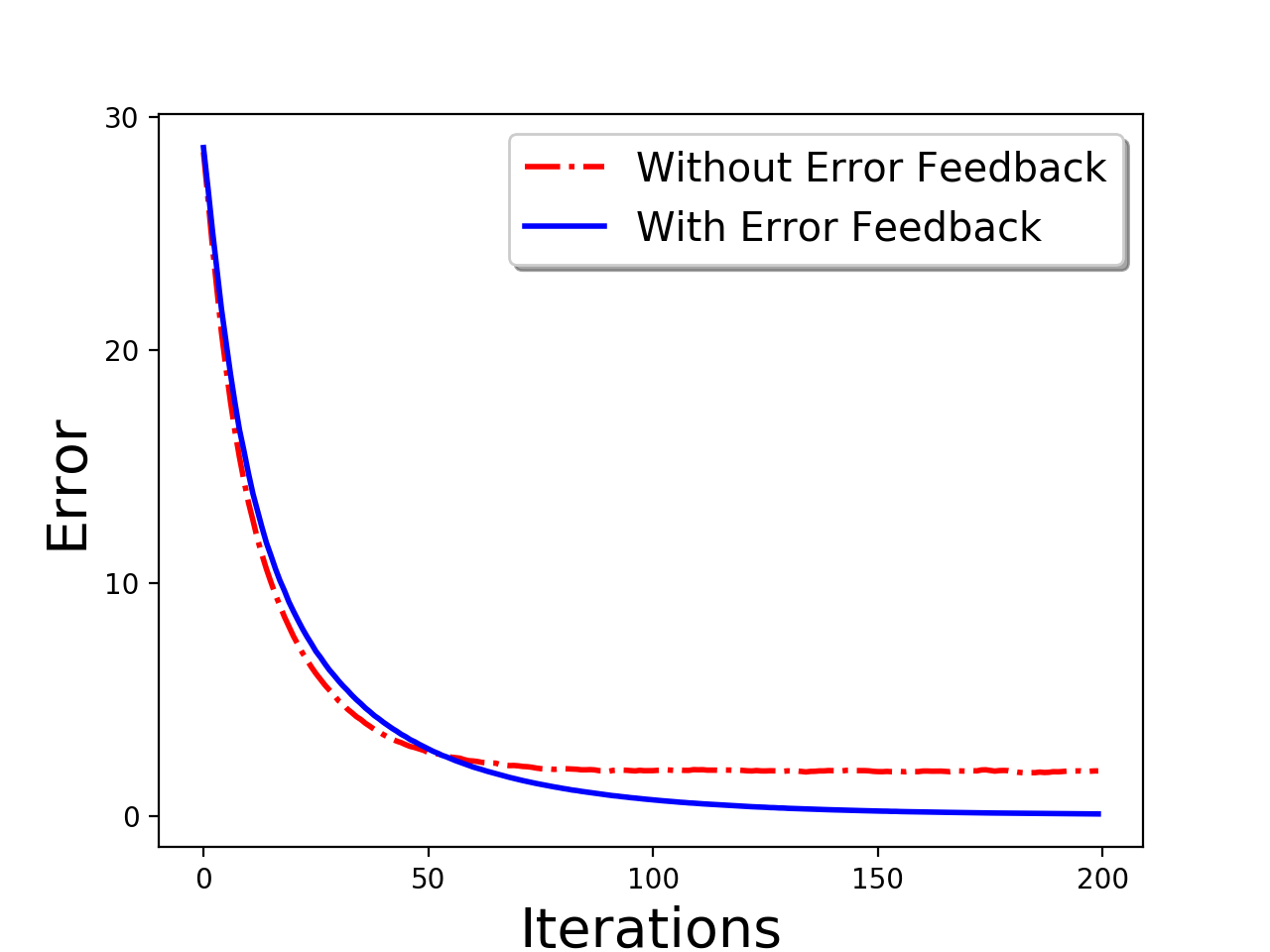} }}%
    \caption{Comparison of norm based thresholding with and without error feedback. The plots show that error feedback based scheme offers better convergence.   }%
    \label{fig:regwfd1000}
\vspace{-4mm}
    %\subfloat[ ]{{\includegraphics[width=4cm]{mnist-tr.png} }}%
    \subfloat[ ]{{\includegraphics[width=5cm]{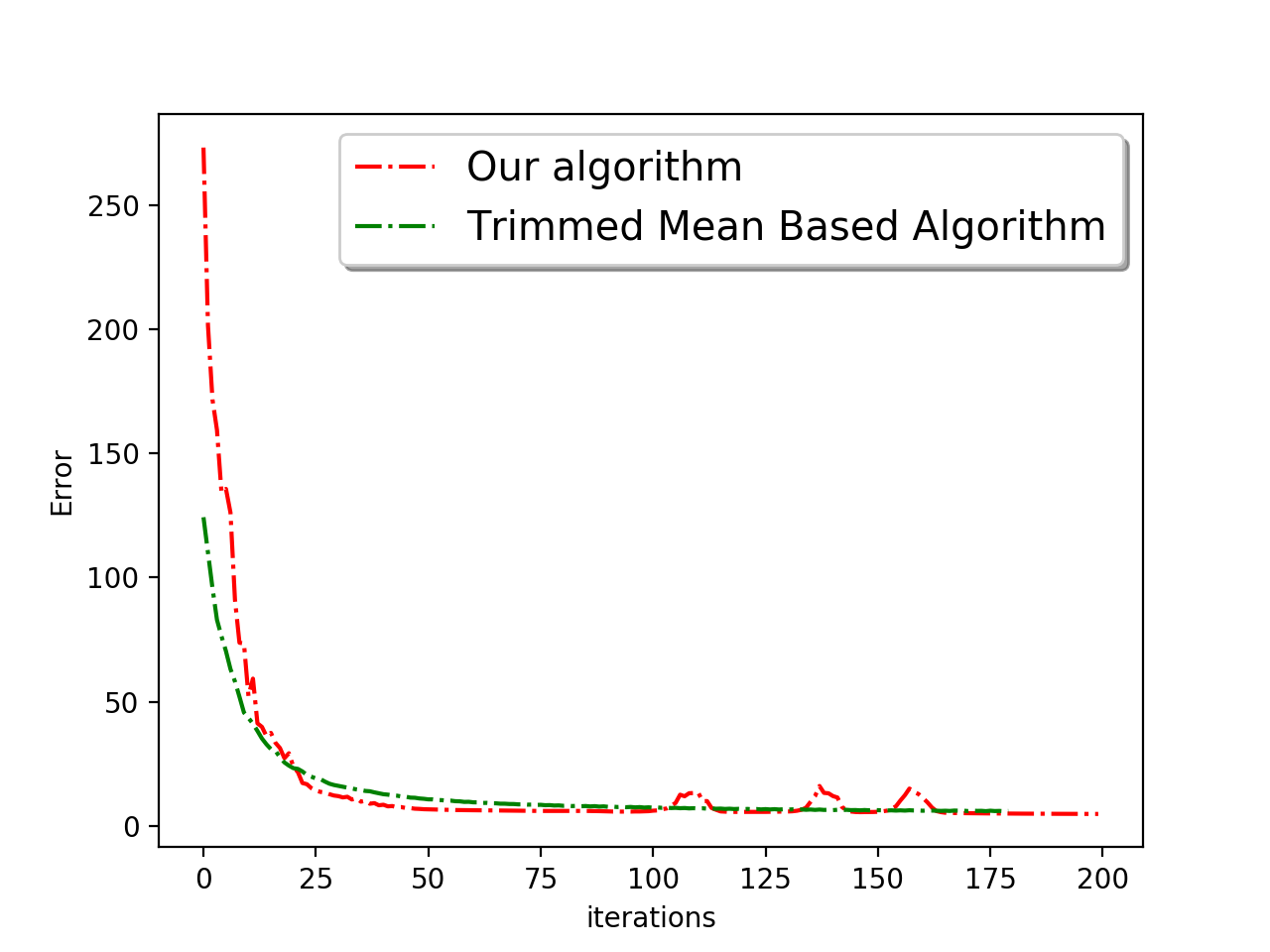} }}
    \qquad
   \subfloat[ ]{{\includegraphics[width=5cm]{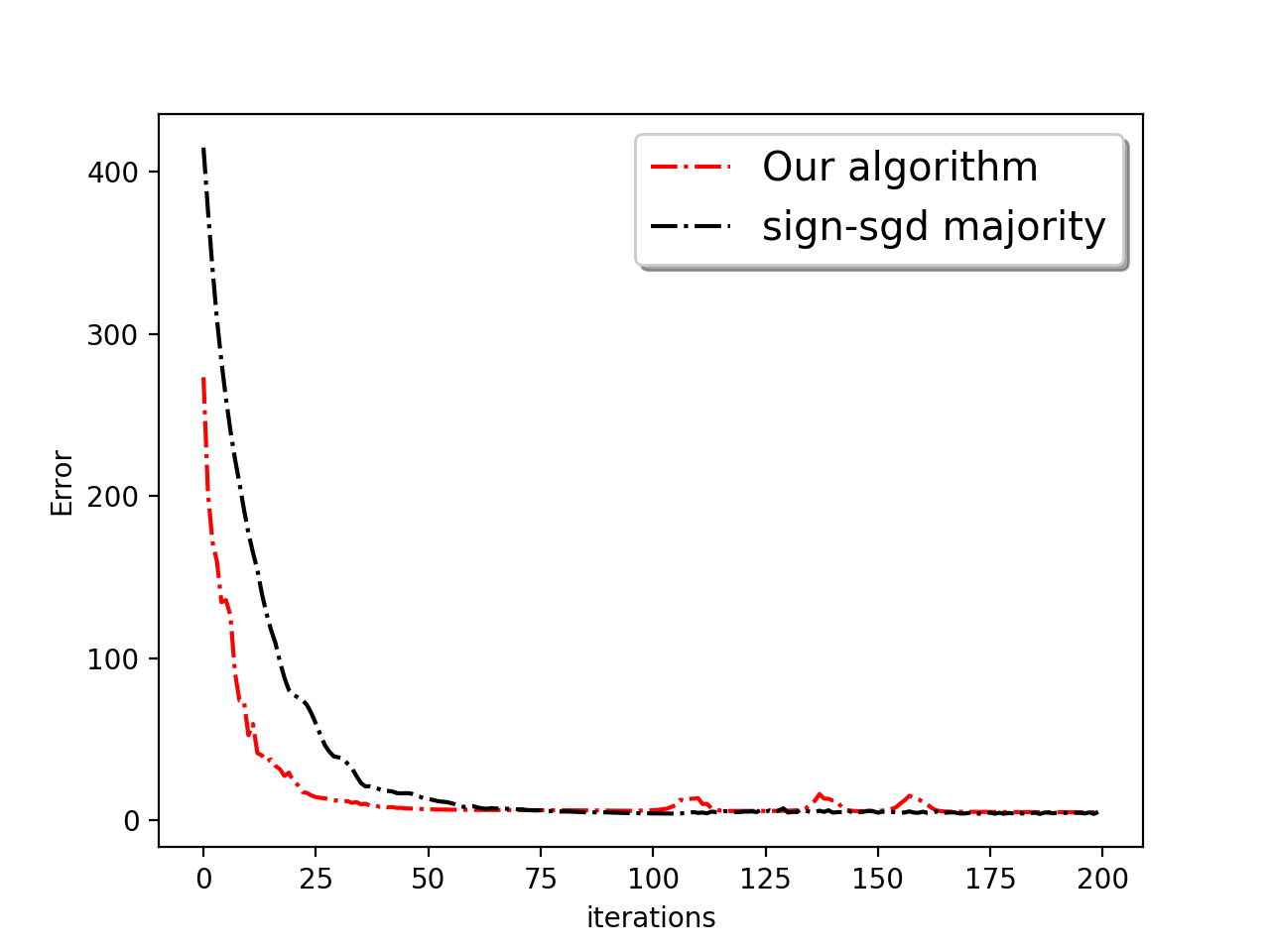} }}
  \vspace{-3mm}
    \caption{Training (cross entropy) loss for MNIST image.  Comparison with (a) Uncompressed  Trimmed mean \cite{dong} (b) majority based  {\it {signSGD}}  of \cite{anima}. In plot (a) show that Robust Gradient descent matches the convergence of the uncompressed trimmed mean \cite{dong}. Plot (b) show a faster convergence compared to the algorithm of \cite{anima}.}
    \label{fig:mnist}
\vspace{-3mm}
    \subfloat[Deterministic shift]{{\includegraphics[width=5cm]{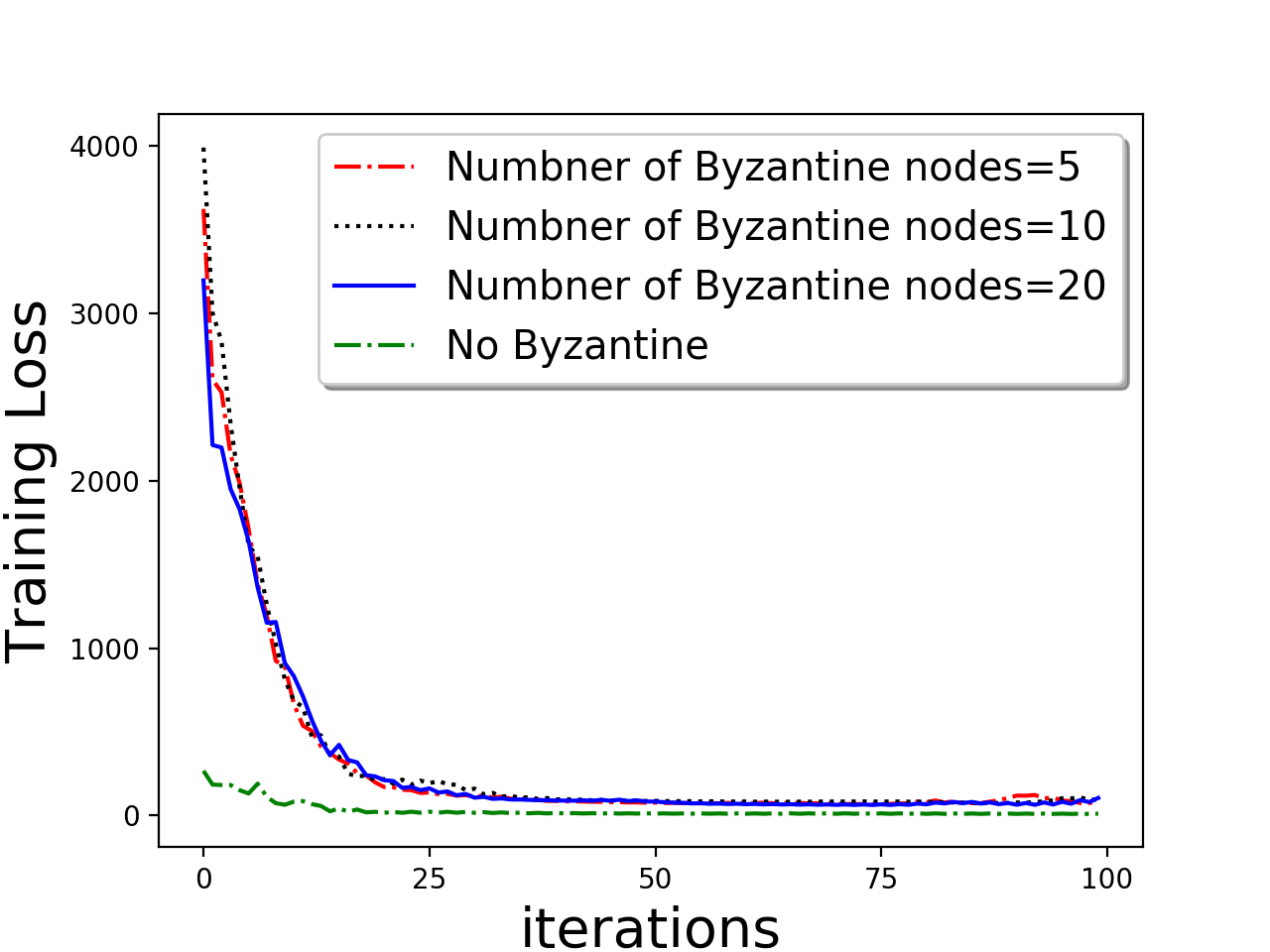} }}%
    %\qquad
    \subfloat[Random Labeling ]{{\includegraphics[width=5cm]{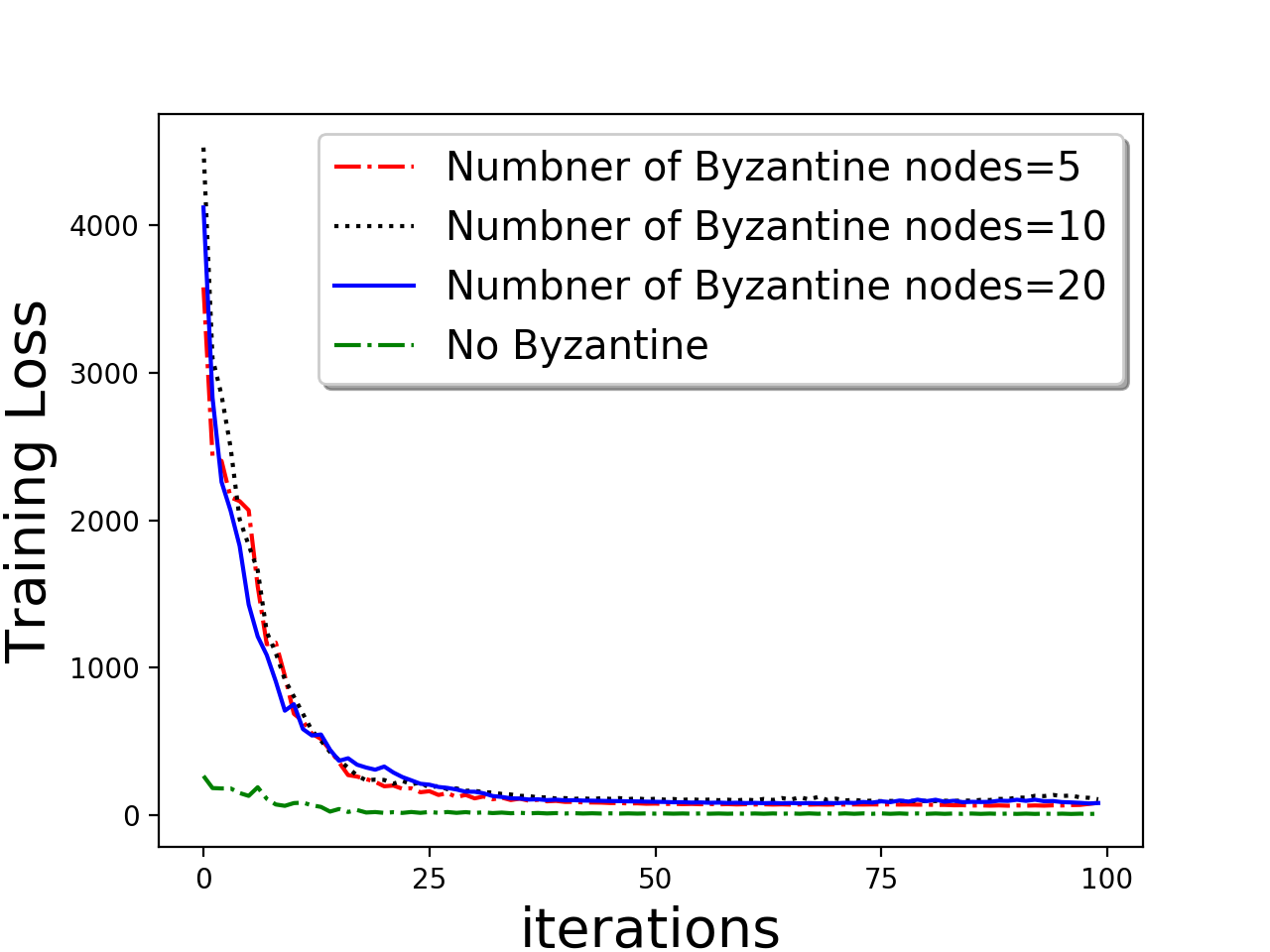} }}%
    \caption{Training (cross entropy) loss for MNIST image. Different types of attack (a)  labels with deterministic shift $(9-\text{label})$  (b) random labels. Plots show theresholding scheme with  different type of Byzantine attacks  achieve similar convergence as `no Byzantine' setup.} %
    \label{fig:diffatk}
\end{figure*}

\paragraph{Feed-forward Neural Net with ReLU activation}
 Next, we show the effectiveness of our method in training a fully connected feed forward neural net. 
 %\swa{What is the architecture? Saying `a neural net' is too little details.} 
 We implement the neural net in pytorch  and use the digit recognition dataset MNIST (\cite{mnist}). We partition $60,000$ training data into 200 different worker nodes. The neural net is equipped with $1000$ node hidden layer with ReLU activation function  and we choose \emph{cross-entropy-loss} as the loss function. We simulate the Byzantine workers by adding i.i.d $\mathcal{N}(0,10I_d)$ entries to the gradient.
 In  Figure~\ref{fig:mnist} we compare our  robust compressed gradeint descent scheme with the trimmed mean scheme of \cite{dong} and majority vote based \textit{signSGD} scheme of \cite{anima}. Compared to the majority vote based scheme, our scheme converges faster. Further, our method shows as good as performance of trimmed mean despite the fact the robust scheme of \cite{dong}  is an uncompressed scheme and uses a more complicated aggregation rules.

\paragraph{Different Types of Attacks } In the previous paragraph we compared our scheme with existing scheme with additive Gaussian noise as a form of Byzantine attack. We also show convergence results with the following type of attacks, which are quite common (\cite{dong}) in neural net training with digit recognition dataset \cite{mnist}. (a) \textit{Random label:} the  Byzantine worker machines randomly replaces  the labels of the data, and (b) \textit{Deterministic Shift: } Byzantine workers  in a deterministic manner replace the labels $y$ with $9-y$ ($0$ becomes $9$ , $9$ becomes $0$). In Figure~\ref{fig:diffatk} we show  the convergence with different numbers of Byzantine workers.

\paragraph{Large Number of Byzantine Workers}
 In Figures~\ref{fig:negatk}, we show the convergence results that holds beyond the theoretical limit (as shown in Theorem~\ref{thm:non_convex} and \ref{thm:non_convex1})  of the number of Byzantine servers in the regression problem and neural net training. In Figure~\ref{fig:negatk} (a,c), for the regression problem, the Byzantine attack is additive Gaussian noise as described before and our algorithm is robust up to $40\% (\alpha = .4)$ of the workers being Byzantine. While training of the feed-forward neural network, we apply a deterministic shift as the Byzantine attack, and the algorithm converges even for $40\% (\alpha = .4)$ Byzantine workers.

 Another `natural' Byzantine attack would be when a Byzantine worker sends $-\epsilon g$ where $0\leq \epsilon \leq 1$ and $g$ is the local gradient making the algorithm `ascent' type. We choose $\epsilon=0.9 $ and show convergence for the regression problem for up to $40\%$ Byzantine workers, and for the  neural network training for up to $33\%$ Byzantine workers in Figure ~\ref{fig:negatk} (b,d).
    
    \begin{figure*}[h!]
    \centering
    \subfloat[Regression Problem ]{{\includegraphics[width=4cm]{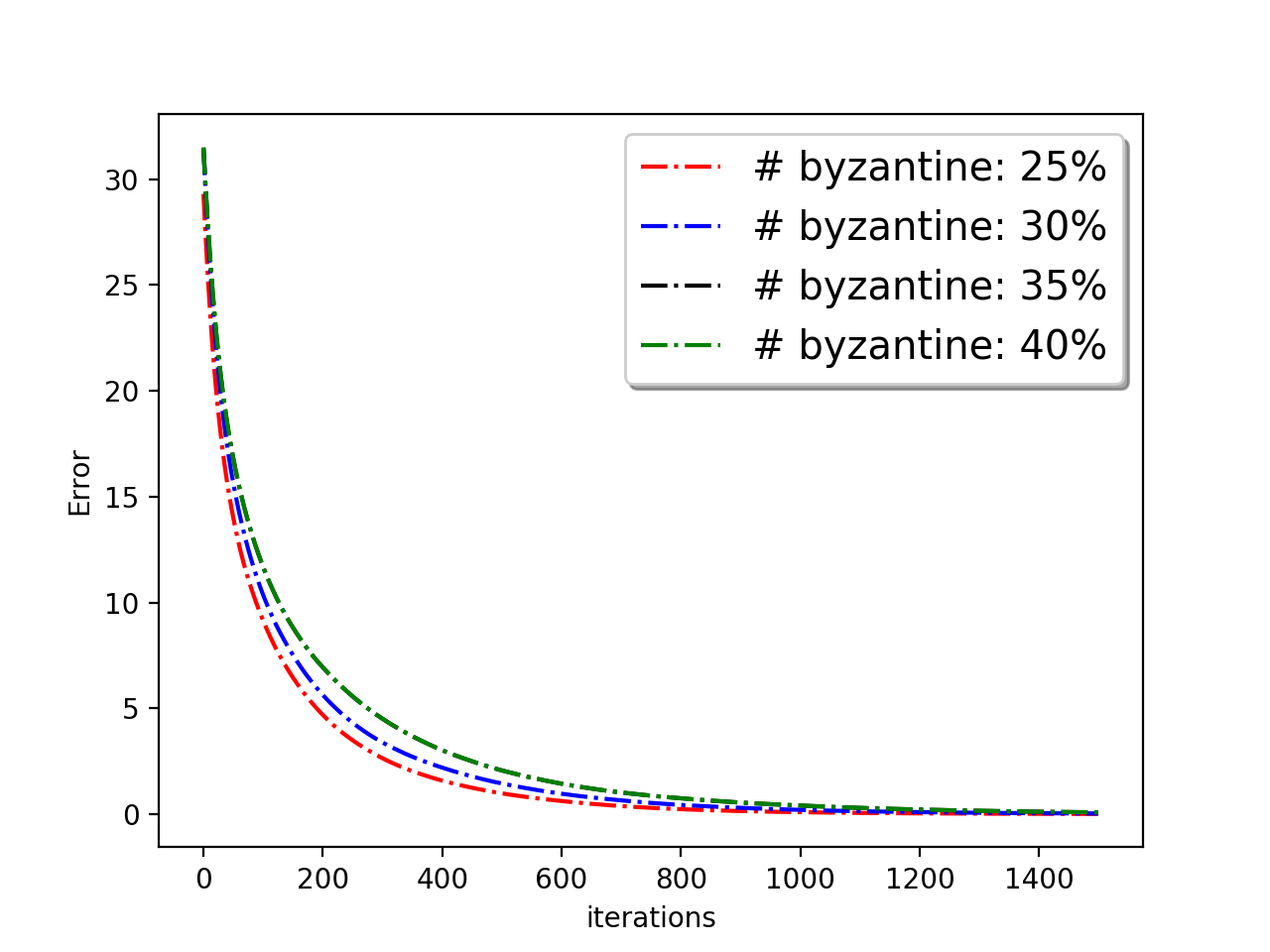} }}%
    %\qquad
    \subfloat[Training loss for ReLU net ]{{\includegraphics[width=4cm]{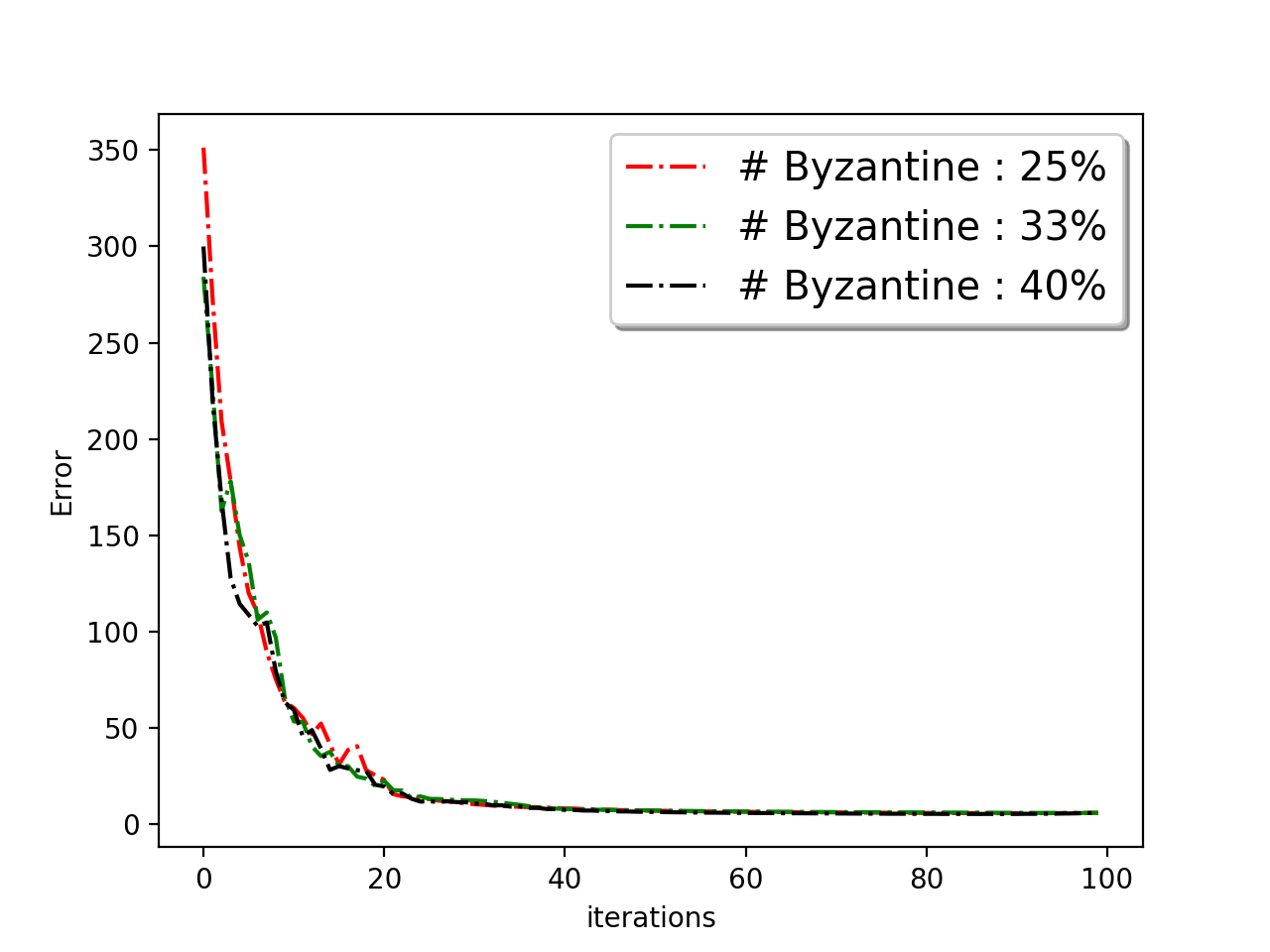} }}%
   % \caption{ } 
    %\label{fig:moreatk}
    %\vspace{-4mm}
\subfloat[Regression Problem]{{\includegraphics[width=4cm]{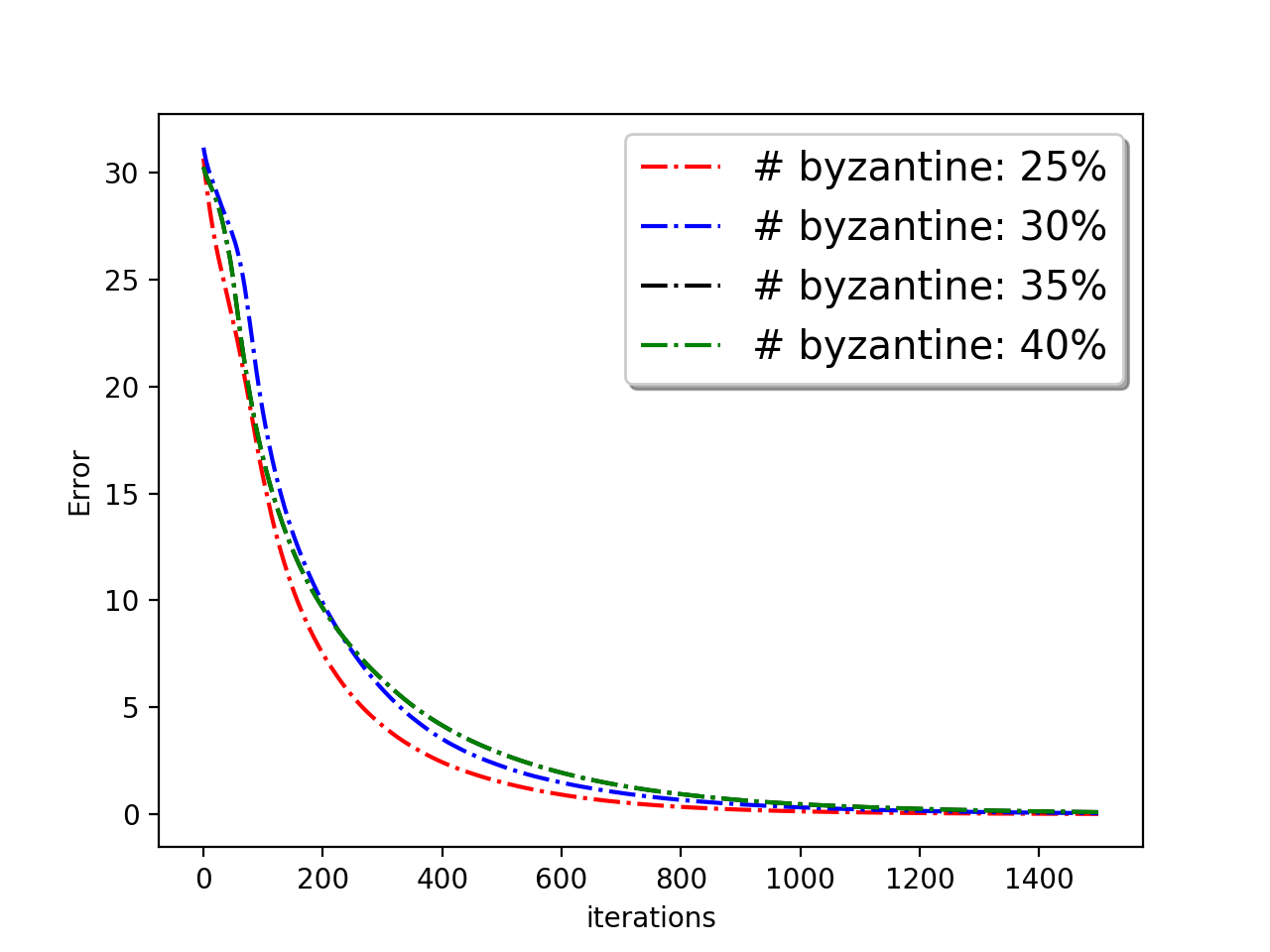} }}%
    %\qquad
    \subfloat[Training loss for ReLU net  ]{{\includegraphics[width=4cm]{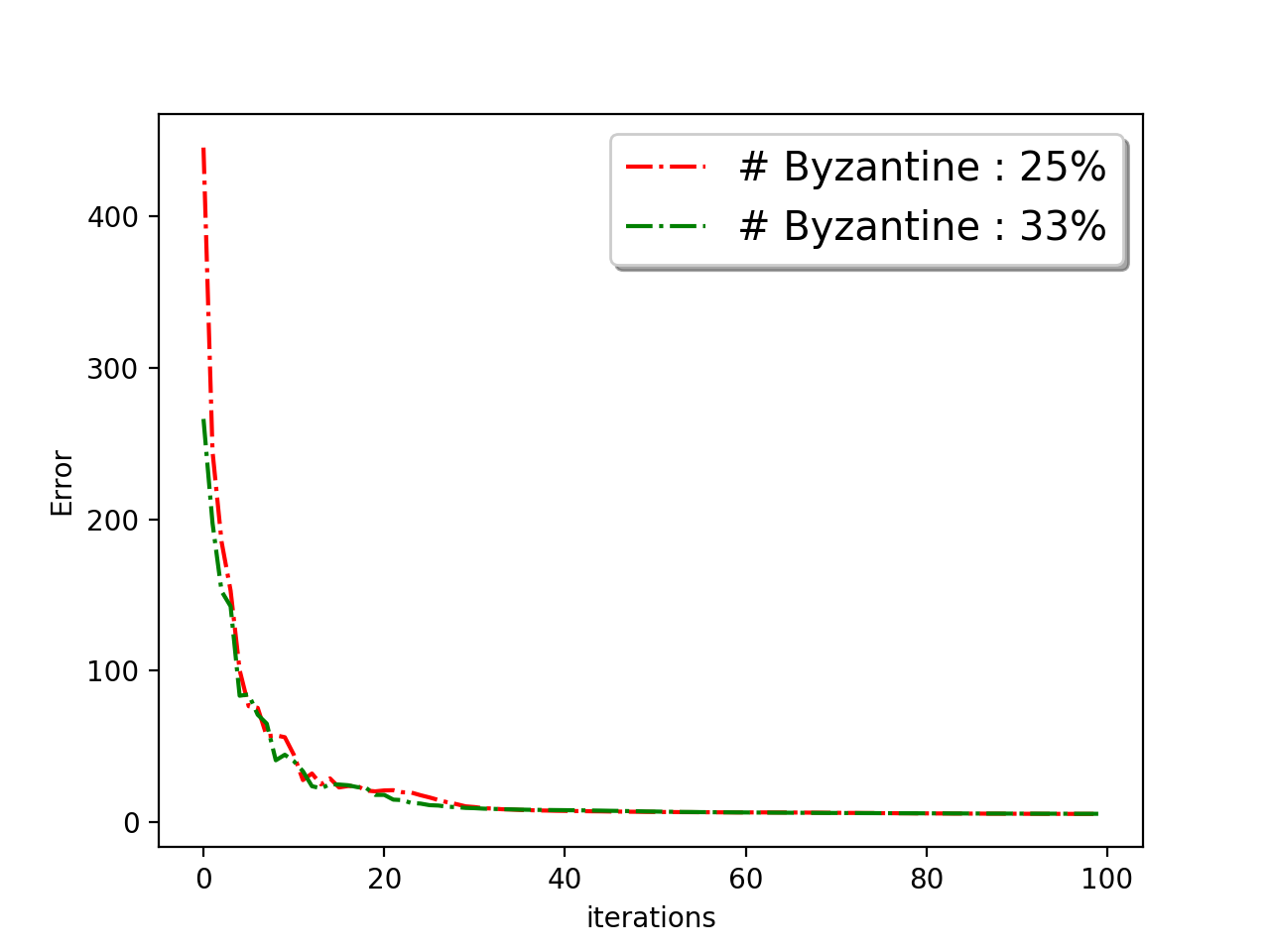} }}%
  
    \caption{Convergence for (a)  regression problem   (b) training (cross entropy) loss for MNIST image. Plots show convergence beyond the theoretical bound on the number of Byzantine machine. Convergence for (c)  regression problem   (d) training (cross entropy) loss for MNIST image. Plots show convergence with an negative Byzantine attack of $-\epsilon$ times the local gradient with high number of Byzantine machines for $\epsilon=0.9$.} %
    \label{fig:negatk}
\end{figure*}

\section{Conclusion and Future work}
 We address the problem of robust distributed optimization where the worker machines send the compressed gradient to the central machine. We propose a first order optimization algorithm, and consider the setting of restricted as well as arbitrary Byzantine machines. Furthermore, we consider the setup where error feedback is used to accelerate the learning process. We provide theoretical guarantees in all these settings and provide experimental validation under different setup. As an immediate future work, it might also be interesting to study a second order distributed optimization algorithm with compressed gradients and Hessians. In this paper we did not consider a few significant features in Federated Learning: (a) data heterogeneity across users and (b) data privacy of the worker machines. We keep these as our future endeavors.

\subsection*{Acknowledgments}
Avishek Ghosh and Kannan Ramchandran are supported in part by NSF grant NSF CCF-1527767. Raj Kumar Maity and Arya Mazumdar are supported by NSF grants NSF CCF 1642658 and 1618512. Swanand Kadhe is supported in part by National Science Foundation grants CCF-1748585 and CNS-1748692

%\clearpage
\bibliographystyle{IEEEtran}
%\bibliography{aistatref}
\bibliography{commbyz}
\section{Additional  Experiments}\label{sec:moreexp}
In Figure ~\ref{fig:mnistfail}, we show the convergence with deterministic label shift and negative update attack (previously explained in Section ~\ref{sec:experiment}) for MNIST dataset. We choose different number of Byzantine machines and the norm based threshloding scheme fails when $50\%$ of worker machines are Byzantine. This is actually a theoretical limit for which the robustness can be provided in Byzantine  resilience. 
 \begin{figure}[ht]
 \vspace{-20pt}
    \centering
    \subfloat[Deterministic Shift ]{{\includegraphics[width=4cm]{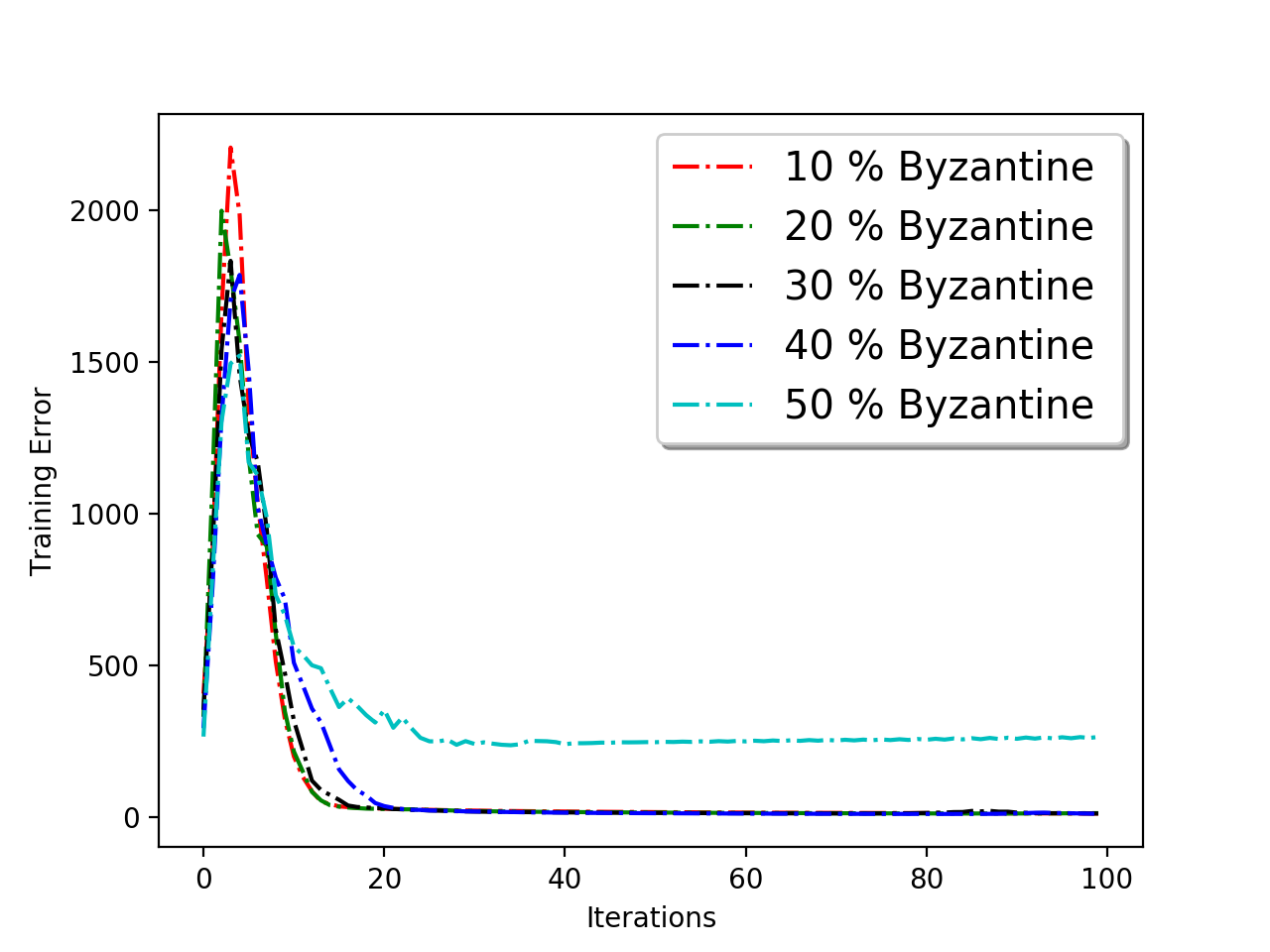} }}%
    %\qquad
    \subfloat[Negative update]{{\includegraphics[width=4cm]{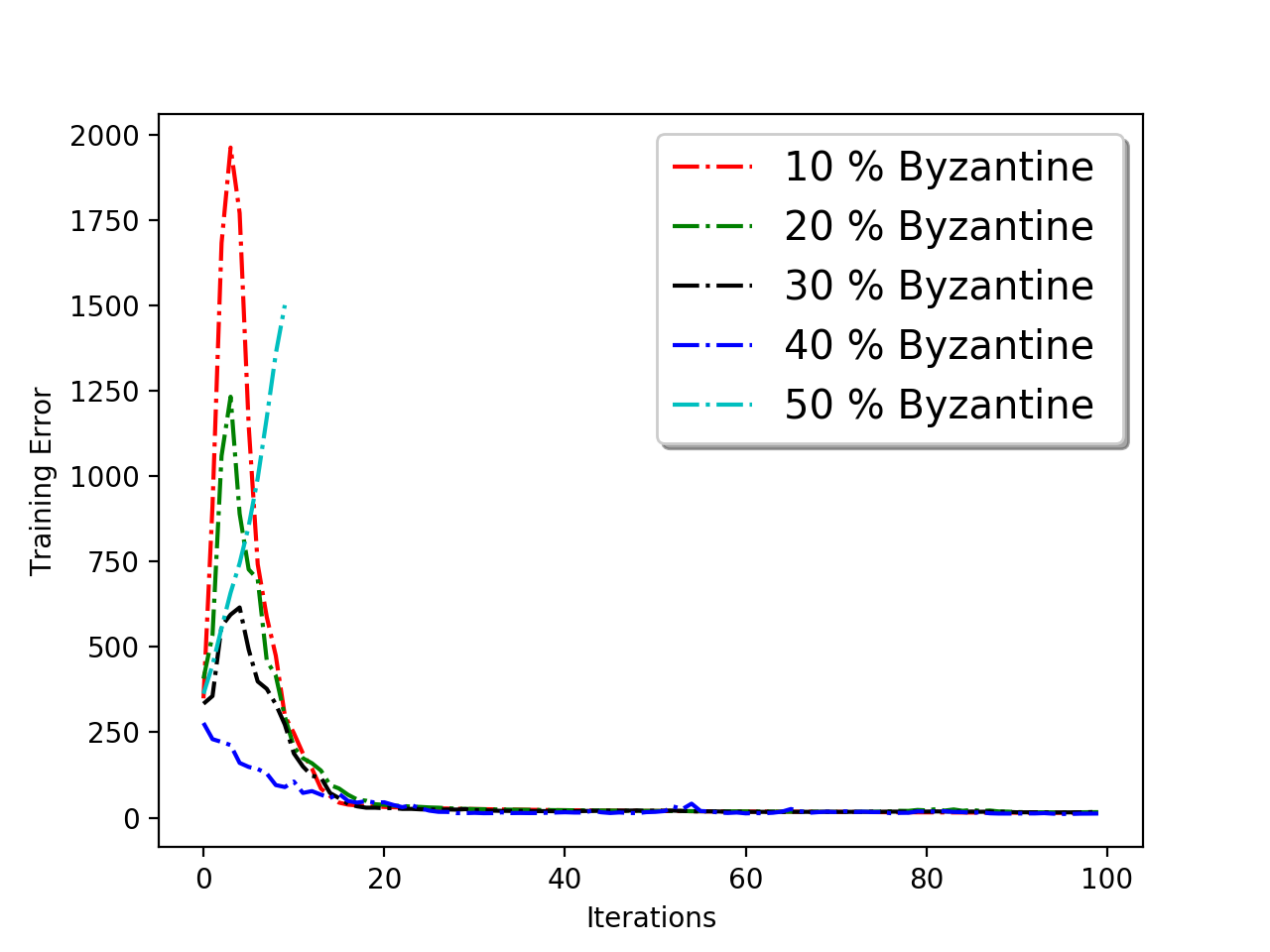} }}%
   % \qquad
    %\subfloat[Number of Byzantine nodes=20]{{\includegraphics[width=5cm]{withouterrofeed/Regression_d_1000_byz_20.png} }}%
    \caption{Convergence of (left) labels with deterministic shift (9-label) and  (Right)  negative Byzantine attack of $-\epsilon$ times the local gradient with different fraction of Byzantine machines for $\epsilon=0.9$ } %
    \label{fig:mnistfail}
\end{figure} 
 
\vspace{-10pt}
In Figure ~\ref{fig:mnistover} (left), we show the effect of negative attack in training neural network with  $20\%$ of worker machines are Byzantine. The plot shows the byzantine resilience  capability of norm based thresholding for negative update attack with different level of severity. The norm based thresholding provides  robustness for all the cases.
In Figure ~\ref{fig:mnistover} (right ), we demonstrate the scenario when the number of Byzantine machines are unknown to the algorithm and the number is either under or over-estimated. Our algorithm trimmed  $\beta$ fraction of  updates from the worker machines that is higher than the number of Byzantine machines. If the number of Byzantine machines are unknown then the safe idea is to trim more than $50\%$ of the updates. In the Figure ~\ref{fig:mnistover} (right), we entertain the idea of not knowing the number of Byzantine machines and trim more, exactly and less updates. To simulate this, we choose $40$ machines out $200$ to be Byzantine machines with Gaussian attack and trim $30,40,50$ updates. It is evident from the plot that in case of underestimating the number of Byzantine machines and trimming less number of machine leads to bad sub-optimal results.

\begin{figure}[h!]
\vspace{-20pt}
    \centering
    \subfloat[Negative attack ]{{\includegraphics[width=4cm]{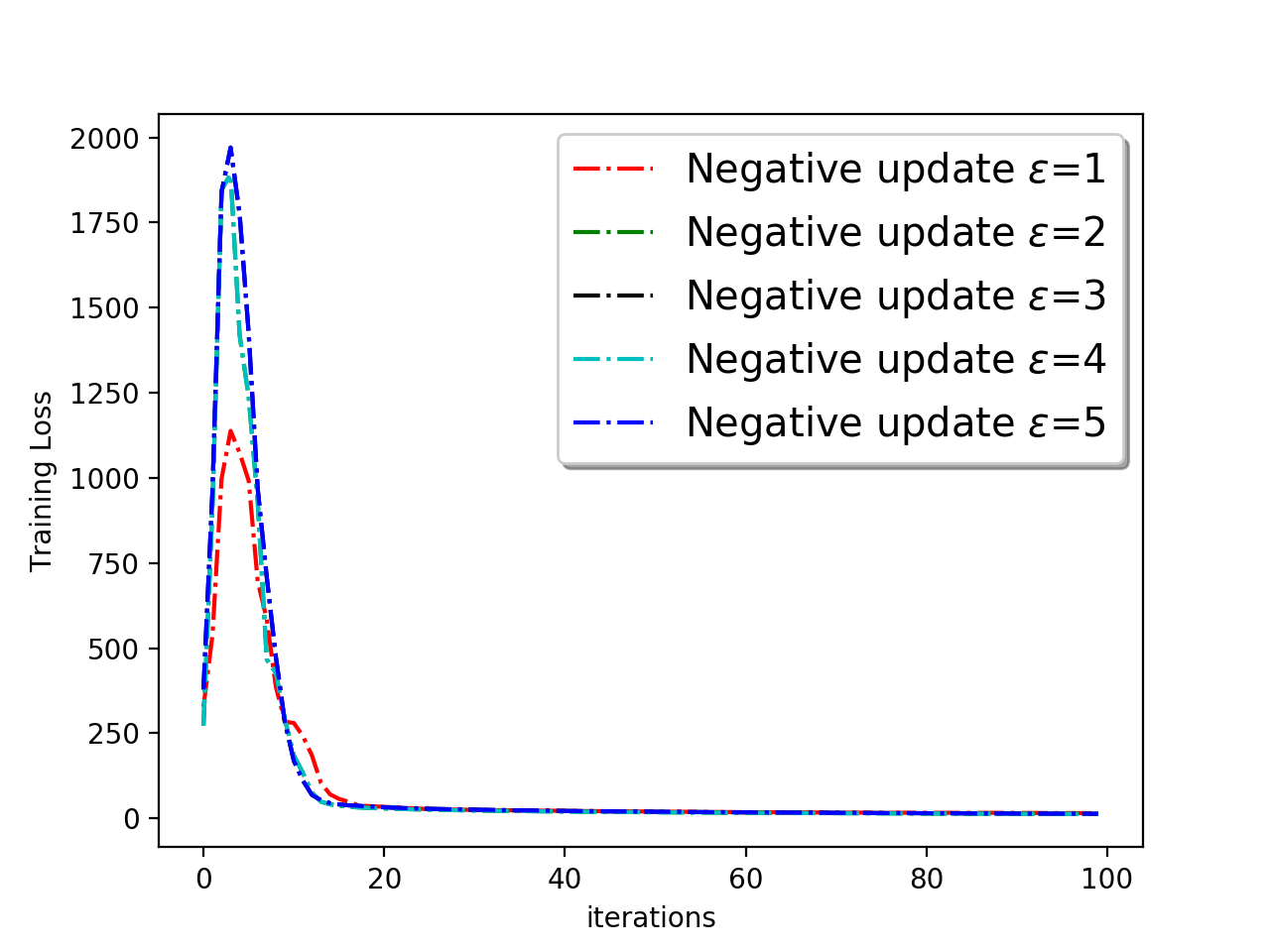} }}%
    %\qquad
    \subfloat[Unknown $\alpha$]{{\includegraphics[width=4cm]{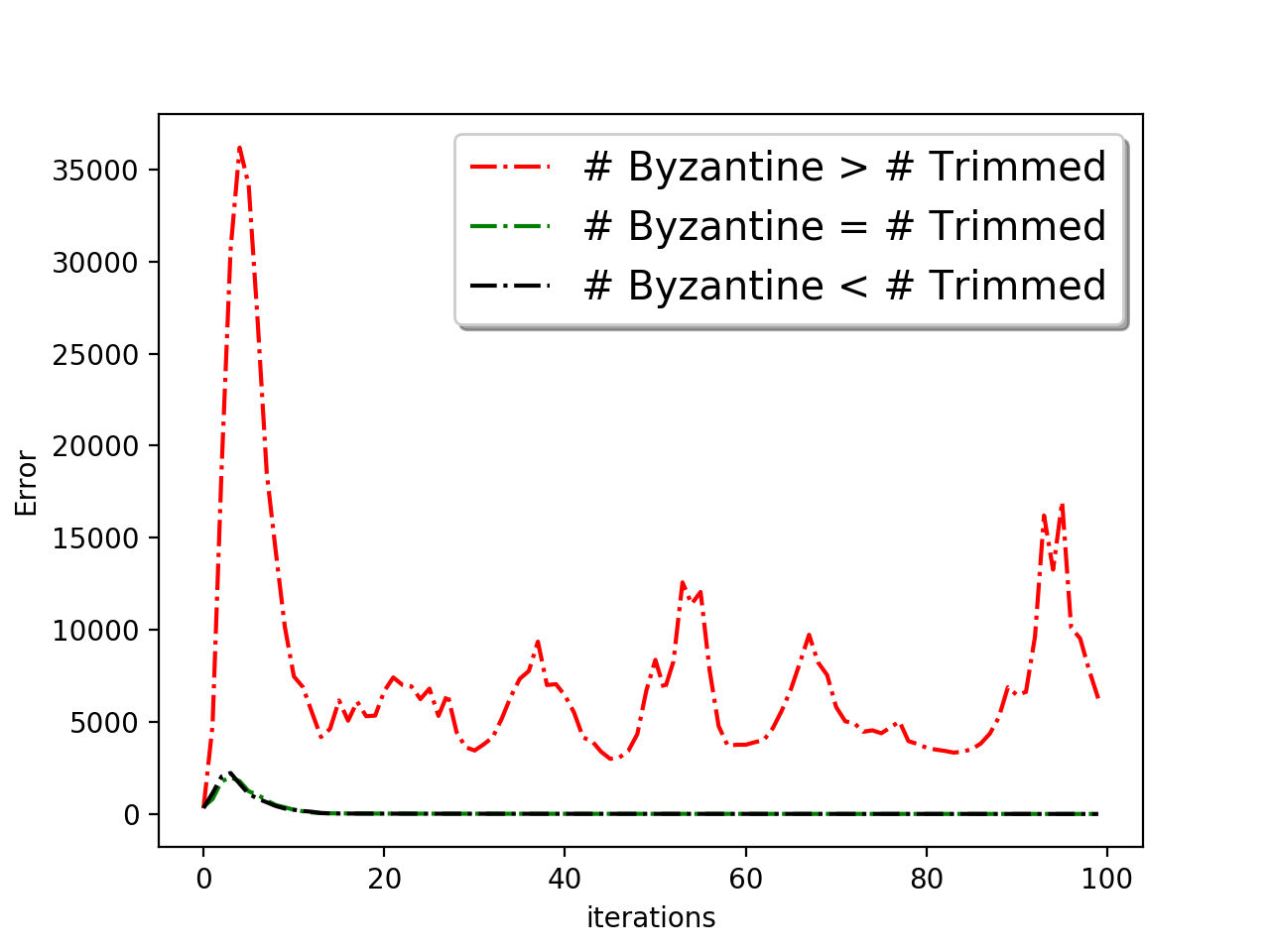} }}
    \caption{ (left) Plot shows the convergence of training Neural network with negative update attack $-\epsilon$ times the local gradient with 
    $\epsilon =1,2,3,4,5$   and $20\%$ of worker machines are Byzantine.  (Right) Plot shows convergence in training neural net when the number of Byzantine machine is over or underestimated.  }%
    \label{fig:mnistover}
\end{figure} 

\vspace{-10pt}
In Figure ~\ref{fig:bitscompare},  we compare the number of bits required  for the convergence of compressed and uncompressed $(\delta=1)$ case of our algorithm (Algorithm ~\ref{alg:main_algo}) upto a given precision for the regression problem. In particular, we choose $\|w_T - w^*\| \leq 0.1$ as a stopping criterion. In the bar plot, we report the total number of bits (in log) that the worker nodes send to the center machine. For the compression, we use the QSGD \cite{qsgd} compression scheme. We use $32$ bits to present a real number. The uncompressed scheme require at least $15 \times$ more bits to achieve the precision compare to the compressed version.

\begin{figure}[h!]
\vspace{-20pt}
    \centering
    \subfloat[Gaussian attack ]{{\includegraphics[width=4cm]{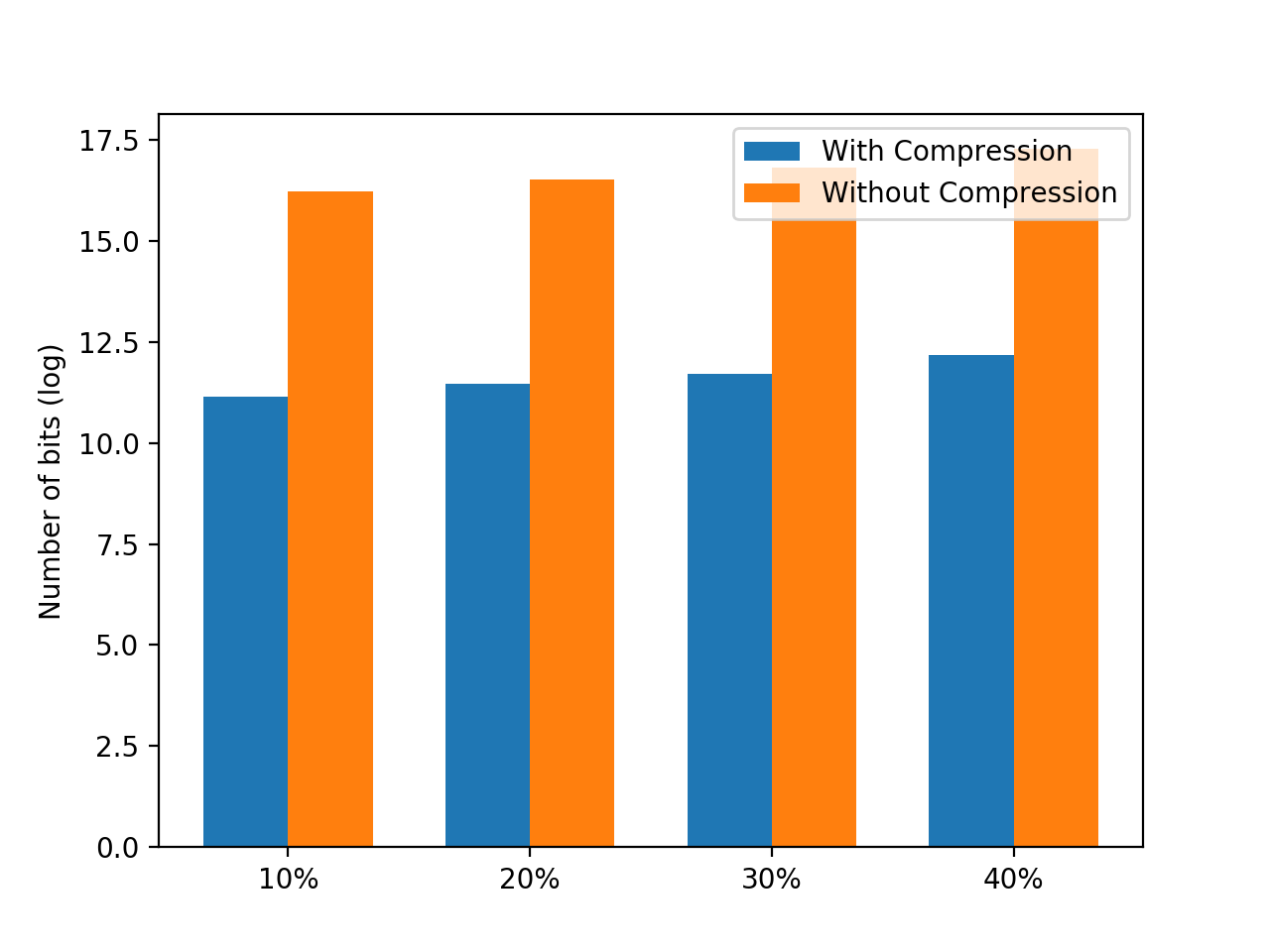} }}%
    %\qquad
    \subfloat[Negative attack]{{\includegraphics[width=4cm]{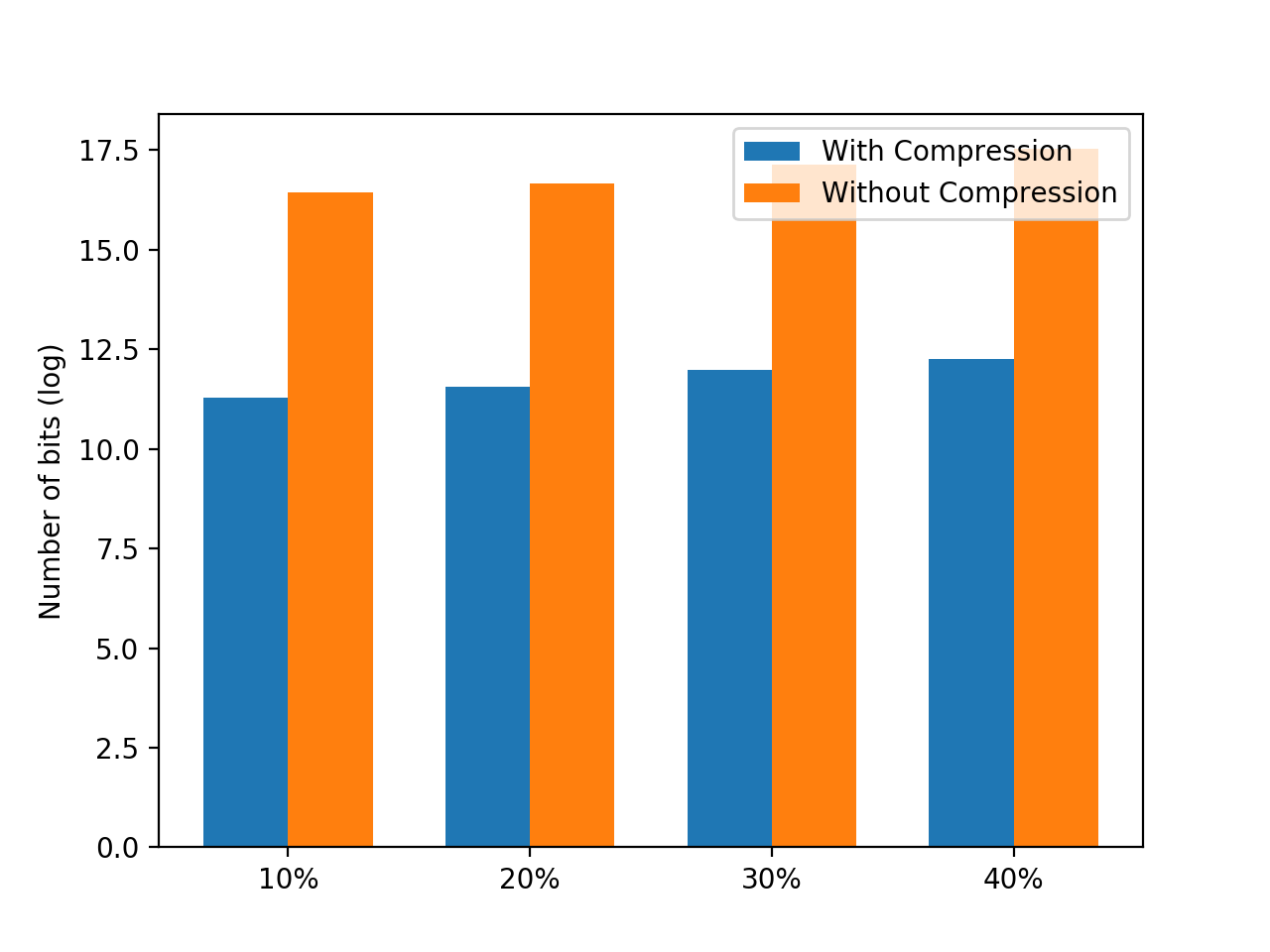} }}%
    \caption{ Comparison in terms of the average number of bits sent by each worker machines to the center machine for compressed and uncompressed case with norm based thresholding for negative update attack $(\epsilon=1)$ (left) and Gaussian attack (right). }%
    \label{fig:bitscompare}
\end{figure} 
In Figure ~\ref{fig:topk}, we show the plot of the convergence with norm based thresholding and co-ordinate wise trimmed mean \cite{dong} for $10\%$ Byzantine machines with Gaussian attack and Negative update attack. For compression, we use top-$k$ sparsification where the worker machines send top $k$ co-ordinate with maximum absolute values. We choose $k=100$. It is evident from the plot that norm based thresholding is a better robust scheme in sparsified domain.

\begin{figure}[h!]
\vspace{-20pt}
    \centering
    \subfloat[Gaussian attack ]{{\includegraphics[width=4cm]{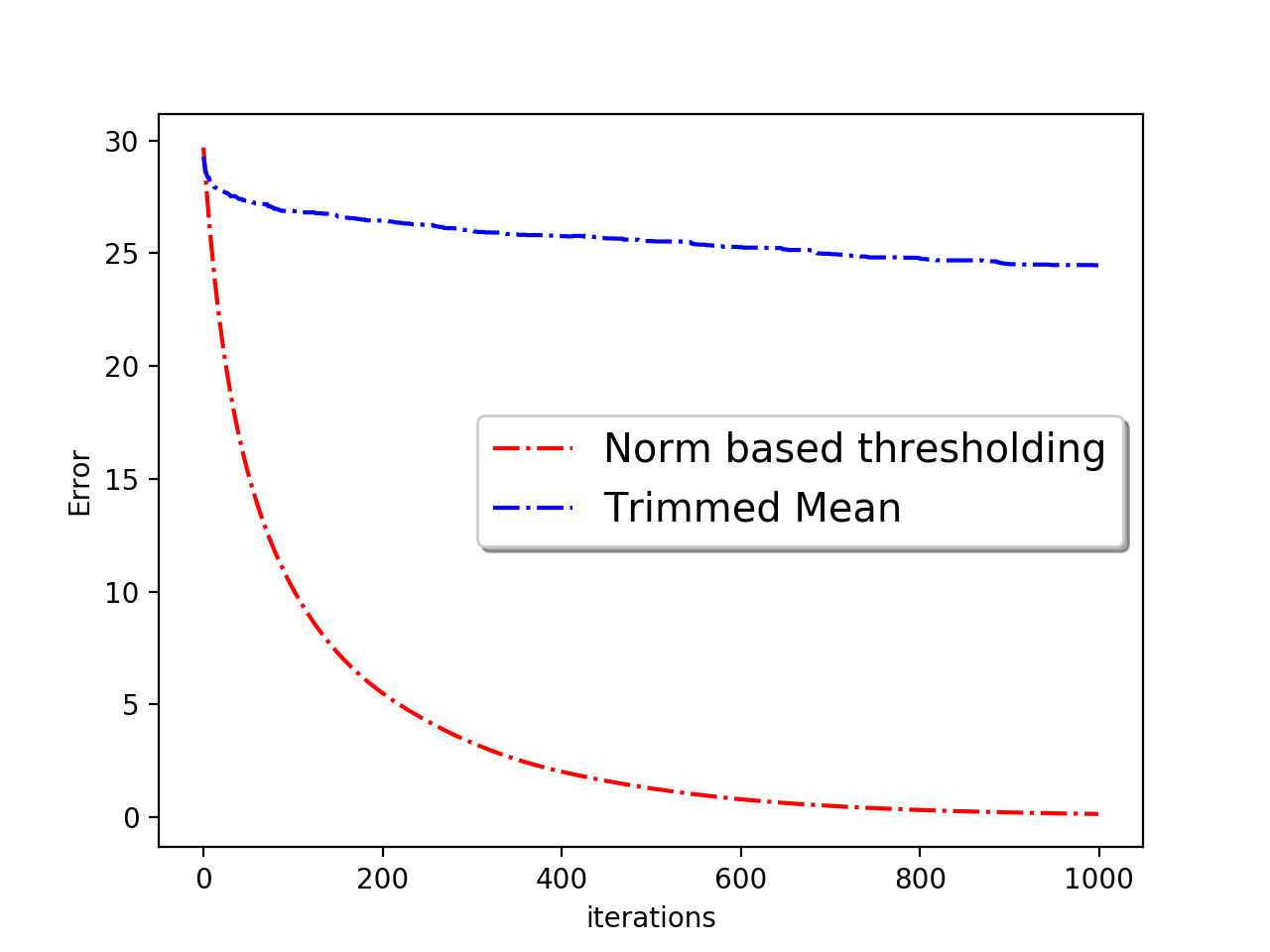} }}%
    %\qquad
    \subfloat[Negative attack]{{\includegraphics[width=4cm]{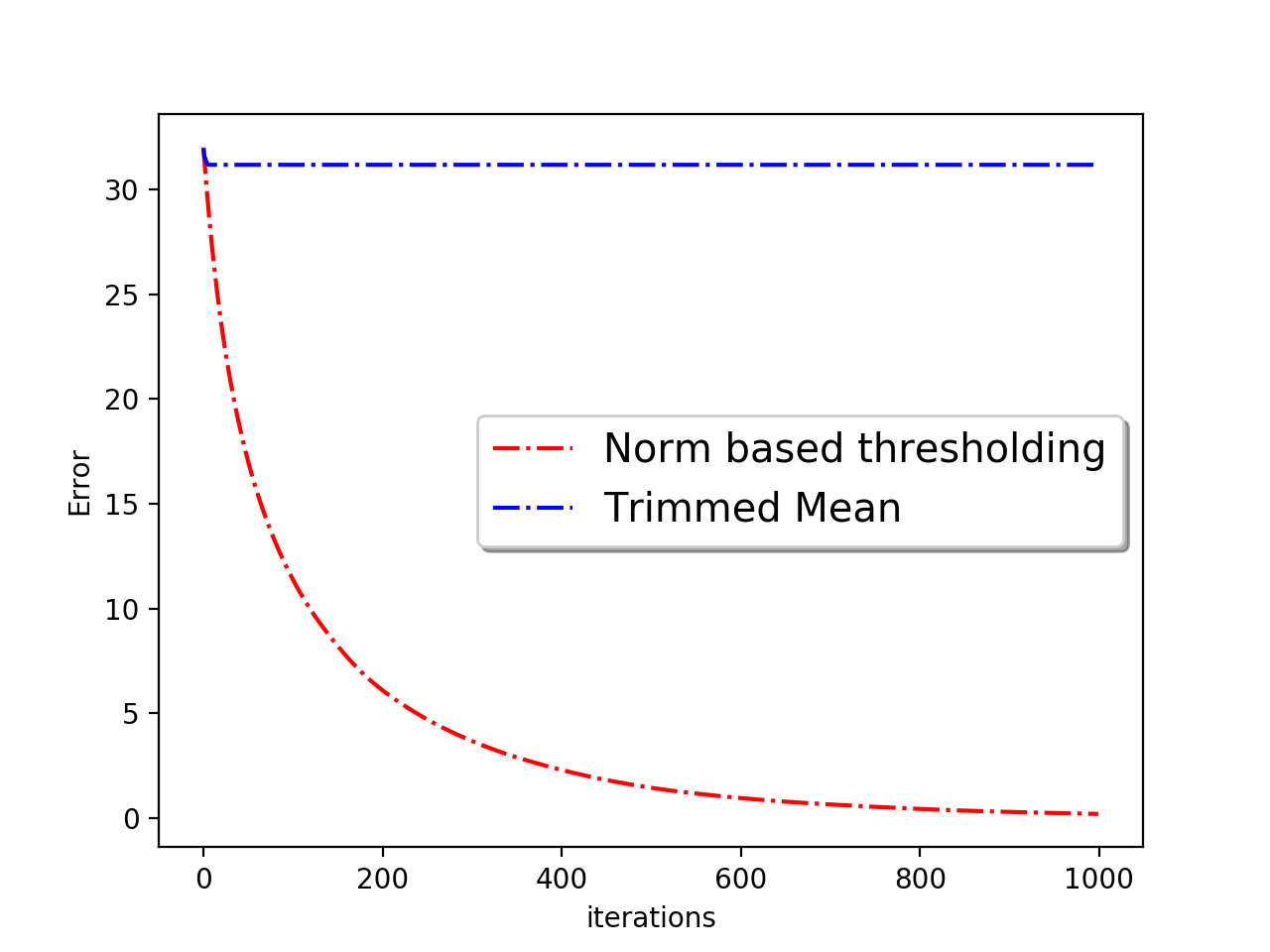} }}%
   % \qquad
    %\subfloat[Number of Byzantine nodes=20]{{\includegraphics[width=5cm]{withouterrofeed/Regression_d_1000_byz_20.png} }}%
    \caption{ Comparison of the convergence of the norm based thresholding and trimmed mean with top-k compression scheme for Gaussian attack (left) and Negative update attack with $\epsilon=1$ (right) for regression problem. We choose $k=100$ and $10\%$ Byzantine machines.  }%
    \label{fig:topk}
\end{figure}

\onecolumn

\begin{center}
%\large{\textbf{Supplementary Material for Communication Efficient Byzantine Robust Distributed Learning}}
\large{\textbf{APPENDIX}}
\end{center}

\section{Analysis of Algorithm~\ref{alg:main_algo}}
\label{sec:proofs}
In this section, we provide analysis of the Lemmas required for the proof of Theorem~\ref{thm:non_convex} and Theorem~\ref{thm:non_convex1}.

\textbf{Notation:} Let $\mathcal{M}$ and $\mathcal{B}$ denote the set of non-Byzantine and Byzantine worker machines. Furthermore, $\mathcal{U}_t$ and $\mathcal{T}_t$ denote untrimmed and trimmed worker machines. So evidently,
\begin{align*}
|\mathcal{M}|+|\mathcal{B}| = |\mathcal{U}_t|+|\mathcal{T}_t|= m.
\end{align*}

\subsection{Proof of Theorem~\ref{thm:non_convex}}

Let $g(w_t)= \frac{1}{|\mathcal{U}_t|}\sum_{i\in \mathcal{U}_t} \mathcal{Q}(\nabla F_i(w_t))$ and $\Delta = g(w_t) - \nabla F(w_t) $. We have the following Lemma to control of $\|\Delta\|^2$.
\begin{lemma}
\label{lem:delta_control}
For any $\lambda >0$, we have,
\begin{align*}
\| \Delta \|^2 \leq (1+\lambda) \left( \frac{\sqrt{1-\delta} + 2\alpha}{1-\beta} \right)^2 \| \nabla F(w_t) \|^2 + \tilde{\epsilon}(\lambda) 
\end{align*}
with probability greater than or equal to $1- \frac{c_1(1-\alpha)md}{(1+n \hat{L} D)^d} - \frac{c_2 d}{(1+(1-\alpha)m n \hat{L} D)^d}$, where
\begin{align*}
\tilde{\epsilon}(\lambda) = 2(1+\frac{1}{\lambda}) \left [ \left(\frac{\sqrt{1-\delta}+\alpha + \beta}{1-\beta} \right)^2 \epsilon_1^2 + \left( \frac{1-\alpha}{1-\beta} \right)^2 \epsilon_2^2 \right ] .
\end{align*}
with $\epsilon_1$ and $\epsilon_2$ as defined in equation~\eqref{eqn:epsilon_def} and \eqref{eqn:epsilon_tilde_def} respectively.
\end{lemma} 

The proof of the lemma is deferred to Section~\ref{sec:proof_lem_delta}. We prove the theorem using the above lemma.

We first show that with Assumption~\ref{asm:size_para} and with the choice of step size $\gamma$, we always stay in $\mathcal{W}$ without projection. Recall that $g(w_t)= \frac{1}{|\mathcal{U}_t|}\sum_{i\in \mathcal{U}_t} \mathcal{Q}(\nabla F_i(w_t))$ and $\Delta = g(w_t) - \nabla F(w_t) $. We have
\begin{align*}
\|w_{t+1} - w^* \| & \leq \|w_t - w^* \| + \gamma(\|\nabla F(w_t) \| + \|g(w_t) - \nabla F(w_t) \|) \\
& \leq  \|w_t - w^* \| + \frac{c}{L_F}(\|\nabla F(w_t)\| + \|\Delta \|)
\end{align*}

We use Lemma~\ref{lem:delta_control} with $\lambda = \lambda_0$ for a sufficiently small positive constant $\lambda_0$. Define  $\delta_0 = \left( 1- \frac{(1-\beta)^2}{1+\lambda_0} \right)$. A little algebra shows that provided  $\delta > \delta_0 + 4 \alpha - 9 \alpha^2 + 4\alpha^3$, we obtain
\begin{align*}
\|\Delta \|^2 \leq (1-c_0) \|\nabla F(w_t)\|^2 + \epsilon
\end{align*}
with probability greater than or equal to $1- \frac{c_1(1-\alpha)md}{(1+n \hat{L} D)^d} - \frac{c_2 d}{(1+(1-\alpha)m n \hat{L} D)^d}$, where $c_0$ is a positive constant and $\epsilon$ is defined in equation~\eqref{eqn:phi_def}.
Substituting, we obtain
\begin{align*}
\|w_{t+1} - w^* \|  & \leq  \|w_t - w^* \| + \frac{c_1}{L_F} \left( (1 + \sqrt{1-c_0})\|\nabla F(w_t)\| + \sqrt{\epsilon} \right) \\
 & \leq  \|w_t - w^* \| + \frac{c_1}{L_F} \left( (2 - \frac{c_0}{2})\|\nabla F(w_t)\| + \sqrt{\epsilon} \right).
\end{align*}
where we use the fact that $\sqrt{1-c_0}\leq 1 - c_0/2$. Now, running $T = C \frac{L_F(F(w_0) - F(w^*))}{\epsilon}$ iterations, we see that Assumption~\ref{asm:size_para} ensures that the iterations of Algorithm~\ref{alg:main_algo} is always in $\mathcal{W}$. Hence, let us now analyze the algorithm without the projection step.

\vspace{3mm}

Using the smoothness of $F(.)$, we have
\begin{align*}
    F(w_{t+1}) \leq F(w_t) & + \langle \nabla F(w_t), w_{t+1}-w_t \rangle  + \frac{L_F}{2}\|w_{t+1} - w_t \|^2.
\end{align*}
 Using the iteration of Algorithm~\ref{alg:main_algo}, we obtain
\begin{align*}
    F(w_{t+1})  & \leq F(w_t) - \gamma \langle \nabla F(w_t), \nabla F(w_t)+ \Delta \rangle  +\frac{ \gamma^2 L_F}{2}\|\nabla F(w_t) + \Delta \|^2 \\
    & \leq F(w_t) - \gamma \|\nabla F(w_t)\|^2 - \gamma \langle \nabla F(w_t),\Delta \rangle  + \frac{\gamma^2 L_F}{2} \|\nabla F(w_t)\|^2 + \frac{\gamma^2 L_F}{2}\|\Delta\|^2 + \gamma^2 L_F \langle \nabla F(w_t),\Delta \rangle \\
    & \leq F(w_t) - (\gamma - \frac{\gamma^2L_F}{2}) \|\nabla F(w_t) \|^2 + (\gamma + \gamma^2 L_F) \left( \frac{\rho}{2}\|\nabla F(w_t)\|^2 + \frac{1}{2\rho}\|\Delta\|^2 \right)+ \frac{\gamma^2 L_F}{2}\|\Delta\|^2,
\end{align*}
where $\rho > 0$ and the last inequality follows from Young's inequality. Substituting $\rho = 1$, we obtain

\begin{align*}
    (\gamma/2 - \gamma^2 L_F) \|\nabla F(w_t)\|^2  \leq F(w_t) - F(w_{t+1})\nonumber   + (\gamma/2 + \gamma^2 L_F) \| \Delta\|^2.
\end{align*}
We now use Lemma~\ref{lem:delta_control} to obtain
 \begin{align*}
    & (\frac{\gamma}{2} -\gamma^2 L_F) \| \nabla F(w_t) \|^2 \leq F(w_t) - F(w_{t+1}) \\
    & + (\gamma/2 + \gamma^2 L_F) \bigg ((1+\lambda) \left( \frac{\sqrt{1-\delta} + 2\alpha}{1-\beta} \right)^2 \| \nabla F(w_t) \|^2 + \tilde{\epsilon}(\lambda) \bigg ). 
 \end{align*}
with high probability. Upon further simplification, we have
\begin{align*}
    \bigg ( \frac{\gamma}{2} - \frac{\gamma}{2} (1+\lambda) \left( \frac{\sqrt{1-\delta} + 2\alpha}{1-\beta} \right)^2  - (1+\lambda) \left( \frac{\sqrt{1-\delta} + 2\alpha}{1-\beta} \right)^2 \gamma^2 L_F  - \gamma^2 L_F \bigg )   \| \nabla F(w_t)\|^2 \\
     \leq F(w_t) -  F(w_{t+1})  + (\gamma/2 + \gamma^2 L_F) \tilde{\epsilon}(\lambda).
\end{align*}
We now substitute $\gamma = \frac{c}{L_F}$, for a small enough constant $c$, so that we can ignore the contributions of the terms with quadratic dependence on $\gamma$. We substitute $\lambda = \lambda_0$ for a sufficiently small positive constant $\lambda_0$. Provided  $\delta > \delta_0 + 4 \alpha - 9 \alpha^2 + 4\alpha^3$, where $\delta_0 = \left( 1 - \frac{(1-\beta)^2}{1+\lambda_0}\right)^2$, we have
\begin{align*}
 \bigg ( \frac{\gamma}{2} - \frac{\gamma}{2} (1+\lambda) \left( \frac{\sqrt{1-\delta} + 2\alpha}{1-\beta} \right)^2  - (1+\lambda) \left( \frac{\sqrt{1-\delta} + 2\alpha}{1-\beta} \right)^2 \gamma^2 L_F  - \gamma^2 L_F \bigg ) = \frac{c_1}{L_F},
\end{align*}
where $c_1$ is a constant. With this choice, we obtain
\begin{align*}
\frac{1}{T+1}\sum_{t=0}^T \| \nabla F (w_t) \|^2 & \leq C_1  \frac{L_F(F(w_0) - F(w^*))}{T+1} + C_2  \epsilon
\end{align*}
where the first term is obtained from a telescopic sum and $\epsilon$ is defined in equation~\eqref{eqn:phi_def}. Finally, we obtain
\begin{align*}
    \min_{t=0,\ldots, T}\| \nabla F(w_t) \|^2 & \leq  C_1  \frac{L_F(F(w_0) - F(w^*))}{T+1} + C_2  \epsilon
\end{align*}
with probability greater than or equal to $1- \frac{c_1(1-\alpha)md}{(1+n \hat{L} D)^d} - \frac{c_2 d}{(1+(1-\alpha)m n \hat{L} D)^d}$, proving Theorem~\ref{thm:non_convex}.

\subsection{Proof of Theorem~\ref{thm:non_convex1}}
The proof of convergence for Theorem~\ref{thm:non_convex1} follows the same steps as Theorem~\ref{thm:non_convex}. Recall that the quantity of interest is 
 \begin{align*}
    \widetilde{\Delta} = g(w_t) -\nabla F(w_t)
\end{align*}
for which  we prove bound in the following lemma.
\begin{lemma}
\label{lem:delta_control1}
For any $\lambda >0$, we have,
\begin{align*}
\| \widetilde{\Delta} \|^2 \leq ((1+\lambda)\bigg(\frac{(1+\beta)\sqrt{1-\delta} +2\alpha}{1-\beta} \bigg)^2||\nabla F(w_t)||^2 + \widetilde{\epsilon}(\lambda)
\end{align*}
with probability greater than or equal to $1- \frac{c_1(1-\alpha)md}{(1+n \hat{L} D)^d} - \frac{c_2 d}{(1+(1-\alpha)m n \hat{L} D)^d}$, where
\begin{align*}
\tilde{\epsilon}(\lambda) = 2(1+\frac{1}{\lambda})\bigg(\bigg(\frac{(1+\beta)\sqrt{1-\delta} +\alpha+\beta}{1-\beta} \bigg)^2\epsilon^2_1+(\frac{1-\alpha}{1-\beta})^2\epsilon_2^2 \bigg) .
\end{align*}
with $\epsilon_1$ and $\epsilon_2$ as defined in equation~\eqref{eqn:epsilon_def} and \eqref{eqn:epsilon_tilde_def} respectively.
\end{lemma} 
Taking the above lemma for granted, we proceed to prove Theorem~\ref{thm:non_convex1}. The proof of  Lemma~\ref{lem:delta_control1} is deferred to Section~\ref{subsec:proof_delta_control1}.

The proof parallels the proof of \ref{thm:non_convex}, except the fact that we use Lemma~\ref{lem:delta_control1} to upper bound $\| \widetilde{\Delta} \|^2$. Correspondingly,  a little algebra shows that we require $\delta > \widetilde{\delta_0} + 4\alpha - 8\alpha^2 + 4\alpha^3$, where $\widetilde{\delta_0} = \left( 1-\frac{(1-\beta)^2}{(1+\beta)^2(1+\lambda_0)} \right)$, where $\lambda_0$ is a sufficiently small positive constant. With the above requirement, the proof follows the same steps as Theorem~\ref{thm:non_convex} and hence we omit the details here.

\subsection{Proof of Lemma~\ref{lem:delta_control}:}
\label{sec:proof_lem_delta}

We require the following result to prove Lemma~\ref{lem:delta_control}. In the following result, we show that for non-Byzantine worker machine $i$, the local gradient $\nabla F_i(w_t)$ is concentrated around the global gradient $\nabla F(w_t)$.
\begin{lemma}
\label{lem:ind_control}
For any $w \in \mathcal{W}$,  we have
\begin{align*}
   \max_{i \in \mathcal{M}} \| \nabla F_i(w) - \nabla F(w) \| \leq \epsilon_1
\end{align*}
with probability exceeding 
 $1-\frac{2(1-\alpha)md}{(1+n \hat{L} D)^d}$,  where $\epsilon_1$ is defined in equation~\eqref{eqn:epsilon_def}.
\end{lemma}
Since the iterations $\{w_t\}_{t=1}^T \in \mathcal{W}$, we have the above lemma for all the iterates of our algorithm. Furthermore, we have the following Lemma which implies that the average of local gradients $\nabla F_i(w_t)$ over non-Byzantine worker machines is close to its expectation $\nabla F(w_t)$.
\begin{lemma}
\label{lem:machine_control}
 For any $w \in \mathcal{W}$, we have
\begin{align*}
    \|\frac{1}{|\mathcal{M}|}\sum_{i\in \mathcal{M}} \nabla F_i(w) - \nabla F(w) \| \leq  \epsilon_2.
\end{align*}
with probability exceeding 
 $1- \frac{2(1-\alpha)md}{(1+n \hat{L} D)^d} - \frac{2d}{(1+(1-\alpha)m n \hat{L} D)^d}$, where $\epsilon_2$ is defined in equation~\eqref{eqn:epsilon_tilde_def}.
\end{lemma}
similarly, since the iterations $\{w_t\}_{t=1}^T \in \mathcal{W}$, we have the above lemma for all the iterates of our algorithm.

\vspace{3mm}
Recall the definition of $\Delta$. Using triangle inequality, we obtain
\begin{align*}
    \|\Delta\|  \leq \underbrace{\|\frac{1}{|\mathcal{U}_t|}\sum_{i\in \mathcal{U}_t} \mathcal{Q}(\nabla F_i(w_t)) - \frac{1}{|\mathcal{U}_t|}\sum_{i\in \mathcal{U}_t} \nabla F_i(w_t) \|}_{T_1} + \underbrace{ \|\frac{1}{|\mathcal{U}_t|}\sum_{i\in \mathcal{U}_t}  \nabla F_i(w_t) - \nabla F(w_t) \|}_{T_2}
\end{align*}

We first control $T_1$. Using the compression scheme (Definition~\ref{asm:compress}), we obtain
\begin{align*}
  T_1 =  & \|\frac{1}{|\mathcal{U}_t|}\sum_{i\in \mathcal{U}_t} \mathcal{Q}(\nabla F_i(w_t)) - \frac{1}{|\mathcal{U}_t|}\sum_{i\in \mathcal{U}_t} \nabla F_i(w_t) \|  \leq \frac{\sqrt{1-\delta}}{|\mathcal{U}_t|} \sum_{i\in \mathcal{U}_t} \|\nabla F_i(w_t) \| \\
    & \leq \frac{\sqrt{1-\delta}}{|\mathcal{U}_t|} \left[ \sum_{i \in \mathcal{M}} \|\nabla F_i(w_t) \| - \sum_{i \in \mathcal{M}\cap \mathcal{T}_t}  \|\nabla F_i(w_t) \| + \sum_{i \in \mathcal{B}\cap \mathcal{U}_t}  \|\nabla F_i(w_t) \| \right ] \\
   & \leq \frac{\sqrt{1-\delta}}{|\mathcal{U}_t|} \left[ \sum_{i \in \mathcal{M}} \|\nabla F_i(w_t) \| + \sum_{i \in \mathcal{B}\cap \mathcal{U}_t}  \|\nabla F_i(w_t) \| \right ]
\end{align*}
Since $\beta \geq \alpha$, we ensure that $\mathcal{M}\cap \mathcal{T}_t \neq \emptyset$. We have,
\begin{align*}
T_1 & \leq \frac{\sqrt{1-\delta}}{|\mathcal{U}_t|} \left[ \sum_{i \in \mathcal{M}} \|\nabla F_i(w_t) \| + \alpha m \max_{i \in \mathcal{M}}\|\nabla F_i(w_t) \| \right ] \\
& \leq \underbrace{\frac{\sqrt{1-\delta}}{|\mathcal{U}_t|} \left[ \sum_{i \in \mathcal{M}} \|\nabla F_i(w_t) - \nabla F(w_t) \| + \sum_{i \in \mathcal{M}}\| \nabla F(w_t) \|\right ] }_{T_3} \\
& + \underbrace{ \frac{\alpha m \sqrt{1-\delta}}{|\mathcal{U}_t|} \max_{i \in \mathcal{M}} \left[ \| \nabla F_i(w_t) - \nabla F(w_t)\| + \|\nabla F(w_t)\| \right]}_{T_4}
\end{align*}
We now upper-bound $T_3$. We have
\begin{align*}
T_3 & \leq \frac{\sqrt{1-\delta} |\mathcal{M}|}{|\mathcal{U}_t|} \max_{i \in \mathcal{M}} \|\nabla F_i(w_t) - \nabla F(w_t) \| + \frac{\sqrt{1-\delta} |\mathcal{M}|}{|\mathcal{U}_t|} \| \nabla F(w_t) \| \\
& \leq \frac{\sqrt{1-\delta} (1-\alpha)}{(1-\beta)}\max_{i \in \mathcal{M}} \|\nabla F_i(w_t) - \nabla F(w_t) \| + \frac{\sqrt{1-\delta} (1-\alpha)}{(1-\beta)} \| \nabla F(w_t) \| \\
& \leq \frac{\sqrt{1-\delta} (1-\alpha)}{(1-\beta)} \epsilon_1 + \frac{\sqrt{1-\delta} (1-\alpha)}{(1-\beta)} \| \nabla F(w_t) \|
\end{align*}
with probability exceeding 
 $1-\frac{2(1-\alpha)md}{(1+n \hat{L} D)^d}$, where we use Lemma~\ref{lem:ind_control}. Similarly, for $T_4$, we have
\begin{align*}
T_4 \leq \frac{\sqrt{1-\delta} \alpha}{1-\beta} \epsilon_1 + \frac{\sqrt{1-\delta} \alpha}{1-\beta} \|\nabla F(w_t) \|.
\end{align*}
We now control the terms in $T_2$. We obtain the following:
\begin{align*}
T_2 & \leq \frac{1}{|\mathcal{U}_t|} \|\sum_{i \in \mathcal{U}_t}  \nabla F_i(w_t) - \nabla F(w_t) \| \\
& \leq \frac{1}{|\mathcal{U}_t|} \| \sum_{i \in \mathcal{M}}  (\nabla F_i(w_t) - \nabla F(w_t)) - \sum_{i \in \mathcal{M} \cap \mathcal{T}_t}  (\nabla F_i(w_t) - \nabla F(w_t)) + \sum_{i \in \mathcal{B} \cap \mathcal{T}_t}  (\nabla F_i(w_t) - \nabla F(w_t)) \| \\
& \leq \frac{1}{|\mathcal{U}_t|} \| \sum_{i \in \mathcal{M}}  (\nabla F_i(w_t) - \nabla F(w_t)) \| + \frac{1}{|\mathcal{U}_t|} \|\sum_{i \in \mathcal{M} \cap \mathcal{T}_t}  (\nabla F_i(w_t) - \nabla F(w_t))\| \\
& + \frac{1}{|\mathcal{U}_t|} \|\sum_{i \in \mathcal{B} \cap \mathcal{T}_t}  (\nabla F_i(w_t) - \nabla F(w_t))\|.
\end{align*}
Using Lemma~\ref{lem:machine_control}, we have
\begin{align*}
\frac{1}{|\mathcal{U}_t|} \| \sum_{i \in \mathcal{M}}  (\nabla F_i(w_t) - \nabla F(w_t)) \| \leq \frac{1-\alpha}{1-\beta} \epsilon_2.
\end{align*}
with probability exceeding $1- \frac{2(1-\alpha)md}{(1+n \hat{L} D)^d} - \frac{2d}{(1+(1-\alpha)m n \hat{L} D)^d}$. Also, we obtain
\begin{align*}
\frac{1}{|\mathcal{U}_t|} \|\sum_{i \in \mathcal{M} \cap \mathcal{T}_t}  (\nabla F_i(w_t) - \nabla F(w_t))\| \leq \frac{\beta}{1-\alpha} \max_{i \in \mathcal{M}} \| \nabla F_i(w_t) - \nabla F(w_t) \| \leq \frac{\beta}{1-\alpha} \epsilon_1,
\end{align*}
with probability at least $1-\frac{2(1-\alpha)md}{(1+n \hat{L} D)^d}$, where the last inequality is derived from Lemma~\ref{lem:ind_control}. Finally, for the Byzantine term, we have
\begin{align*}
\frac{1}{|\mathcal{U}_t|} \|\sum_{i \in \mathcal{B} \cap \mathcal{T}_t}  (\nabla F_i(w_t) - \nabla F(w_t))\| & \leq \frac{\alpha}{1-\beta} \max_{i \in \mathcal{B} \cap \mathcal{T}_t} \|\nabla F_i(w_t)\| + \frac{\alpha}{1-\beta}\| \nabla F(w_t) \| \\
& \leq \frac{\alpha}{1-\beta} \max_{i \in \mathcal{M}} \|\nabla F_i(w_t)\|+ \frac{\alpha}{1-\beta}\| \nabla F(w_t) \| \\
& \leq \frac{\alpha}{1-\beta} \max_{i \in \mathcal{M}} \|\nabla F_i(w_t) - \nabla F(w_t)\|+ \frac{2\alpha}{1-\beta}\| \nabla F(w_t) \| \\
& \leq \frac{\alpha}{1-\beta} \epsilon_1 + \frac{2\alpha}{1-\beta}\| \nabla F(w_t) \|,
\end{align*}
with high probability, where the last inequality follows from Lemma~\ref{lem:ind_control}.

Combining all the terms of $T_1$ and $T_2$, we obtain,
\begin{align*}
\|\Delta\| \leq \frac{\sqrt{1-\delta} + 2\alpha}{1-\beta} \| \nabla F(w_t) \| + \frac{\sqrt{1-\delta} + \alpha + \beta}{1-\beta} \epsilon_1 + \frac{1-\alpha}{1-\beta} \epsilon_2.
\end{align*}
Now, using Young's inequality, for any $\lambda > 0$, we obtain
\begin{align*}
\| \Delta \|^2 \leq (1+\lambda) \left( \frac{\sqrt{1-\delta} + 2\alpha}{1-\beta} \right)^2 \| \nabla F(w_t) \|^2 + \tilde{\epsilon}(\lambda) 
\end{align*}
where
\begin{align*}
\tilde{\epsilon}(\lambda) = 2(1+\frac{1}{\lambda}) \left [ \left(\frac{\sqrt{1-\delta}+\alpha + \beta}{1-\beta} \right)^2 \epsilon_1^2 + \left( \frac{1-\alpha}{1-\beta} \right)^2 \right ] \epsilon_2^2.
\end{align*}

\subsection{Proof of Lemma~\ref{lem:ind_control}:}
For a fixed $i \in \mathcal{M}$, we first analyze the quantity $\|\nabla F_i(w_t) - \nabla F(w_t) \|$. Notice that $i$ is non-Byzantine. Recall that machine $i$ has $n$ independent data points. We use the sub-exponential concentration to control this term. Let us rewrite the concentration inequality.

\emph{Univariate sub-exponential concentration:} Suppose $Y$ is univariate random variable with $\E Y = \mu$ and $y_1,\ldots,y_n$ are i.i.d draws of $Y$. Also, $Y$ is $v$ sub-exponential. From sub-exponential concentration (Hoeffding's inequality), we obtain
\begin{align*}
    \Prob \left( |\frac{1}{n}\sum_{i=1}^n y_i -\mu| > t \right) \leq 2 \exp\{-n\min(\frac{t}{v},\frac{t^2}{v^2})\}.
\end{align*}

We directly use this to the $k$-th partial derivative of $F_i$. Let $\partial_k f (w_t, z^{i,j})$ be the partial derivative of the loss function with respect to $k$-th coordinate on $i$-th machine with $j$-th data point. From Assumption~\ref{asm:sub-exp}, we obtain
\begin{align*}
    \Prob \left(  |\frac{1}{n}\sum_{j=1}^n \partial_k  f (w_t, z^{i,j}) - \partial_k F(w_t)| \geq t \right)   \leq 2 \exp\{-n\min \left (\frac{t}{v},\frac{t^2}{v^2} \right)\}.
\end{align*}

Since $\nabla F_i (w_t) = \frac{1}{n}\sum_{j=1}^n \nabla f (w_t, z^{i,j})$, denoting $\nabla F_i^{(k)}(w_t)$ as the $k$-th coordinate of $\nabla F_i(w_t)$, we have
\begin{align*}
    |\nabla F_i^{(k)} (w_t) - \partial_k F(w_t)| \leq t
\end{align*}
with probability at least $1- 2 \exp\{-n\min(\frac{t}{v},\frac{t^2}{v^2})\}$.

This result holds for a particular $w_t$. To extend this for all $w \in \mathcal{W}$, we exploit the covering net argument and the Lipschitz continuity of the partial derivative of the loss function (Assumption~\ref{asm:struct_loss}). Let $\{w_1,\ldots,w_N\}$ be a $\delta$ covering of $\mathcal{W}$. Since $\mathcal{W}$ has diameter $D$, from Vershynin, we obtain $N \leq (1+\frac{D}{\delta})^d$. Hence with probability at least
\begin{align*}
    1-2 N d \exp\{-n\min \left(\frac{t}{v},\frac{t^2}{v^2} \right )\},
\end{align*}
we have 
\begin{align*}
    |\nabla F_i^{(k)} (w) - \partial_k F(w)| \leq t
\end{align*}
for all $w \in \{w_1,\ldots,w_N\}$ and $k \in [d]$. This implies
\begin{align*}
    \|\nabla F_i(w_t) - \nabla F(w_t) \| \leq t \sqrt{d},
\end{align*}
with probability greater than or equal to $1-2 N d \exp\{-n\min(\frac{t}{v},\frac{t^2}{v^2})\}$.

We now reason about $w \in \mathcal{W} \setminus \{w_1,\ldots,w_N \}$ via Lipschitzness (Assumption~\ref{asm:struct_loss}). From the definition of $\delta$ cover, for any $w \in \mathcal{W}$, there exists $w_\ell$, an element of the cover such that $\|w-w_\ell\| \leq \delta$. Hence, we obtain
\begin{align*}
    |\nabla F_i^{(k)} (w) - \partial_k F(w)| \leq t+ 2L_k \delta
\end{align*}
for all $w \in \mathcal{W}$ and consequently
\begin{align*}
     \|\nabla F_i(w_t) - \nabla F(w_t) \| \leq \sqrt{d}\, t + 2 \delta \hat{L}
\end{align*}
with probability at least $1-2 N d \exp\{-n\min(\frac{t}{v},\frac{t^2}{v^2})\}$, where $\hat{L}= \sqrt{\sum_{k=1}^d L_k^2}$.

Choosing $\delta = \frac{1}{2n \hat{L}}$ and
\begin{align*}
 t = v \max \{ \frac{d}{n} \log(1+2n\hat{L}d), \sqrt{\frac{d}{n} \log(1+2n\hat{L}d)}\},
\end{align*}
we obtain
\begin{align}
     \|\nabla F_i(w_t) - \nabla F(w_t) \|  \leq v \sqrt{d} \left ( \max \{ \frac{d}{n} \log(1+2n\hat{L}d), \sqrt{\frac{d}{n} \log(1+2n\hat{L}d)}\} \right ) + \frac{1}{n} = \epsilon_1, \label{eqn:epsilon}
\end{align}
with probability greater than $1- \frac{d}{(1+n\hat{L}D)^d}$. Taking union bound on all non-Byzantine machines yields the theorem.

\subsection{Proof of Lemma~\ref{lem:machine_control}}
%Recall that $(\mathcal{B}\cap \mathcal{U}_t)$ are the byzantine servers that we use in the update i.e not get trimmed by the $\mathtt{Byzantine \, remove}$ algorithm and $(\mathcal{M}\cap \mathcal{T}_t)$ are the non-byzantine server that get trimmed as outlier by the $\mathtt{Byzantine \, remove}$ algorithm. We have
%\begin{align*}
%  & \frac{1}{|\mathcal{U}_t|} \|\sum_{i\in \mathcal{U}_t} (\nabla F_i(w_t) - \nabla F(w_t)) \| \\
%   \leq & \frac{1}{|\mathcal{U}_t|} \| \sum_{i\in \mathcal{M}}(\nabla F_i(w_t) - \nabla F(w_t))\\
%    & + \sum_{i \in \mathcal{M}\cap \mathcal{T}_t}(\nabla F_i(w_t) - \nabla F(w_t)) \\
%   & -\sum_{i \in \mathcal{B}\cap \mathcal{U}_t} (\nabla F_i(w_t) - \nabla F(w_t))\|
%\end{align*}
%Using triangle inequality,
%\begin{align*}
%  & \frac{1}{|\mathcal{U}_t|} \|\sum_{i\in \mathcal{U}_t} (\nabla F_i(w_t) - \nabla F(w_t)) \\
%   \leq &  \frac{1}{|\mathcal{U}_t|} \| \sum_{i\in \mathcal{M}}(\nabla F_i(w_t) - \nabla F(w_t))\| \\
%  & + \frac{1}{|\mathcal{U}_t|} \| \sum_{i\in \mathcal{M}\cap T_t}(\nabla F_i(w_t) - \nabla F(w_t))\| \\
%  & + \frac{1}{|\mathcal{U}_t|} \| \sum_{i\in \mathcal{B} \cap \mathcal{U}_t}(\nabla F_i(w_t) - \nabla F(w_t))\|.
%\end{align*}
%
%\paragraph{Control of $\frac{1}{|\mathcal{U}_t|} \| \sum_{i\in \mathcal{M}}(\nabla F_i(w_t) - \nabla F(w_t))\|$:}

We need to upper bound the following quantity:
\begin{align*}
 \| \frac{1}{|\mathcal{M}|} \sum_{i\in \mathcal{M}}(\nabla F_i(w_t) - \nabla F(w_t))\|
\end{align*}

We now use similar argument (sub-exponential concentration) like Lemma~\ref{lem:ind_control}. The only difference is that in this case, we also consider \emph{averaging} over worker nodes. We obtain the following:
\begin{align*}
    \| \frac{1}{|\mathcal{M}|} \sum_{i\in \mathcal{M}}(\nabla F_i(w_t) - \nabla F(w_t))\| \leq \epsilon_2
\end{align*}
where 
\begin{align*}
    \epsilon_2 = v \sqrt{d} & \bigg (\max \{ \frac{d}{(1-\alpha)m n}\log (1+2(1-\alpha)m n\hat{L}d), \sqrt{\frac{d}{(1-\alpha)m n}\log (1+2(1-\alpha)m n\hat{L}d)} \} \bigg ),
\end{align*}
with probability $1-\frac{2d}{(1+(1-\alpha)m n \hat{L} D)^d}$.

%\paragraph{Control of the other terms:}
%We have
%\begin{align*}
%   & \frac{1}{|\mathcal{U}_t|} \| \sum_{i\in \mathcal{M}\cap T_t}(\nabla F_i(w_t) - \nabla F(w_t))\| \\
%   & \leq \frac{\beta}{1-\beta}\, \max_{i \in \mathcal{M}}  \|\nabla F_i(w_t) - \nabla F(w_t))\| \leq \frac{\beta}{1-\beta} \, \epsilon_1
%\end{align*}
%with probability greater than $1-\frac{2(1-\alpha)m d}{(1+n \hat{L}D)^d}$. Also since $\beta \geq \alpha$, without loss of generality, we have \swa{This needs to be verified.}
%\begin{align*}
% & \frac{1}{|\mathcal{U}_t|} \| \sum_{i\in \mathcal{B} \cap \mathcal{U}_t}(\nabla F_i(w_t) - \nabla F(w_t))\| \\
%  & \leq \frac{\alpha}{1-\beta} \,   \max_{i \in \mathcal{M}}  \|\nabla F_i(w_t) - \nabla F(w_t))\| \leq \frac{\alpha}{1-\beta} \, \epsilon_1
%\end{align*}
%with probability exceeding $1-\frac{2(1-\alpha) m d}{(1+n \hat{L}D)^d}$.
%
%Hence, combining the terms together, we obtain
%\begin{align*}
%    \frac{1}{|\mathcal{U}_t|} \|\sum_{i\in \mathcal{U}_t} (\nabla F_i(w_t) - \nabla F(w_t)) \| \leq \frac{1-\alpha}{1-\beta}\, \epsilon_2 + \frac{\alpha + \beta}{1-\beta} \, \epsilon_1.
%\end{align*}

\subsection{Proof of Lemma~\ref{lem:delta_control1}}
\label{subsec:proof_delta_control1}

Here we prove an upper bound on the norm of
\begin{align*}
    \widetilde{\Delta} = g(w_t) -\nabla F(w_t)
\end{align*}
where $g(w_t)=\frac{1}{|\mathcal{U}_t|} \sum_{i \in \mathcal{U}_t}Q(\nabla F_i(w_t))$.
%In the proof we are going to use the following useful result
%\begin{itemize}
%    \item If $\cM$ is the set of good machine and 
%    \begin{align}
%        |\cU_t| = |\cM| - |\cM \cap \cT_t| + |\cB \cap \cU_t| \label{mcal}  
%    \end{align}
%    \item The following has been proved on the assumption of sub-exponential gradient.
%    \begin{align}
%        || \nabla f_i(w_t)-\nabla F(w_t)|| \leq \epsilon_1 \label{subexp}
%    \end{align}
%    with high probability.
%\end{itemize}

We have 
\begin{align*}
    ||\widetilde{\Delta}|| =& || \frac{1}{|\mathcal{U}_t|} \sum_{i \in \mathcal{U}_t}Q(\nabla F_i(w_t)) -\nabla F(w_t) ||\\
    &= \frac{1}{|\cU_t|} || \sum_{i \in \cM}[Q(\nabla F_i(w_t)) -\nabla F(w_t)] - \sum_{i \in (\cM\cap \cT_t )}[Q(\nabla F_i(w_t)) -\nabla F(w_t)] \\
    & + \sum_{i \in (\cB\cap \cU_t )}[Q(\nabla F_i(w_t)) -\nabla F(w_t)] ||\\
    & \leq \frac{1}{|\cU_t|}\bigg(\underbrace{ || \sum_{i \in \cM}Q(\nabla F_i(w_t)) -\nabla F(w_t)||}_{T_1} +\underbrace{ ||\sum_{i \in (\cM\cap \cT_t )}Q(\nabla F_i(w_t)) -\nabla F(w_t)||}_{T_2}\\
    & +\underbrace{ ||\sum_{i \in (\cB\cap \cU_t )}Q(\nabla F_i(w_t)) -\nabla F(w_t)||}_{T_3} \bigg)
\end{align*}
Now we bound each term separately. For the first term, we have
\begin{align*}
  \frac{1}{|\cU_t|}  T_1 & =\frac{1}{|\cU_t|}|| \sum_{i \in \cM}Q(\nabla F_i(w_t)) -\nabla F(w_t)|| \\
    &=\frac{1}{|\cU_t|} ||\sum_{i \in \cM}Q(\nabla F_i(w_t))-\nabla F_i(w_t)|| + \frac{1}{|\cU_t|} ||\sum_{i \in \cM} \nabla F_i(w_t)-\nabla F(w_t)||\\
   &\leq \frac{1}{|\cU_t|} \sum_{i \in \cM} \bigg(||Q(\nabla F_i(w_t))-\nabla F_i(w_t)|| \bigg) + \frac{1-\alpha}{1-\beta}\epsilon_2 \\
   & \leq \frac{1}{|\cU_t|} \sum_{i \in \cM} \bigg(  \sqrt{1-\delta}||\nabla F_i(w_t)||\bigg) + \frac{1-\alpha}{1-\beta}\epsilon_2 \\
   & \leq \frac{\sqrt{1-\delta}}{|\cU_t|}\sum_{i \in \cM} \bigg(  ||\nabla F(w_t)||  +||\nabla F_i(w_t)-\nabla F(w_t)||\bigg) +\frac{1-\alpha}{1-\beta}\epsilon_2\\
  & \leq \frac{\sqrt{1-\delta}(1-\alpha)}{1-\beta} ||\nabla F(w_t)||  +\frac{\sqrt{1-\delta}(1-\alpha)}{1-\beta}\epsilon_1 +\frac{1-\alpha}{1-\beta}\epsilon_2
\end{align*}
where we use the definition of a $\delta$-approximate compressor, Lemma~\ref{lem:ind_control} and Lemma~\ref{lem:machine_control}. Similarly, we can bound $T_2$ as
\begin{align*}
    T_2 & \leq \sum_{i \in (\cM\cap \cT_t )} || Q(\nabla F_i(w_t)) -\nabla F(w_t)||\\
    & \leq \beta m \max_{i \in \mathcal{M}} || Q(\nabla F_i(w_t)) -\nabla F(w_t)|| \\
     & \leq  \beta m \max_{i \in \mathcal{M}} \bigg(\sqrt{1-\delta}|| \nabla F_i(w_t))|| + ||\nabla F_i(w_t) -\nabla F(w_t)|| \bigg)\\
     & \leq \beta m \max_{i \in \mathcal{M}} \bigg(\sqrt{1-\delta}|| \nabla F(w_t))|| +(1+\sqrt{1-\delta}) ||\nabla F_i(w_t) -\nabla F(w_t)|| \bigg)
     \end{align*}
where we use the definition of $\delta$-approximate compressor. Hence invoking Lemma~\ref{lem:ind_control}, we obtain
\begin{align*}
  \frac{1}{|\cU_t|}  T_2&\leq \frac{\beta\sqrt{1-\delta}}{1-\beta} || \nabla F(w_t))|| +    \frac{\beta(1+\sqrt{1-\delta})}{1-\beta}\epsilon_1
\end{align*}
 Also, owing to the trimming with $\beta > \alpha$, we have at least one good machine in the set $\cT_t$ for all $t$. Now each term in the set $\cB \cap \cU_t$, we have
\begin{align*}
    T_3 &= \sum_{i \in (\cB\cap \cU_t )}||Q(\nabla F_i(w_t)) -\nabla F(w_t)|| \\
   & \leq \alpha m (\max_{i\in \cM} ||Q(\nabla F_i(w_t))|| +||\nabla F(w_t)||)\\
   & \leq \alpha m (\max_{i\in \cM} \sqrt{1-\delta}|| \nabla F_i(w_t)|| +|| \nabla F_i(w_t)||+||\nabla F(w_t)||) \\
   & \leq  \alpha m \bigg((1+ \sqrt{1-\delta})\epsilon_1 +(2+\sqrt{1-\delta})||\nabla F(w_t)||\bigg)\\
  \frac{1}{|\cU_t|}  T_3 & \leq \frac{\alpha(2+\sqrt{1-\delta})}{1-\beta}||\nabla F(w_t)||   +\frac{\alpha(1+\sqrt{1-\delta})}{1-\beta}\epsilon_1
\end{align*}
where we use Lemma~\ref{lem:ind_control}. Putting $T_1,T_2,T_3$ we get 
\begin{align*}
||\widetilde{\Delta} || & \leq \bigg(\frac{\sqrt{1-\delta}(1-\alpha)}{1-\beta} +\frac{\beta\sqrt{1-\delta}}{1-\beta}+\frac{\alpha(2+\sqrt{1-\delta})}{1-\beta}  \bigg)||\nabla F(w_t)|| \\
& + \bigg(\frac{\sqrt{1-\delta}(1-\alpha)}{1-\beta} + \frac{\beta(1+\sqrt{1-\delta})}{1-\beta}+\frac{\alpha(1+\sqrt{1-\delta})}{1-\beta}\bigg) \epsilon_1 
+\frac{1-\alpha}{1-\beta}\epsilon_2\\
&=\bigg(\frac{(1+\beta)\sqrt{1-\delta} +2\alpha}{1-\beta} \bigg)||\nabla F(w_t)||+ \bigg(\frac{(1+\beta)\sqrt{1-\delta} +\alpha+\beta}{1-\beta} \bigg)\epsilon_1+\frac{1-\alpha}{1-\beta}\epsilon_2\\
||\widetilde{\Delta}||^2&\leq (1+\lambda)\bigg(\frac{(1+\beta)\sqrt{1-\delta} +2\alpha}{1-\beta} \bigg)^2||\nabla F(w_t)||^2 + \widetilde{\epsilon}(\lambda)
\end{align*}
where $\widetilde{\epsilon}(\lambda) = 2(1+\frac{1}{\lambda})\bigg(\bigg(\frac{(1+\beta)\sqrt{1-\delta} +\alpha+\beta}{1-\beta} \bigg)^2\epsilon^2_1+(\frac{1-\alpha}{1-\beta})^2\epsilon_2^2 \bigg)$.
Hence, the lemma follows.

\section{Bound on Gradient Norm}
\label{sec:bounded}

\begin{proposition}\label{prop:bounded}
Consider the $i$-th worker machine has data label pair $(X,y)$ where $ X\in \mathbb{R}^{n\times d}$ and $y \in \mathbb{R}^n$, and $n \geq d$. Suppose each data point $j \in [n]$ is  generated by $y_j= x_j^Tw^* + \eta_j$ for some $d$ dimensional regressor $w^* \in \mathcal{W}$ and  noise $\eta_j \sim \mathcal{N}(0,1)$ drawn independently. Moreover, assume that the data points $x_j \sim \mathcal{N}(0,I_d)$ for all $j \in [n]$, and the loss function is given by $F_i(w;X,y)= \frac{1}{2n}\|y-Xw\|^2$. We obtain $\|\nabla F_i (w; X,y)\|^2 \lesssim D^2$ with probability exceeding $1-4\exp(-c n)$.
\end{proposition}

\begin{corollary}
Suppose we consider a fixed design setup, where the data matrix, $X$ is deterministic, with $\|X\|_{op} \leq R$. In this framework, with the same least squared loss, we have, $\|\nabla F_i (w; X,y)\|^2 \lesssim (\frac{R^2 D}{n} + \frac{R}{\sqrt{n}})^2$ with probability exceeding $1-4\exp(-c n)$.
\end{corollary}

\begin{proof}
The gradient of the loss function is given by
\begin{align*}
   \| \nabla F_i(w;X,y)\| &= \| \frac{1}{n}X^T(Xw-y)\| \\
    & = \frac{1}{n} \| X^T (Xw - Xw^* + \eta) \| \\
    & \leq \frac{1}{n} \|X\|_{op} \|X(w-w^*) + \eta \| \\
    & \leq \frac{1}{n} \|X\|_{op} \left( \|X(w-w^*)\| + \|\eta\| \right) \\
    & \leq \frac{1}{n} \|X\|_{op} \left( \|X\|_{op} \|(w-w^*)\| + \|\eta\| \right),
\end{align*}
where the second line uses the definition of $y$, the third line uses the definition of $\ell_2$ operator norm, and the fourth line uses triangle inequality. Since $X \in \mathbb{R}^{n \times d}$ is a random matrix, with iid standard Gaussian entries, we have (\cite{vershynin2010introduction}), 
\begin{align*}
    \|X\|_{op} \lesssim \sqrt{n} + \sqrt{d} \lesssim \sqrt{n},
\end{align*}
with probability at least $1-2\exp(-c d)$.

Furthermore, since $\eta \in \mathbb{R}^n$ has iid Gaussian entries, from chi-squared concentration (\cite{wainwright}), we obtain
\begin{align*}
    \|\eta\| \lesssim \sqrt{n},
\end{align*}
with probability at least $1-2\exp(-c_1 n)$. 
Furthermore, we have, from definition, $\|w-w^*\| \leq D$, where $D$ is the diameter of $\mathcal{W}$. Putting everything together, we obtain
\begin{align*}
     \| \nabla F_i(w;X,y)\| \lesssim \frac{1}{n} \left( n D + n \right) \lesssim (1 + D),
\end{align*}
with probability at least $1-4\exp(-c n)$.

The proof of the corollary is immediate with the observation that $\|X\|_{op} \leq R$.

\end{proof}

\section{Proof of Theorem~\ref{thm:non_convex_err}}
\label{sec:proof}

We first define an auxiliary sequence defined as:
\begin{align*}
    \widetilde{w}_t = w_t - \frac{1}{|\cM|}\sum_{i \in \cM}\er
\end{align*}
Hence, we obtain
\begin{align*}
    \widetilde{w}_{t+1} = w_{t+1} - \frac{1}{|\cM|}\sum_{i \in \cM}\err.
\end{align*}
For notational simplicity, let us drop the subscript $t$ from $\cU_t$ and $\cT_t$ and denote them as $\cU$ and $\cT$.

Since (we will ensure that the iterates remain in the parameter space and hence we can ignore the projection step),
\begin{align*}
    w_{t+1} = w_t - \frac{1}{|\cU|}\sum_{i \in \cU}\pr,
\end{align*}
we get
\begin{align*}
    \widetilde{w}_{t+1} & =  w_t - \frac{1}{|\cU|}\sum_{i \in \cU}\C(\pr) - \frac{1}{|\cM|}\sum_{i \in \cM}\err \\
    & = w_t - \frac{1}{|\cU|}\left( \sum_{i \in \cM}\C(\pr) + \sum_{i \in \cB \cap \cU}\C(\pr) - \sum_{i \in \cM \cap \cT } \C(\pr) \right) - \frac{1}{|\cM|}\sum_{i \in \cM}\err \\
    & = w_t - \left( \frac{1-\alpha}{1-\beta}\right) \frac{1}{|\cM|}\sum_{i \in \cM} \C(\pr) - \frac{1}{|\cM|}\sum_{i \in \cM}\err - \frac{1}{|\cU|}\sum_{i \in \cB \cap \cU}\C(\pr) + \frac{1}{|\cU|}\sum_{i \in \cM \cap \cT}\C(\pr)
\end{align*}
Since $\C(\pr)+\err = \pr$ for all $i \in \cM$, we obtain
\begin{align*}
   \left( \frac{1-\alpha}{1-\beta}\right) \frac{1}{|\cM|}\sum_{i \in \cM} \C(\pr) + \frac{1}{|\cM|}\sum_{i \in \cM}\err = \frac{1}{|\cM|}\sum_{i \in \cM}\pr +\frac{\beta - \alpha}{1-\beta} \frac{1}{|\cM|} \sum_{i \in \cM}\C(\pr)
\end{align*}
Let us denote
$T_1 = \frac{1}{|\cU|}\sum_{i \in \cB \cap \cU}\C(\pr)$, $T_2 = \frac{1}{|\cU|}\sum_{i \in \cM \cap \cT}\C(\pr)$ and $T_3 = \frac{\beta - \alpha}{1-\beta} \frac{1}{|\cM|} \sum_{i \in \cM}\C(\pr)$. With this, we obtain
\begin{align*}
    \widetilde{w}_{t+1} &= w_t -\frac{1}{|\cM|}\pr -T_1 +T_2 -T_3 \\
    & = \widetilde{w}_t + \frac{1}{|\cM|}\sum_{i \in \cM}\er -\frac{1}{|\cM|}\sum_{i \in \cM} \pr - \tilde{T} \\
    &= \widetilde{w}_t - \gamma \frac{1}{|\cM|}\sum_{i \in \cM}\gr - \tilde{T}
\end{align*}
where $\tilde{T} = T_1 -T_2 + T_3$. Observe that the auxiliary sequence looks similar to a distributed gradient step with a presence of $\tilde{T}$. For the convergence analysis, we will use this relation along with an upper bound on $\|\tilde{T}\|$.
\vspace{6mm}

Using this auxiliary sequence, we first ensure that the iterates of our algorithm remains close to one another. To that end, we have
\begin{align*}
    w_{t+1} - w_t & = \widetilde{w}_{t+1} - \widetilde{w}_t + \frac{1}{|\cM|}\err - \frac{1}{|\cM|}\er \\
    & = - \gamma \frac{1}{|\cM|}\sum_{i \in \cM}\gr - \tilde{T} + \frac{1}{|\cM|}\err - \frac{1}{|\cM|}\er.
\end{align*}
Hence, we obtain
\begin{align*}
    \|w_{t+1} - w_t\| & \leq \|\gamma \frac{1}{|\cM|}\sum_{i \in \cM}\gr\| + \|\tilde{T}\| + \|\frac{1}{|\cM|}\err\| + \|\frac{1}{|\cM|}\er\| \\
    & \leq \gamma \| \frac{1}{|\cM|}\sum_{i \in \cM}\gr - \nabla F(w_t) \| + \gamma \|\nabla F(w_t)\| + \|\tilde{T}\| + \|\frac{1}{|\cM|}\err\| + \|\frac{1}{|\cM|}\er\| \\
    & \leq \gamma \epsilon_2 + \gamma \|\nabla F(w_t)\| + \|\tilde{T}\| + \|\frac{1}{|\cM|}\err\| + \|\frac{1}{|\cM|}\er\|.
\end{align*}
Now, using Lemma~\ref{lem:control_of_error} and Lemma~\ref{lem:control_of_T} in conjunction with Assumption~\ref{asm:size_para} ensures the iterates of Algorithm~\ref{alg:err_feed} stays in the parameter space $\mathcal{W}$.
\vspace{6mm}

\noindent We assume that the global loss function $F(.)$ is $L_F$ smooth. We get
\begin{align*}
    F(\widetilde{w}_{t+1}) \leq F(\widetilde{w}_t) + \langle \nabla F(\widetilde{w}_t), \widetilde{w}_{t+1}- \widetilde{w}_t \rangle + \frac{L_F}{2}\|\widetilde{w}_{t+1}- \widetilde{w}_t\|^2.
\end{align*}
Now, we use the above recursive equation
\begin{align*}
    \widetilde{w}_{t+1} = \widetilde{w}_t - \gamma \frac{1}{|\cM|}\sum_{i \in \cM}\gr - \tilde{T}.
\end{align*}
Substituting, we obtain
\begin{align}
    F(\widetilde{w}_{t+1}) &\leq F(\widetilde{w}_t) - \gamma \langle \nabla F(\widetilde{w}_t), \frac{1}{|\cM|}\sum_{i \in \cM} \gr \rangle - \langle \nabla F(\widetilde{w}_t),\tilde{T}\rangle + \frac{L_F}{2}\| \frac{\gamma}{|\cM|}\sum_{i \in \cM}\gr + \tilde{T} \|^2 \nonumber \\
    & \leq F(\widetilde{w}_t) - \gamma \langle \nabla F(\widetilde{w}_t), \frac{1}{|\cM|}\sum_{i \in \cM} \gr \rangle - \langle \nabla F(\widetilde{w}_t),\tilde{T}\rangle + L_F\gamma^2 \|\frac{1}{|\cM|}\sum_{i \in \cM} \gr\|^2 + L_F \|\tilde{T}\|^2 \label{eqn:smoothness}
\end{align}
In the subsequent calculation, we use the following definition of smoothness:
\begin{align*}
    \| \nabla F(y_1) - \nabla F(y_2)\| \leq L_F \|y_1 - y_2\|
\end{align*}
for all $y_1$ and $y_2 \in \real^d$.

Rewriting the right hand side (R.H.S) of equation~\eqref{eqn:smoothness}, we obtain
\begin{align*}
 R.H.S & =  \underbrace{ F(\widetilde{w}_t) - \gamma \langle \nabla F(\widetilde{w}_t),\nabla F(w_t) \rangle}_{Term-I} + \underbrace{ \gamma \langle \nabla F(\widetilde{w}_t),\nabla F(w_t) - \frac{1}{|\cM|}\sum_{i \in \cM} \gr \rangle}_{Term-II} \\
   & + \underbrace{  \langle \nabla F(w_t), -\tilde{T} \rangle + \langle \nabla F(\widetilde{w}_t) - \nabla F(w_t), -\tilde{T} \rangle}_{Term-III} \\
   & + \underbrace{2L_F\gamma^2 \| \frac{1}{|\cM|}\sum_{i \in \cM} \gr - \nabla F(w_t)\|^2 + 2L_F\gamma^2 \|\nabla F(w_t)\|^2 + L_F\|\tilde{T}\|^2}_{Term-IV}.
\end{align*}
We now control the $4$ terms separately. We start with Term-I.
\paragraph*{Control of Term-I:}
We obtain
\begin{align*}
   \mbox{Term-I}& =F(\widetilde{w}_t) - \gamma \langle \nabla F(w_t), \nabla F(w_t) \rangle - \gamma \langle \nabla F(\widetilde{w}_t) - \nabla F(w_t), \nabla F(w_t) \rangle \\
   & \leq F(\widetilde{w}_t) - \gamma \|\nabla F(w_t) \|^2 + 25 \gamma \| \nabla F(\widetilde{w}_t) - \nabla F(w_t)\|^2 + \frac{\gamma}{100} \|\nabla F(w_t)\|^2,
\end{align*}
where we use Young's inequality ($\langle a,b \rangle \leq \frac{\rho}{2}\|a\|^2 + \frac{1}{2\rho}\|b\|^2$ with $\rho = 50$) in the last inequality. Using the smoothness of $F(.)$, we obtain
\begin{align}
    \mbox{Term-I} \leq F(\widetilde{w}_t) - \gamma \|\nabla F(w_t)\|^2 + \frac{\gamma}{100}\|\nabla F(w_t)\|^2 + 25 \gamma L_F^2 \| \frac{1}{|\cM|}\sum_{i \in \cM}\er \|^2. \label{eqn:term_one}
\end{align}
\paragraph*{Control of Term-II:}
Similarly, for Term-II, we have
\begin{align}
   \mbox{Term-II} &= \gamma \langle \nabla F(\widetilde{w}_t),\nabla F(w_t) - \frac{1}{|\cM|}\sum_{i \in \cM} \gr \rangle \leq 50 \gamma \epsilon_2^2 + \frac{\gamma}{200}\| \nabla F(\widetilde{w}_t) \|^2 \nonumber \\
   & \leq  50 \gamma \epsilon_2^2 + \frac{\gamma}{100}\|\nabla F(w_t)\|^2 + \frac{\gamma L_F^2}{100}\|\frac{1}{|\cM|}\sum_{i \in \cM}\er\|^2 \label{eqn:term_two}.
\end{align}
\paragraph*{Control of Term-III:}
We obtain
\begin{align}
    \mbox{Term-III} & = \langle \nabla F(w_t), -\tilde{T} \rangle + \langle \nabla F(\widetilde{w}_t) - \nabla F(w_t), -\tilde{T} \rangle \nonumber \\
    & \leq \frac{\gamma}{2}\|\nabla F(w_t) \|^2 + \frac{1}{2\gamma}\|\tilde{T}\|^2 + \frac{L_F^2}{2} \|\frac{1}{|\cM|}\sum_{i \in \cM}\er\|^2 +\frac{1}{2}\|\tilde{T}\|^2 \label{eqn:term_three}.
    \end{align}
\paragraph*{Control of Term-IV:}
\begin{align}
    \mbox{Term-IV}& = 2L_F\gamma^2 \| \frac{1}{|\cM|}\sum_{i \in \cM} \gr - \nabla F(w_t)\|^2 + 2L_F\gamma^2 \|\nabla F(w_t)\|^2 + L_F\|\tilde{T}\|^2 \nonumber \\
    & \leq 2L_F \gamma^2 \epsilon_2^2 + 2L_F\gamma^2 \|\nabla F(w_t)\|^2 + L_F \|\tilde{T}\|^2 \label{eqn:term_four}
\end{align}

Combining all $4$ terms, we obtain
\begin{align}
    F(\widetilde{w}_{t+1}) &\leq F(\widetilde{w}_t) - \left(\frac{\gamma}{2} -\frac{\gamma}{50} -2L_F\gamma^2 \right) \|\nabla F(w_t) \|^2 + \left(25\gamma L_F^2 + \frac{\gamma L_F^2}{100} + \frac{L_F^2}{2} \right) \|\frac{1}{|\cM|}\sum_{i \in \cM}\er\|^2 \nonumber \\
    &+ 50\gamma \epsilon_2^2 + 2L_F\gamma^2 \epsilon_2^2 + \left( \frac{1}{2\gamma}+\frac{1}{2}+L_F\right) \|\tilde{T}\|^2 \label{eqn:term_all}
\end{align}
We now control the error sequence and $\|\tilde{T}\|^2$. These will be separate lemmas, but here we write is as a whole.
\paragraph*{Control of error sequence:} 
\begin{lemma}
\label{lem:control_of_error}
For all $i \in \cM$, we have
\begin{align*}
\| e_i(t)\|^2 \leq \frac{3(1-\delta)}{\delta}\gamma^2 \sigma^2
\end{align*}
for all $t \geq 0$.
\end{lemma}
\begin{proof}
For machine $i \in \cM$, we have
\begin{align*}
    \|\err\|^2 = \|\C(\pr) -\pr\|^2 \leq (1-\delta)\|\pr\|^2=(1-\delta)\| \gamma \gr + \er\|^2
\end{align*}
Using technique similar to the proof of \cite[Lemma 3]{errorfeed} and using $\| \gr\|^2 \leq \sigma^2$, we obtain
\begin{align*}
    \|\err\|^2 \leq \frac{2(1-\delta)(1+1/\eta)}{\delta}\gamma^2 \sigma^2
\end{align*}
where $\eta > 0$. Substituting $\eta = 2$ implies
\begin{align}
    \|\err\|^2 \leq \frac{3(1-\delta)}{\delta}\gamma^2 \sigma^2
    \label{eqn:error}
\end{align}
for all $i \in \cM$. This also implies
\begin{align*}
    \max_{i \in \cM}\|\err\|^2 \leq \frac{3(1-\delta)}{\delta}\gamma^2 \sigma^2.
\end{align*}
\end{proof}
\paragraph*{Control of $\|\tilde{T}\|^2$:}

\begin{lemma}
\label{lem:control_of_T}
We obtain
\begin{align*}
\|\tilde{T}\|^2 \leq \frac{9(1+\sqrt{1-\delta})^2 \gamma^2}{(1-\beta)^2} \left[\alpha^2 + \beta^2 + (\beta - \alpha)^2 \right]\left(  \epsilon_1^2 +  \|\nabla F(w_t)\|^2 + \frac{3(1-\delta)}{\delta} \sigma^2 \right)
\end{align*}
with probability exceeding $1-\frac{2(1-\alpha)md}{(1+n \hat{L} D)^d}$.
\end{lemma}
\begin{proof}
We have
\begin{align*}
    \|\tilde{T}\| = \|T_1 -T_2 +T_3\| \leq \|T_1\| + \|T_2\| + \|T_3\|.
\end{align*}
We control these $3$ terms separately. We obtain
\begin{align*}
    \|T_1\| = \|\frac{1}{|\cU|}\sum_{i \in \cB \cap \cU} \C(\pr)\| \leq \frac{1}{(1-\beta)m} \sum_{i \in \cB \cap \cU} \|\C(\pr)\|.
\end{align*}
Since the worker machines are sorted according to $\|\C(\pr)\|$ (the central machine only gets to see $\C(\pr)$, and so the most natural metric to sort is $\|\C(\pr)\|$), we obtain
\begin{align*}
    \|T_1\| &\leq \frac{\alpha m}{(1-\beta)m} \max_{i \in \cM}\|\C(\pr)\| \\
    & \leq (1+\sqrt{1-\delta}) \frac{\alpha m}{(1-\beta)m} \max_{i \in \cM}\|\pr\| \\
    & \leq (1+\sqrt{1-\delta}) \frac{\alpha m}{(1-\beta)m} \max_{i \in \cM}\|\gamma \gr + \er\| \\
    & \leq (1+\sqrt{1-\delta}) \frac{\alpha }{(1-\beta)} \gamma \max_{i \in \cM}\|\gr - \nabla F(w_t)\| + (1+\sqrt{1-\delta}) \frac{\alpha }{(1-\beta)} \gamma \|\nabla F(w_t)\| \\
    & + (1+\sqrt{1-\delta}) \frac{\alpha }{(1-\beta)}  \max_{i \in \cM}\|\er \| \\
    & \leq (1+\sqrt{1-\delta}) \frac{\alpha \gamma \epsilon_1 }{(1-\beta)}  + (1+\sqrt{1-\delta}) \frac{\alpha \gamma}{(1-\beta)}  \|\nabla F(w_t)\| \\
    & + (1+\sqrt{1-\delta}) \frac{\alpha \gamma \sigma }{(1-\beta)}  \sqrt{\frac{3(1-\delta)}{\delta}}.
\end{align*}
Hence,
\begin{align*}
    \|T_1\|^2 \leq 3 \frac{(1+\sqrt{1-\delta})^2}{(1-\beta)^2} \alpha^2 \gamma^2 \left(  \epsilon_1^2 +  \|\nabla F(w_t)\|^2 + \frac{3(1-\delta)}{\delta} \sigma^2 \right).
\end{align*}
Similarly, we obtain,
\begin{align*}
    \|T_2\|^2 \leq 3 \frac{(1+\sqrt{1-\delta})^2}{(1-\beta)^2} \beta^2 \gamma^2 \left(  \epsilon_1^2 +  \|\nabla F(w_t)\|^2 + \frac{3(1-\delta)}{\delta} \sigma^2 \right).
\end{align*}
For $T_3$, we have
\begin{align*}
    \|T_3\| & = \frac{\beta - \alpha}{1-\beta} \| \frac{1}{|\cM|} \sum_{i \in \cM}\C(\pr)\| \leq \frac{\beta - \alpha}{1-\beta} \frac{1}{|\cM|}\sum_{i \in \cM} (1+\sqrt{1-\delta}) \|\pr\| \\
    & \leq (1+\sqrt{1-\delta})\frac{\beta - \alpha}{1-\beta}\max_{i \in \cM}\|\pr\|
\end{align*}
Using the previous calculation, we obtain
\begin{align*}
    \|T_3\|  & \leq (1+\sqrt{1-\delta}) \frac{(\beta -\alpha) \gamma \epsilon_1 }{(1-\beta)}  + (1+\sqrt{1-\delta}) \frac{(\beta -\alpha) \gamma}{(1-\beta)}  \|\nabla F(w_t)\| \\
    & + (1+\sqrt{1-\delta}) \frac{(\beta -\alpha) \gamma \sigma }{(1-\beta)}  \sqrt{\frac{3(1-\delta)}{\delta}}, 
\end{align*}
and as a result,
\begin{align*}
    \|T_3\|^2 \leq 3 \frac{(1+\sqrt{1-\delta})^2}{(1-\beta)^2} (\beta - \alpha)^2 \gamma^2 \left(  \epsilon_1^2 +  \|\nabla F(w_t)\|^2 + \frac{3(1-\delta)}{\delta} \sigma^2 \right).
\end{align*}
Combining the above $3$ terms, we obtain
\begin{align*}
    \|\tilde{T}\|^2 &\leq 3\|T_1\|^2 + 3\|T_2\|^2 + 3\|T_3\|^2 \\
    & \leq \frac{9(1+\sqrt{1-\delta})^2 \gamma^2}{(1-\beta)^2} \left[\alpha^2 + \beta^2 + (\beta - \alpha)^2 \right]\left(  \epsilon_1^2 +  \|\nabla F(w_t)\|^2 + \frac{3(1-\delta)}{\delta} \sigma^2 \right).
\end{align*}
\end{proof}

\paragraph*{Back to the convergence of $F(.)$:} We use the above bound on $\|\tilde{T}\|^2$ and Lemma~\ref{lem:control_of_error} to conclude the proof of the main convergence result. Recall equation~\eqref{eqn:term_all}:
\begin{align}
    F(\widetilde{w}_{t+1}) &\leq F(\widetilde{w}_t) - \left(\frac{\gamma}{2} -\frac{\gamma}{50} -2L_F\gamma^2 \right) \|\nabla F(w_t) \|^2 + \left(25\gamma L_F^2 + \frac{\gamma L_F^2}{100} + \frac{L_F^2}{2} \right) \|\frac{1}{|\cM|}\sum_{i \in \cM}\er\|^2 \nonumber \\
    &+ 50\gamma \epsilon_2^2 + 2L_F\gamma^2 \epsilon_2^2 + \left( \frac{1}{2\gamma}+\frac{1}{2}+L_F\right) \|\tilde{T}\|^2 \nonumber
\end{align}
First, let us compute the term associated with the error sequence. Note that (from Cauchy-Schwartz inequality)
\begin{align*}
    \|\frac{1}{|\cM|}\sum_{i \in \cM}\er\|^2 \leq \frac{1}{|\cM|}\sum_{i \in \cM}\|\er\|^2,
\end{align*}
and from equation~\eqref{eqn:error}, we obtain
\begin{align*}
  \|\frac{1}{|\cM|}\sum_{i \in \cM}\er\|^2 \leq  \frac{3(1-\delta)}{\delta}\gamma^2 \sigma^2,
\end{align*}
and so the error term is upper bounded by
\begin{align*}
    \left(\frac{\gamma^2 L_F^2 }{2} + \frac{\gamma^3 L_F^2}{100} + 25 \gamma^3 L_F^2 \right) \frac{3(1-\delta)\sigma^2}{\delta}.
\end{align*}
We now substitute the expression for $\|\tilde{T}\|^2$. We obtain
\begin{align*}
    \left(\frac{1}{2\gamma} +\frac{1}{2} + L_F \right) \|\tilde{T}\|^2 = \frac{1}{2\gamma}\|\tilde{T}\|^2 + \left(\frac{1}{2}+L_F\right) \|\tilde{T}\|^2.
\end{align*}
The first term in the above equation is
\begin{align*}
    \frac{1}{2\gamma}\|\tilde{T}\|^2 & \leq \frac{9\gamma (1+\sqrt{1-\delta})^2 }{2(1-\beta)^2} \left[\alpha^2 + \beta^2 + (\beta - \alpha)^2 \right]\left(  \epsilon_1^2 +  \|\nabla F(w_t)\|^2 + \frac{3(1-\delta)}{\delta} \sigma^2 \right) \\
    & \leq \frac{9\gamma (1+\sqrt{1-\delta})^2 }{2(1-\beta)^2} \left[\alpha^2 + \beta^2 + (\beta - \alpha)^2 \right] \|\nabla F(w_t)\|^2 \\
    &+ \frac{9\gamma (1+\sqrt{1-\delta})^2 }{2(1-\beta)^2} \left[\alpha^2 + \beta^2 + (\beta - \alpha)^2 \right]\left(  \epsilon_1^2 + \frac{3(1-\delta)}{\delta} \sigma^2 \right),
\end{align*}
and the second term is
\begin{align*}
    \left(\frac{1}{2}+L_F\right)\|\tilde{T}\|^2 & \leq \left(\frac{1}{2}+L_F\right) \frac{9\gamma^2 (1+\sqrt{1-\delta})^2 }{(1-\beta)^2} \left[\alpha^2 + \beta^2 + (\beta - \alpha)^2 \right]\left(  \epsilon_1^2 +  \|\nabla F(w_t)\|^2 + \frac{3(1-\delta)}{\delta} \sigma^2 \right) \\
    & \leq \left(\frac{1}{2}+L_F\right) \frac{9\gamma^2 (1+\sqrt{1-\delta})^2 }{(1-\beta)^2} \left[\alpha^2 + \beta^2 + (\beta - \alpha)^2 \right] \|\nabla F(w_t)\|^2 \\
    & + \left(\frac{1}{2}+L_F\right) \frac{9\gamma^2 (1+\sqrt{1-\delta})^2 }{(1-\beta)^2} \left[\alpha^2 + \beta^2 + (\beta - \alpha)^2 \right] \left( \epsilon_1^2 + \frac{3(1-\delta)}{\delta} \sigma^2 \right)
\end{align*}
Collecting all the above terms, the coefficient of $- \gamma \|\nabla F(w_t)\|^2$ is given by
\begin{align*}
    \frac{1}{2}-\frac{1}{50}- 2L_F\gamma - \frac{9 (1+\sqrt{1-\delta})^2 }{2(1-\beta)^2} \left[\alpha^2 + \beta^2 + (\beta - \alpha)^2 \right] - (\frac{1}{2}+L_F) \frac{9\gamma (1+\sqrt{1-\delta})^2 }{(1-\beta)^2} \left[\alpha^2 + \beta^2 + (\beta - \alpha)^2 \right].
\end{align*}
Provided we select a sufficiently small $\gamma$, a little algebra shows that if
\begin{align*}
    \frac{9 (1+\sqrt{1-\delta})^2 }{2(1-\beta)^2} \left[\alpha^2 + \beta^2 + (\beta - \alpha)^2 \right] < \left(\frac{1}{2}-\frac{1}{50}\right),
\end{align*}
the coefficient of $\|\nabla F(w_t)\|^2$ becomes $-c \gamma$, where $c > 0$ is a universal constant. Considering the other terms and rewriting equation~\eqref{eqn:term_all}, we obtain
\begin{align*}
     F(\widetilde{w}_{t+1}) &\leq F(\widetilde{w}_t) - c \gamma \|\nabla F(w_t) \|^2 +  \left(\frac{\gamma^2 L_F^2 }{2} + \frac{\gamma^3 L_F^2}{100} + 25 \gamma^3 L_F^2 \right) \frac{3(1-\delta)\sigma^2}{\delta}+ 50\gamma \epsilon_2^2 \\
     & + 2L_F\gamma^2 \epsilon_2^2  +  \frac{9\gamma (1+\sqrt{1-\delta})^2 }{2(1-\beta)^2} \left[\alpha^2 + \beta^2 + (\beta - \alpha)^2 \right]\left(  \epsilon_1^2 + \frac{3(1-\delta)}{\delta} \sigma^2 \right)  \\
    & + \left(\frac{1}{2}+L_F\right) \frac{9\gamma^2 (1+\sqrt{1-\delta})^2 }{(1-\beta)^2} \left[\alpha^2 + \beta^2 + (\beta - \alpha)^2 \right] \left( \epsilon_1^2 + \frac{3(1-\delta)}{\delta} \sigma^2 \right). 
\end{align*}
Continuing, we get
\begin{align*}
    \frac{1}{T+1}\sum_{t=0}^T \|\nabla F(w_t)\|^2 &\leq \frac{1}{c \gamma(T+1)}\sum_{t=0}^T (F(\widetilde{w}_t) - F(\widetilde{w}_{t+1}))+  \left(\frac{\gamma L_F^2 }{2} + \frac{\gamma^2 L_F^2}{100} + 25 \gamma^2 L_F^2 \right) \frac{3(1-\delta)\sigma^2}{c\delta}+ \frac{50}{c} \epsilon_2^2 \\
    & + \frac{2L_F\gamma \epsilon_2^2}{c}  +  \frac{9 (1+\sqrt{1-\delta})^2 }{2c(1-\beta)^2} \left[\alpha^2 + \beta^2 + (\beta - \alpha)^2 \right]\left(  \epsilon_1^2 + \frac{3(1-\delta)}{\delta} \sigma^2 \right)  \\
    & + \left(\frac{1}{2}+L_F\right) \frac{9\gamma (1+\sqrt{1-\delta})^2 }{c(1-\beta)^2} \left[\alpha^2 + \beta^2 + (\beta - \alpha)^2 \right] \left( \epsilon_1^2 + \frac{3(1-\delta)}{\delta} \sigma^2 \right).
\end{align*}
Using the telescoping sum, we obtain
\begin{align*}
    &\min_{t=0,\ldots,T}\|\nabla F(w_t)\|^2 \leq \frac{F(w_0)-F^*}{c\gamma(T+1)} + \left[ \frac{9 (1+\sqrt{1-\delta})^2 }{2c(1-\beta)^2} \left[\alpha^2 + \beta^2 + (\beta - \alpha)^2 \right]\left(  \epsilon_1^2 + \frac{3(1-\delta)}{\delta} \sigma^2 \right) + \frac{50}{c}\epsilon_2^2 \right] \\
    & + \gamma \left[ \frac{ L_F^2 }{2}  \frac{3(1-\delta)\sigma^2}{c\delta} + \frac{2L_F \epsilon_2^2}{c} + \left(\frac{1}{2}+L_F\right) \frac{9 (1+\sqrt{1-\delta})^2 }{c(1-\beta)^2} \left[\alpha^2 + \beta^2 + (\beta - \alpha)^2 \right] \left( \epsilon_1^2 + \frac{3(1-\delta)}{\delta} \sigma^2 \right) \right] \\
    &+ \gamma^2 \left[ (\frac{ L_F^2}{100} + 25  L_F^2 ) \frac{3(1-\delta)\sigma^2}{c\delta} \right]
\end{align*}
Simplifying the above expression, we write
\begin{align*}
    \min_{t=0,\ldots,T}\|\nabla F(w_t)\|^2 &\leq \frac{F(w_0)-F^*}{c\gamma(T+1)} + \Delta_1 + \gamma \Delta_2 + \gamma^2 \Delta_3,
\end{align*}
where the definition of $\Delta_1,\Delta_2$ and $\Delta_3$ are immediate from the above expression.

\end{document}